\documentclass[12pt,reqno]{article}

\usepackage{fullpage}
\usepackage[dvipsnames]{xcolor}
\usepackage[hidelinks,bookmarks=false,linkcolor=black]{hyperref}
\usepackage{graphicx,here,comment}
\usepackage{multirow}
\usepackage{enumitem}
\usepackage{natbib}
\usepackage{booktabs}
\usepackage{subcaption}

\usepackage{amssymb,amsmath,amsfonts,amsthm}
\usepackage{bm}
\usepackage{mathtools}

\newtheorem{theorem}{Theorem}[section]
\newtheorem{corollary}[theorem]{Corollary}
\newtheorem{lemma}[theorem]{Lemma}
\newtheorem{proposition}[theorem]{Proposition}
\newtheorem{assumption}[theorem]{Assumption}

\theoremstyle{definition}
\newtheorem{definition}[theorem]{Definition}

\newtheorem{remark}[theorem]{Remark}

\title{Fast Learning Rate Transfer in Shallow Linear Networks at Growing Training Horizons}
\author{Mana Sakai$^{1,2}$\quad Masaaki Imaizumi$^{1,2,3}$}
\date{}

\begin{document}

\maketitle

\vspace{-2em}
\begin{center}
    $^{1}$The University of Tokyo\\
    $^{2}$RIKEN Center for Advanced Intelligence Project\\
    $^{3}$Kyoto University
\end{center}

\begin{abstract}
Hyperparameter transfer across model width can substantially reduce the cost of tuning large neural networks, but its behavior when the training horizon grows with width is not fully understood. Building on the framework of fast hyperparameter transfer \citep{ghosh2026understanding}, which formalizes when transfer is effective, we investigate conditions that ensure fast transfer in the growing-horizon regime. Specifically, we study learning-rate transfer in a shallow linear network with a single trainable hidden matrix, trained by full-batch gradient descent. Under additional spectral assumptions, our main results are threefold. (i) We prove fast learning-rate transfer as $n,T\to\infty$ whenever $T=o(\sqrt{n})$. (ii) We characterize the transfer rates through the finite-width perturbation scale, the first-order sensitivities of the loss and its learning-rate derivative to finite-width perturbations, and the local loss curvature. (iii) We derive limiting distributions for the optimal learning rate and optimized loss, governed by fluctuations associated with the extreme eigenvalues of the data Gram matrix. These results clarify how spectral structure and local loss sensitivities govern learning-rate transfer at growing horizons.
\end{abstract}

\section{Introduction}\label{sec:introduction}

\subsection{Background and Motivation}

Tuning hyperparameters, such as learning rates, becomes increasingly costly as neural networks grow. \emph{Hyperparameter transfer} seeks to reduce this cost by reusing hyperparameters tuned on small proxy models when training larger models. Building on $\mu$P \citep{yang2021tensor4}, which specifies width-dependent initialization and learning-rate scalings that preserve nontrivial feature learning in the infinite-width limit, \citet{yang2021tuning} introduced $\mu$Transfer: tune hyperparameters on a small proxy model and transfer them to a larger model without retuning. Their experiments demonstrate that appropriately parameterized learning rates can remain effective across substantially different widths. Subsequent studies have examined the empirical robustness of transfer to architectural and optimizer choices \citep{lingle2024empirical}, as well as the role of layerwise learning-rate scaling across parameterizations \citep{lingle2024empirical,everett2024scaling}. These empirical successes motivate theoretical analysis of the reliable transfer of hyperparameters.

\citet{ghosh2026understanding} introduced the notion of \emph{fast hyperparameter transfer} to formalize when transfer is effective. Their framework shows that what matters is whether the optimal hyperparameter converges sufficiently fast relative to the rate at which model performance converges. Thus convergence of the optimal hyperparameter alone does not guarantee fast transfer. Relatedly, \citet{hayou2026proof} established convergence of the optimal learning rate for deep linear networks trained by full-batch gradient descent under $\mu$P, providing a theoretical basis for learning-rate transfer across width.

One major challenge in hyperparameter transfer is dealing with \textit{growing training horizons} $T$, particularly for learning-rate transfer. Recent language-model scaling practices motivate increasing the number of training tokens alongside model size \citep{hoffmann2022empirical,dey2025dont}, which often increases the number of gradient updates. From a theoretical perspective, the joint scaling limit of width and training horizon is substantially more delicate than the fixed-horizon setting, since the limiting objective itself changes as $T$ grows. Consequently, the theoretical understanding of this regime remains limited. An important exception is \citet{wen2026fast}, who establish fast learning-rate transfer at growing horizons in sketched linear regression. Motivated by these observations, we ask the following questions:
\begin{quote}
    \emph{As model width and training horizon grow jointly, what conditions ensure fast learning-rate transfer, and which quantities govern whether such transfer occurs?}
\end{quote}

\subsection{This Study}

To answer these questions, we consider a linear network with a single trainable hidden matrix, frozen random input and output maps, and $\mu$P initialization. The network is trained by full-batch gradient descent with a constant learning rate. For each horizon $T$, we compare the optimal learning rate and optimized training loss at width $n$ with their infinite-width counterparts.

Our contributions are as follows:
\begin{itemize}[leftmargin=*]
    \item \textbf{A sufficient growth condition on the training horizon for fast transfer}: For fixed sample size and input dimension, we prove fast transfer whenever $T\to\infty$ and $T=o(\sqrt{n})$, under additional spectral assumptions (Theorem~\ref{thm:a-b-c-summary}; Corollary~\ref{cor:a-b-c-summary}).
    \item \textbf{Summary of transfer errors using loss-based quantities}: In our setting, the rates of the three transfer errors can be summarized in terms of the finite-width perturbation scale and three horizon-dependent quantities: the first-order sensitivities of the loss value and its learning-rate derivative to finite-width perturbations, and the local curvature of the loss. This formulation makes explicit how these quantities combine to determine the transfer rates as the training horizon grows (Theorem~\ref{thm:a-b-c-summary}; Section~\ref{subsec:perturbation-formulation}). It also helps interpret the failure of fast transfer in the fixed-horizon counterexample of Proposition~\ref{prop:isotropic-no-fast-transfer}.
    \item \textbf{Limiting distributions of the optimal learning rate and the optimized loss}: We derive limiting distributions of the optimal learning rate and optimized loss as the training horizon grows (Theorem~\ref{thm:a-b-c-summary}). Both limits are linear functionals of the same two components of the perturbation, associated with the largest and smallest positive eigenvalues of the data Gram matrix.
\end{itemize}

Our analysis provides some insights. First, different spectra can lead to different transfer rates and sufficient width--horizon scalings, since the loss-based quantities vary depending on the data spectrum as the horizon grows. This observation contrasts with \citet{wen2026fast}. Second, the leading loss-value and slope responses to the same finite-width perturbation need not be aligned. Consequently, some perturbation directions can change the optimized loss at leading order while producing no leading-order shift in the optimal learning rate. This illustrates how width dependence in the attained loss can differ from width dependence in hyperparameter selection. The observation is qualitatively consistent with the mechanism proposed by \citet{ghosh2026understanding}, although our analysis concerns local finite-width perturbations rather than a decomposition of the optimization trajectory; Section~\ref{sec:transfer-mechanisms} discusses this connection.

Additional related work is deferred to Appendix~\ref{app:additional-related-work}.

\section{Problem Setup and Transfer Framework}\label{sec:setup}

\subsection{Model and exact loss dynamics}\label{subsec:model}

Fix positive integers $d,m$, training inputs $X^{\top}=[x_{1}\ ...\ x_{m}]$, and targets $y=(y_{1},...,y_{m})^{\top}$. Consider the linear network
\[
    f_{n}(x;W_{0},W_{1},V)
    =V^{\top}W_{1}W_{0}x
    ,
\]
where $W_{0}\in\mathbb{R}^{n\times d}$, $W_{1}\in\mathbb{R}^{n\times n}$, and $V\in\mathbb{R}^{n}$. This model is similar to the one analyzed in \citet{hayou2026proof}, except we have only one trainable weight $W_{1}$. The initialization is given entrywise as $W_{0,ij}\sim\mathcal{N}(0,d^{-1})$, $W_{1,ij}^{(0)}\sim\mathcal{N}(0,n^{-1})$, and $V_{i}\sim\mathcal{N}(0,n^{-2})$, following the scalings in $\mu$P. All entries of $W_{0}$, $W_{1}^{(0)}$, and $V$ are mutually independent. Define
\[
    K
    =d^{-1}XX^{\top}
    \in\mathbb{R}^{m\times m}
    ,\qquad
    A_{n}
    =\|V\|^{2}H_{n}H_{n}^{\top}
    \in\mathbb{R}^{m\times m}
    ,\qquad
    H_{n}
    =XW_{0}^{\top}
    \in\mathbb{R}^{m\times n}
    .
\]
The vector of predictions and the residual at step $t$ are
\[
    f_{n}^{(t)}(X)
    =H_{n}(W_{1}^{(t)})^{\top}V
    \in\mathbb{R}^{m}
    ,
    \qquad
    \chi_{n}^{(t)}
    =f_{n}^{(t)}(X)-y
    \in\mathbb{R}^{m}
    .
\]
The loss is $\mathcal{L}_{n}(W_{1})=(2m)^{-1}\|f_{n}(X;W_{1})-y\|^{2}$, and for a fixed learning rate $\eta$, we train the model by full-batch gradient descent (GD) as
\[
    W_{1}^{(t+1)}=W_{1}^{(t)}-\eta\nabla_{W_{1}}\mathcal{L}_{n}(W_{1}^{(t)})
    \qquad(t=0,\ldots,T-1)
    .
\]
Since there is only one trainable weight, training this model with the GD update is equivalent to performing GD for linear regression on fixed random features.

For $A\in\mathbb{R}^{m\times m}$ and $\eta\in\mathbb{R}$, define $B_{A}(\eta)=I_{m}-m^{-1}\eta A$. We abbreviate $B_{n}(\eta):=B_{A_{n}}(\eta)$ and $B_{\infty}(\eta):=B_{K}(\eta)$.

\begin{proposition}[Exact finite-width loss dynamics]\label{prop:dynamics}
    For every integer $T\ge1$ and every $\eta\in\mathbb{R}$, we have $\chi_{n}^{(T)}=B_{n}(\eta)^{T}\chi_{n}^{(0)}$, and consequently,
    \[
        \phi_{n,T}(\eta)
        :=\mathcal{L}_{n}(W_{1}^{(T)})
        =(2m)^{-1}(\chi_{n}^{(0)})^{\top}B_{n}(\eta)^{2T}\chi_{n}^{(0)}
        .
    \]
\end{proposition}

Proposition~\ref{prop:dynamics} turns learning-rate selection into the minimization of a scalar random loss curve. All dependence on the initialization enters through the fixed-dimensional pair $(\chi_{n}^{(0)},A_{n})$, even though the parameter matrix grows with $n$. This reduction separates the two ingredients of the analysis: the spectral geometry of the reference curve and the finite-width fluctuations around it.

\subsection{Infinite-width objectives and spectral geometry}\label{subsec:infinite-width-objective}

For each $T$, define the corresponding infinite-width objective by
\[
    \phi_{\infty,T}(\eta)
    :=(2m)^{-1}\|B_{\infty}(\eta)^{T}y\|^{2}
    =(2m)^{-1}y^{\top}B_{\infty}(\eta)^{2T}y
    .
\]
For every fixed $T$, Corollary~\ref{cor:fixed-T-limit-loss} shows that $\phi_{n,T}$ converges to $\phi_{\infty,T}$ uniformly on deterministic compact learning-rate intervals.

Consider the decomposition $K=\sum_{j=1}^{r}\lambda_{j}P_{j}$, where $\lambda_{1}>\cdots>\lambda_{r}>0$ are the distinct positive eigenvalues and $P_{j}$ are their orthogonal spectral projectors.\footnote{If $K$ has repeated eigenvalues, the corresponding $P_{j}$ are sums of the projections onto the individual eigenspaces.} Let $P_{0}$ denote the projection onto $\ker(K)$, and set $c_{j}=\|P_{j}y\|^{2}$ and $c_{0}=\|P_{0}y\|^{2}$. Define $\mathcal{J}_{y}=\{j\in[r]:c_{j}>0\}$, and we assume $\mathcal{J}_{y}\neq\emptyset$. Then we can express the infinite-width objective as
\[
    \phi_{\infty,T}(\eta)
    =\frac{c_{0}}{2m}+\frac{1}{2m}\sum_{j=1}^{r}c_{j}\left(1-\frac{\eta\lambda_{j}}{m}\right)^{2T}
    .
\]
Each active eigenspace contributes a term with contraction factor $1-\eta\lambda_{j}/m$, while $c_{0}/(2m)$ is independent of the learning rate. Thus $\mathcal{J}_{y}$ identifies the spectral modes relevant to learning-rate selection.
\subsection{Optimal learning rates and transfer quantities}

Fix a deterministic compact interval $\mathcal{I}=[\eta_{\ell},\eta_{u}]$. For every $n\in\mathbb{N}$ and $T\ge1$, define
\[
    \eta_{n,T}
    =\min\arg\min_{\eta\in\mathcal{I}}\phi_{n,T}(\eta)
    ,\qquad
    \eta_{\infty,T}
    =\min\arg\min_{\eta\ge0}\phi_{\infty,T}(\eta)
    .
\]
Since $\phi_{n,T}$ is continuous and $\mathcal{I}$ is compact, $\arg\min_{\eta\in\mathcal{I}}\phi_{n,T}(\eta)$ is a nonempty compact set, so $\eta_{n,T}$ is well defined. Moreover, Corollary~\ref{cor:optimizer-characterization} and Proposition~\ref{prop:loss-convex-as} imply that for sufficiently large $n$, $\eta_{n,T}$ is almost surely the unique minimizer of $\phi_{n,T}$ over $\mathcal{I}$. When $Ky\neq0$, Corollary~\ref{cor:optimizer-characterization} also implies that the minimizer of $\phi_{\infty,T}$ is positive and unique over $[0,\infty)$.

Following \citet{ghosh2026understanding}, we define the transfer quantities.

\begin{definition}
    For each width $n$ and training horizon $T$, define the loss gap $a_{n,T}$, hyperparameter gap $b_{n,T}$, and transfer suboptimality gap $c_{n,T}$ by
    \[
        a_{n,T}
        =|\phi_{n,T}(\eta_{n,T})-\phi_{\infty,T}(\eta_{\infty,T})|
        ,\quad
        b_{n,T}
        =|\eta_{n,T}-\eta_{\infty,T}|
        ,\quad
        c_{n,T}
        =\phi_{\infty,T}(\eta_{n,T})-\phi_{\infty,T}(\eta_{\infty,T})
        .
    \]
    We say the hyperparameter transfer is \emph{fast} when $c_{n,T}=o_{p}(a_{n,T})$ holds.
\end{definition}

The loss gap measures how much optimized performance varies across width, whereas the transfer suboptimality gap measures the additional target loss due specifically to using the proxy-selected learning rate. The hyperparameter gap records the displacement between the two minimizers. Fast transfer means that the cost of not retuning is negligible relative to the width-induced discrepancy in optimized losses. These comparisons always use a common horizon $T$, including along a sequence $T=T_{n}\to\infty$; we do not transfer between models trained for different numbers of updates.

\section{Main results}\label{sec:main-results}

\subsection{Long-horizon structure of the infinite-width objective}\label{subsec:long-horizon-infinite-width}

We now let $T=T_{n}\to\infty$. In this regime, the limiting objective itself changes with $n$, and its optimizer moves toward the large-$T$ limit. Thus, we first identify the moving deterministic reference in Section~\ref{subsec:long-horizon-infinite-width}, and then compare the finite-width problem with this reference.

The spectral representation of $\phi_{\infty,T}$ in Section~\ref{subsec:infinite-width-objective} shows that, for large $T$, the learning-rate-dependent loss is governed by the largest active contraction magnitude, $\max_{j\in\mathcal{J}_{y}}|1-\eta\lambda_{j}/m|$. The long-horizon optimizer must therefore balance the two extreme active modes. The following theorem identifies this balance.

\begin{theorem}[Active-spectrum limit of the optimal learning rate]\label{thm:finite-minimax}
    In the decomposition of $K=\sum_{j=1}^{r}\lambda_{j}P_{j}$, define $\lambda_{+}^{y}=\max_{j\in\mathcal{J}_{y}}\lambda_{j}$ and $\lambda_{-}^{y}=\min_{j\in\mathcal{J}_{y}}\lambda_{j}$. Suppose $\lambda_{+}^{y}>\lambda_{-}^{y}>0$. Then as $T\to\infty$, we have
    \begin{equation}\label{eq:finite-minimax-eta}
        \eta_{\infty,T}
        \to\eta_{*}^{y}
        :=2m/(\lambda_{+}^{y}+\lambda_{-}^{y})
        .
    \end{equation}
    If $\lambda_{+}^{y}=\lambda_{-}^{y}=\lambda$, then we have $\eta_{\infty,T}=m/\lambda$ for all $T\ge1$.
\end{theorem}

In Theorem~\ref{thm:finite-minimax}, observe that the optimizer limit $\eta_{*}^{y}$ remains strictly inside the active-spectrum stability interval $(0,2m/\lambda{+}^{y})$. If $\lambda_{+}^{y}=\lambda_{1}$, it also lies inside the stability interval of the full iteration. Indeed, if the top active eigenvalue equals $\lambda_{1}=\lambda_{\max}(K)$, then we have $\eta_{*}^{y}=2m/(\lambda_{1}+\lambda_{-}^{y})<2m/\lambda_{1}$. We also note that the limit in Theorem~\ref{thm:finite-minimax} is closely related to the classical optimal step size for Richardson iteration; see, e.g., \citet[Section~4.2.1]{saad2003iterative}.

Theorem~\ref{thm:finite-minimax} concerns the deterministic objective, which involves only modes active in $y$. At finite width, initialization and matrix perturbations can also activate modes absent from this reference objective. To keep the long-horizon comparison governed by the same spectral endpoints, we impose the following stronger condition.

\begin{assumption}[Endpoint activity]\label{asmp:endpoint-support}
    In the decomposition of $K=\sum_{j=1}^{r}\lambda_{j}P_{j}$, suppose $r\ge2$. The largest and smallest positive eigenvalues $\lambda_{1}$ and $\lambda_{r}$ of $K$ are simple and satisfy $\langle y,u_{1}\rangle\neq0$ and $\langle y,u_{r}\rangle\neq0$, where $u_{1}$ and $u_{r}$ are unit eigenvectors satisfying $Ku_{1}=\lambda_{1}u_{1}$ and $Ku_{r}=\lambda_{r}u_{r}$. Thus we have $\lambda_{+}^{y}=\lambda_{1}$ and $\lambda_{-}^{y}=\lambda_{r}$.
\end{assumption}

Under Assumption~\ref{asmp:endpoint-support}, set $\eta_{*}=2m/(\lambda_{1}+\lambda_{r})=\eta_{*}^{y}$ and $q_{*}=(\lambda_{1}-\lambda_{r})/(\lambda_{1}+\lambda_{r})$. The endpoint contraction factors at $\eta_{*}$ are $-q_{*}$ and $q_{*}$. Proposition~\ref{prop:finite-T-correction} and Corollary~\ref{cor:T-loss-rate} give a more quantitative picture: $\eta_{\infty,T}-\eta_{*}=O(T^{-1})$, i.e., the optimizer approaches a fixed location, while
\[
    \phi_{\infty,T}(\eta_{\infty,T})-c_{0}/(2m)
    =\Theta(q_{*}^{2T})
    ,\qquad
    \phi_{\infty,T}''(\eta_{\infty,T})
    =\Theta(T^{2}q_{*}^{2T-2})
\]
so that the optimizable part of the loss and the local curvature shrink exponentially. The growing-horizon transfer analysis must track finite-width perturbations relative to these changing scales.

\subsection{Fast learning-rate transfer at growing horizons}\label{subsec:transfer-growing-T}

We now turn from the deterministic long-horizon geometry of the infinite-width objective to finite-width fluctuations. By Proposition~\ref{prop:dynamics}, all randomness relevant to the finite-width loss curve is entirely captured by $(\chi_{n}^{(0)},A_{n})$. Thus, the first step is to characterize how this pair fluctuates around its infinite-width limit. The following theorem identifies the width-dependent fluctuation scale and its limiting random direction.

\begin{lemma}[Joint initialization central limit theorem]\label{lem:joint-clt}
    Suppose $m$ and $d$ are fixed. Define $\alpha\sim\mathcal{N}(0,2)$. Let $G$ be a symmetric Gaussian $d\times d$ matrix satisfying $\mathbb{E}G_{rs}=0$ and $\mathbb{E}(G_{rs}G_{tu})=d^{-2}(\delta_{rt}\delta_{su}+\delta_{ru}\delta_{st})$. Define $\Xi=\alpha K+XGX^{\top}$, and $Z\sim\mathcal{N}(0,K)$. Suppose that $\alpha,G$, and $Z$ are mutually independent. Then we have
    \[
        \sqrt{n}(\chi_{n}^{(0)}+y,A_{n}-K)
        \stackrel{d}{\longrightarrow}(Z,\Xi)
        .
    \]
\end{lemma}

Lemma~\ref{lem:joint-clt} shows that the primitive finite-width perturbation has scale $\epsilon_{n}=n^{-1/2}$. The effect of this perturbation on learning-rate transfer, however, also depends on the training horizon. As established above, the infinite-width loss and its local curvature vary substantially with $T$, hence the sensitivities of the loss value and its learning-rate derivative to finite-width perturbations have their own horizon-dependent scales. To separate these effects from the width-dependent fluctuation scale $\epsilon_n$, define
\[
    \rho_{L,T}
    =Tq_{*}^{2T-1}
    ,\qquad
    \rho_{G,T}
    =T^{2}q_{*}^{2T-2}
    ,\qquad
    \kappa_{T}
    =T^{2}q_{*}^{2T-2}
    .
\]
The scales $\rho_{L,T}$ and $\rho_{G,T}$ measure the first-order sensitivity of the loss and its learning-rate derivative to finite-width perturbations, respectively. The scale $\kappa_{T}$ measures the local curvature $\phi_{\infty,T}''(\eta_{\infty,T})$. Thus, $\epsilon_{n}$ quantifies the magnitude of the finite-width perturbation, whereas $\rho_{L,T}$, $\rho_{G,T}$, and $\kappa_T$ capture how the loss response, loss-gradient response, and local curvature vary with the training horizon. The following theorem combines these scales to characterize the fluctuations of the optimal learning rate and optimized loss, and the resulting rates of the three transfer gaps.

\begin{theorem}[Growing-horizon fluctuations and transfer rates; Theorems~\ref{thm:b-limit-dist},~\ref{thm:a-limit-dist}, and~\ref{thm:a-b-c-rate}, extracted]\label{thm:a-b-c-summary}
    Suppose $m$ and $d$ are fixed and $T=T_{n}$ satisfies $T\to\infty$ and $T/\sqrt{n}\to0$. Suppose Assumption~\ref{asmp:endpoint-support} holds, and suppose $0<\eta_{\ell}<\eta_{*}<\eta_{u}<2m/\lambda_{1}$. Set $ L_{*}=\log(c_{1}\lambda_{1}/(c_{r}\lambda_{r}))$, $a_{1}^{*}=\exp(-\lambda_{1}L_{*}/(\lambda_{1}+\lambda_{r}))$, and $a_{r}^{*}=\exp(\lambda_{r}L_{*}/(\lambda_{1}+\lambda_{r}))$. Then we have
    \begin{align*}
        &(\epsilon_{n}\rho_{G,T}/\kappa_{T})^{-1}(\eta_{n,T}-\eta_{\infty,T})
        \stackrel{d}{\longrightarrow}\Omega_{\eta,*}
        :=-2m(\lambda_{1}+\lambda_{r})^{-2}(u_{1}^{\top}\Xi u_{1}+u_{r}^{\top}\Xi u_{r})
        ,\\
        &(\epsilon_{n}\rho_{L,T})^{-1}(\phi_{n,T}(\eta_{n,T})-\phi_{\infty,T}(\eta_{\infty,T}))
        \stackrel{d}{\longrightarrow}\Omega_{L,*}
        :=\langle Q_{L,*},\Xi\rangle_{\mathrm{F}}
        ,
    \end{align*}
    where $\Xi$ is defined in Lemma~\ref{lem:joint-clt} and $Q_{L,*}$ is defined by $Q_{L,*}=m^{-2}\eta_{*}(c_{1}a_{1}^{*}u_{1}u_{1}^{\top}-c_{r}a_{r}^{*}u_{r}u_{r}^{\top})$. Consequently, we have
    \[
        a_{n,T}
        =\Theta_{p}(\epsilon_{n}\rho_{L,T})
        =\Theta_{p}\bigg(\frac{Tq_{*}^{2T-1}}{\sqrt{n}}\bigg)
        ,\qquad
        b_{n,T}
        =\Theta_{p}\bigg(\frac{\epsilon_{n}\rho_{G,T}}{\kappa_{T}}\bigg)
        =\Theta_{p}\bigg(\frac{1}{\sqrt{n}}\bigg)
        ,
    \]
    and
    \[
        c_{n,T}
        =\Theta_{p}(\kappa_{T}b_{n,T}^{2})
        =\Theta_{p}\bigg(\frac{\epsilon_{n}^{2}\rho_{G,T}^{2}}{\kappa_{T}}\bigg)
        =\Theta_{p}\bigg(\frac{T^{2}q_{*}^{2T-2}}{n}\bigg)
        .
    \]
\end{theorem}

\begin{corollary}[Fast transfer at growing horizons]\label{cor:a-b-c-summary}
    Under the assumptions of Theorem~\ref{thm:a-b-c-summary}, we have
    \[
        \frac{c_{n,T}}{a_{n,T}}
        =\Theta_{p}\bigg(\frac{\epsilon_{n}^{2}\rho_{G,T}^{2}/\kappa_{T}}{\epsilon_{n}\rho_{L,T}}\bigg)
        =\Theta_{p}\bigg(\frac{T}{q_{*}\sqrt{n}}\bigg)
        \stackrel{p}{\longrightarrow}0
        .
    \]
\end{corollary}

By Corollary~\ref{cor:a-b-c-summary}, under the assumptions of Theorem~\ref{thm:a-b-c-summary}, fast learning-rate transfer persists at growing training horizons when $T=o(\sqrt{n})$.

\subsection{A fixed-horizon benchmark}\label{sec:transfer-fixed-T}

Having established our main growing-horizon result in Theorem~\ref{thm:a-b-c-summary}, we now turn to the fixed-horizon regime. In particular, we verify that, under the nondegeneracy condition below, our model recovers the general rates conjectured by \citet{wen2026fast} for fixed $T$: $a_{n,T}=\Theta_{p}(n^{-1/2})$, $b_{n,T}=\Theta_{p}(n^{-1/2})$, and $c_{n,T}=\Theta_{p}(n^{-1})$. These rates also provide a useful baseline for Theorem~\ref{thm:a-b-c-summary}. When $T$ is fixed, the horizon-dependent factors are absorbed into constants, whereas allowing $T$ to grow introduces nontrivial interactions between width and training horizon.

\begin{theorem}[Fixed-horizon transfer rates]\label{thm:fixed-horizon-transfer}
    Fix $m,d,T\in\mathbb{N}$, and suppose $|\mathcal{J}_{y}|=\#\{j\in[r]:c_{j}>0\}\ge2$.\footnote{This condition is equivalent to $K^{2}y\notin\operatorname{span}\{Ky\}$.} Assume in addition that $\eta_{\infty,T}\in(\eta_{\ell},\eta_{u})$. Then for $\epsilon_{n}=n^{-1/2}$, we have
    \[
        a_{n,T}
        =\Theta_{p}(\epsilon_{n})
        ,\qquad
        b_{n,T}
        =\Theta_{p}(\epsilon_{n})
        ,\qquad
        c_{n,T}
        =\Theta_{p}(\epsilon_{n}^{2})
        .
    \]
\end{theorem}

\begin{corollary}[Fast transfer at a fixed horizon]\label{cor:fixed-horizon-fast-transfer}
    Under the assumptions of Theorem~\ref{thm:fixed-horizon-transfer}, we have $c_{n,T}/a_{n,T}=\Theta_{p}(\epsilon_{n})\stackrel{p}{\longrightarrow}0$.
\end{corollary}

Corollary~\ref{cor:fixed-horizon-fast-transfer} implies that, under the assumptions of Theorem~\ref{thm:fixed-horizon-transfer}, learning-rate transfer is fast at every fixed horizon $T$.

\subsection{Counterexample to fast transfer}

The preceding results establish fast transfer both at growing horizons under the conditions of Theorem~\ref{thm:a-b-c-summary} and, more broadly at fixed horizons, under the nondegeneracy condition of Theorem~\ref{thm:fixed-horizon-transfer}. Fast transfer, however, is not automatic even when the training horizon is fixed. We now exhibit a counterexample in which fast transfer fails.

\begin{proposition}[Failure of fast transfer with a single positive eigenvalue]\label{prop:isotropic-no-fast-transfer}
    Fix integers $2\le m\le d$, deterministic data $X\in\mathbb{R}^{m\times d}$ satisfying $XX^{\top}=I_{m}$, and a nonzero $y\in\mathbb{R}^{m}$. Note that $|\mathcal{J}_{y}|=1$ holds in this case. Set $\eta_{0}:=md\in(\eta_{\ell},\eta_{u})$, and fix an integer $T\ge1$. Then we have
    \[
        a_{n,T}
        =\Theta_{p}(n^{-T})
        ,\qquad
        b_{n,T}
        =\Theta_{p}(n^{-1/2})
        ,\qquad
        c_{n,T}
        =\Theta_{p}(n^{-T})
        ,\qquad
        c_{n,T}/a_{n,T}
        =\Theta_{p}(1)
        ,
    \]
    thus fast transfer fails.
\end{proposition}

Proposition~\ref{prop:isotropic-no-fast-transfer} shows that convergence of the optimal learning rate alone is not sufficient for fast transfer. Indeed, although $b_{n,T}=\Theta_{p}(n^{-1/2})$ holds exactly as in Theorem~\ref{thm:fixed-horizon-transfer}, the loss gap and the transfer suboptimality gap both scale as $\Theta_{p}(n^{-T})$. Hence the transfer penalty is not asymptotically negligible relative to the finite-width loss discrepancy. The mechanism behind this failure is a degeneracy of the leading loss sensitivity; Section~\ref{sec:transfer-mechanisms} revisits this example from the local perturbation perspective.

\section{Perturbation structure of hyperparameter transfer}\label{sec:proof-overview}

\subsection{Model-specific perturbation limits for Theorem~\ref{thm:a-b-c-summary}}\label{subsec:proof-structure}

In this section, assume the conditions of Theorem~\ref{thm:a-b-c-summary}. The purpose of this section is model-specific: we identify the first-order finite-width perturbations and their limiting distributions in our shallow linear model. Section~\ref{subsec:perturbation-formulation} subsequently abstracts the rate-level consequences of these calculations into a general local perturbation formulation.

For symmetric $A\in\mathbb{R}_{\mathrm{sym}}^{m\times m}$ and $\chi\in\mathbb{R}^{m}$, define the deterministic loss map and score
\[
    \Phi_{T}(\eta;\chi,A)
    =(2m)^{-1}\chi^{\top}B_{A}(\eta)^{2T}\chi
    ,\qquad
    F_{T}(\eta;\chi,A)
    =\chi^{\top}AB_{A}(\eta)^{2T-1}\chi
    .
\]
Thus we have $\phi_{n,T}=\Phi_{T}(\cdot;\chi_{n}^{(0)},A_{n})$ and $\phi_{\infty,T}=\Phi_{T}(\cdot;-y,K)$. Define $F_{n,T}$ and
$F_{\infty,T}$ by $F_{n,T}=F_{T}(\cdot;\chi_{n}^{(0)},A_{n})$ and $F_{\infty,T}=F_{T}(\cdot;-y,K)$. Proposition~\ref{prop:loss-derivatives} gives $\partial_{\eta}\Phi_{T}=-(T/m^{2})F_{T}$, so interior minimizers of $\Phi_{T}$ are roots of $F_{T}$.

\paragraph{Linearization and horizon dependence.}
Lemma~\ref{lem:joint-clt} gives $\sqrt{n}(\chi_{n}^{(0)}+y,A_{n}-K)\stackrel{d}{\longrightarrow}(Z,\Xi)$. Writing $D_{(\chi,A)}$ for the Fr\'echet derivative with respect to $(\chi,A)$ at $(\eta_{\infty,T};-y,K)$, and noting that Corollary~\ref{cor:optimizer-characterization} implies $F_{\infty,T}(\eta_{\infty,T})=0$, Lemmas~\ref{lem:frechet},~\ref{lem:score-linearization}, and~\ref{lem:loss-linearization} yield
\begin{align}
    \label{eq:F-Frechet-expansion}
    &F_{n,T}(\eta_{\infty,T})
    =D_{(\chi,A)}F_{T}[\chi_{n}^{(0)}+y,A_{n}-K]+O_{p}(T^{2}q_{*}^{2T-3}/n)
    ,\\
    \label{eq:Phi-Frechet-expansion}
    &\phi_{n,T}(\eta_{\infty,T})-\phi_{\infty,T}(\eta_{\infty,T})
    =D_{(\chi,A)}\Phi_{T}[\chi_{n}^{(0)}+y,A_{n}-K]+O_{p}(T^{2}q_{*}^{2T-2}/n)
    .
\end{align}
In particular, Lemma~\ref{lem:frechet} gives
\begin{align*}
    &D_{(\chi,A)}F_{T}(\eta_{\infty,T};-y,K)[h,E]
    =g_{F,T}^{\top}h+\langle Q_{F,T},E\rangle_{\mathrm{F}}
    ,\\
    &D_{(\chi,A)}\Phi_{T}(\eta_{\infty,T};-y,K)[h,E]
    =g_{L,T}^{\top}h+\langle Q_{L,T},E\rangle_{\mathrm{F}}
    ,
\end{align*}
where $g_{F,T},Q_{F,T},g_{L,T},Q_{L,T}$ are the fixed Fr\'echet coefficients defined in \eqref{eq:loss-derivative-sensitivity-chi}, \eqref{eq:loss-derivative-sensitivity-A}, \eqref{eq:loss-sensitivity-chi}, and \eqref{eq:loss-sensitivity-A}. Since
$(\chi_{n}^{(0)}+y,A_{n}-K)$ is $O_{p}(n^{-1/2})$, the size of the finite-width fluctuations is determined by the $T$-dependence of the Fr\'echet coefficients. Lemma~\ref{lem:frechet-coefficient-limits} shows that on $\operatorname{range}(K)$, as $T\to\infty$, $Q_{F,T}/(Tq_{*}^{2T-2})$ converges to $Q_{F,*}$ defined in Lemma~\ref{lem:frechet-coefficient-limits}, and $Q_{L,T}/(Tq_{*}^{2T-1})$ converges to $Q_{L,*}$. Lemma~\ref{lem:frechet-coefficient-limits} also shows that on $\operatorname{range}(K)$, $g_{F,T}$ and $g_{L,T}$ are $o(Tq_{*}^{2T-2})$ and $o(Tq_{*}^{2T-1})$, respectively. Thus, at the infinite-width optimizer, the leading finite-width perturbations come from $A_{n}-K$, rather than from the initial prediction $f_{n}^{(0)}(X)=\chi_{n}^{(0)}+y$.

Combining these coefficient asymptotics with Lemma~\ref{lem:joint-clt} and \eqref{eq:F-Frechet-expansion}, we have $F_{n,T}(\eta_{\infty,T})=\langle Q_{F,T},A_{n}-K\rangle_{\mathrm{F}}+o_{p}(Tq_{*}^{2T-2}/\sqrt{n})$, and hence
\begin{equation}\label{eq:normalized-score-limit-2}
    \sqrt{n}(Tq_{*}^{2T-2})^{-1}F_{n,T}(\eta_{\infty,T})
    \stackrel{d}{\longrightarrow}\langle Q_{F,*},\Xi\rangle_{\mathrm{F}}
    =-2m^{-1}\eta_{*}h_{*}(u_{1}^{\top}\Xi u_{1}+u_{r}^{\top}\Xi u_{r})
    ,
\end{equation}
where $h_{*}:=c_{1}\lambda_{1}a_{1}^{*}$; the equality $c_{1}\lambda_{1}a_{1}^{*}=c_{r}\lambda_{r}a_{r}^{*}$ is established in \eqref{eq:endpoint-balance-constant}. Similarly, Lemma~\ref{lem:joint-clt} and \eqref{eq:Phi-Frechet-expansion} give $\phi_{n,T}(\eta_{\infty,T})-\phi_{\infty,T}(\eta_{\infty,T})=\langle Q_{L,T},A_{n}-K\rangle_{\mathrm{F}}+o_{p}(Tq_{*}^{2T-1}/\sqrt{n})$, so that
\begin{equation}\label{eq:loss-at-limit-eta-2}
    \sqrt{n}(Tq_{*}^{2T-1})^{-1}(\phi_{n,T}(\eta_{\infty,T})-\phi_{\infty,T}(\eta_{\infty,T}))
    \stackrel{d}{\longrightarrow}\langle Q_{L,*},\Xi\rangle_{\mathrm{F}}
    =\Omega_{L,*}
    .
\end{equation}

\paragraph{From the perturbation limits to the distributional results.}
Lemma~\ref{lem:preliminary-b-rate} localizes the finite-width optimizer at $\eta_{n,T}-\eta_{\infty,T}=O_{p}(n^{-1/2})=o_{p}(T^{-1})$. Together with the local score-derivative asymptotics, the score limit \eqref{eq:normalized-score-limit-2} yields $\sqrt{n}(\eta_{n,T}-\eta_{\infty,T})\stackrel{d}{\longrightarrow}\Omega_{\eta,*}$, as stated in Theorem~\ref{thm:b-limit-dist}. Likewise, finite-width reoptimization changes the loss by $O_{p}(T^{2}q_{*}^{2T-2}/n)=o_{p}(Tq_{*}^{2T-1}/\sqrt{n})$, so the limit in \eqref{eq:loss-at-limit-eta-2} is unchanged when $\eta_{\infty,T}$ is replaced by $\eta_{n,T}$, yielding Theorem~\ref{thm:a-limit-dist}. The remaining rate-level consequences of these perturbation estimates are organized abstractly in Section~\ref{subsec:perturbation-formulation}, while the complete model-specific arguments are given in Appendix~\ref{app:proof-width-horizon-transfer}.

\subsection{A local perturbation formulation for horizon-dependent transfer}\label{subsec:perturbation-formulation}

Our fixed- and growing-horizon proofs can be organized through a local perturbation argument whose ingredients are closely related to earlier analyses of fast hyperparameter transfer. \citet{ghosh2026understanding} use local strong convexity to relate hyperparameter displacement to transfer suboptimality for a fixed limiting objective. In the growing-horizon setting, \citet{wen2026fast} track horizon-dependent score perturbations, local curvature, optimizer localization, and finite-width loss fluctuations.

The purpose of this section is to recast these proof ingredients in a common horizon-dependent formulation that makes the relevant scales explicit. Section~\ref{subsec:proof-structure} provides the model-specific first-order perturbation estimates and limiting distributions for our shallow linear model; here we retain only the corresponding loss-response, loss-gradient-response, curvature, and localization scales in order to explain the resulting rates of $a_{n,T}$, $b_{n,T}$, and $c_{n,T}$ and the condition for fast transfer.

We consider a scalar hyperparameter $\eta$, and allow $T=T_{n}$ to be fixed or diverging. Write
\[
    \phi_{n,T}(\eta)=\Phi_{T}(\eta;\gamma_{n})
    ,\qquad
    \phi_{\infty,T}(\eta)=\Phi_{T}(\eta;\gamma_{\infty})
    ,\qquad
    \|\gamma_{n}-\gamma_{\infty}\|=O_{p}(\epsilon_{n})
    ,
\]
where $\gamma_{\infty}$ is deterministic and $\epsilon_{n}\to0$. All stochastic orders below are along the chosen sequence $(n,T_{n})$, and the scales $\rho_{L,T},\rho_{G,T},\kappa_{T},r_{T}$ are deterministic and positive. Assume that the objectives are twice continuously differentiable in $\eta$ near the search interval $\mathcal{I}$, and that the derivative in $\gamma$ appearing below exists. Let $\eta_{\infty,T}$ be a global minimizer of $\phi_{\infty,T}$ lying in the interior of $\mathcal{I}$, and let $\eta_{n,T}$ minimize $\phi_{n,T}$ over $\mathcal{I}$.

Define the loss gradient $G_{T}(\eta;\gamma):=\partial_{\eta}\Phi_{T}(\eta;\gamma)$. At the moving infinite-width optimizer, suppose $G_{T}(\eta_{\infty,T};\gamma_{\infty})=0$ and $G_{T}(\eta_{\infty,T};\gamma_{n})=O_{p}(\epsilon_{n}\rho_{G,T})$, and assume the first-order loss expansion
\begin{equation}\label{eq:phi-Frechet-expansion-general}
    \begin{aligned}
        &d_{n,T}
        :=\phi_{n,T}(\eta_{\infty,T})-\phi_{\infty,T}(\eta_{\infty,T})
        =D_{\gamma}\Phi_{T}(\eta_{\infty,T};\gamma_{\infty})[\gamma_{n}-\gamma_{\infty}]+o_{p}(\epsilon_{n}\rho_{L,T})
        ,\\
        &|d_{n,T}|=\Theta_{p}(\epsilon_{n}\rho_{L,T})
        .
    \end{aligned}
\end{equation}
Here $\rho_{L,T}$ and $\rho_{G,T}$ describe the loss and loss-gradient responses to the finite-width perturbation, respectively. The lower bound in \eqref{eq:phi-Frechet-expansion-general} holds, for example, if the normalized first-order loss term converges to a Gaussian with positive variance.

Set $\kappa_{T}:=\phi_{\infty,T}''(\eta_{\infty,T})>0$ and $J_{T}:=[\eta_{\infty,T}-r_{T},\eta_{\infty,T}+r_{T}]\subset\operatorname{int}(\mathcal{I})$. Suppose first that optimizer localization $|\eta_{n,T}-\eta_{\infty,T}|=o_{p}(r_{T})$ has been established independently. The argument below is conditional on this localization; the scale comparison derived below is a consistency requirement for the local expansion. Writing $\underline{\kappa}_{n,T}:=\inf_{\eta\in J_{T}}\partial_{\eta}G_{T}(\eta;\gamma_{n})$, assume $\mathrm{P}(\underline{\kappa}_{n,T}>0)\to1$, $\kappa_{T}/\underline{\kappa}_{n,T}=O_{p}(1)$, and $\sup_{\eta\in J_{T}}\{|\phi_{n,T}''(\eta)|+|\phi_{\infty,T}''(\eta)|\}=O_{p}(\kappa_{T})$. These conditions ensure positive finite-width curvature and uniform local control. With probability tending to one, both optimizers are interior, and the mean value theorem gives
\[
    b_{n,T}
    =|\eta_{n,T}-\eta_{\infty,T}|
    =|G_{T}(\eta_{\infty,T};\gamma_{n})/\partial_{\eta}G_{T}(\tilde{\eta}_{n,T};\gamma_{n})|
    =O_{p}(\epsilon_{n}\rho_{G,T}/\kappa_{T})
    ,
\]
where $\tilde{\eta}_{n,T}$ lies between the two optimizers. The bound is compatible with the assumed localization whenever $\epsilon_{n}\rho_{G,T}/\kappa_{T}=o(r_{T})$.

Taylor expansion of the infinite-width objective then yields
\[
    c_{n,T}
    =(1/2)\phi_{\infty,T}''(\bar{\eta}_{n,T})(\eta_{n,T}-\eta_{\infty,T})^{2}
    =O_{p}(\kappa_{T}b_{n,T}^{2})
    =O_{p}(\epsilon_{n}^{2}\rho_{G,T}^{2}/\kappa_{T})
    ,
\]
where $\bar{\eta}_{n,T}$ lies between the two optimizers. For fixed $T$ with nondegenerate local curvature, this is the same local quadratic relation underlying $c_{n}=\Theta_{p}(b_{n}^{2})$ in \citet[Proposition 1]{ghosh2026understanding}. Here we keep the factor $\kappa_{T}$ explicit because the curvature may vanish as $T$ grows.

Similarly, the improvement from optimizing the finite-width objective satisfies $R_{n,T}:=\phi_{n,T}(\eta_{\infty,T})-\phi_{n,T}(\eta_{n,T})=O_{p}(\epsilon_{n}^{2}\rho_{G,T}^{2}/\kappa_{T})$. Since $a_{n,T}=|d_{n,T}-R_{n,T}|$, a sufficient condition for this optimization correction to be negligible relative to the first-order loss fluctuation is $\epsilon_{n}\rho_{G,T}^{2}/(\kappa_{T}\rho_{L,T})\to0$. Under this condition, we have
\[
    a_{n,T}=\Theta_{p}(\epsilon_{n}\rho_{L,T})
    ,\qquad
    c_{n,T}/a_{n,T}
    =O_{p}(\epsilon_{n}\rho_{G,T}^{2}/(\kappa_{T}\rho_{L,T}))
    =o_{p}(1)
    .
\]
The lower bound in \eqref{eq:phi-Frechet-expansion-general} is needed for this comparison; an $O_{p}$ bound on the loss fluctuation alone does not suffice. Figure~\ref{fig:HP-general-rates} illustrates these scales.

\begin{figure}[t]
    \centering
    \includegraphics[width=\linewidth]{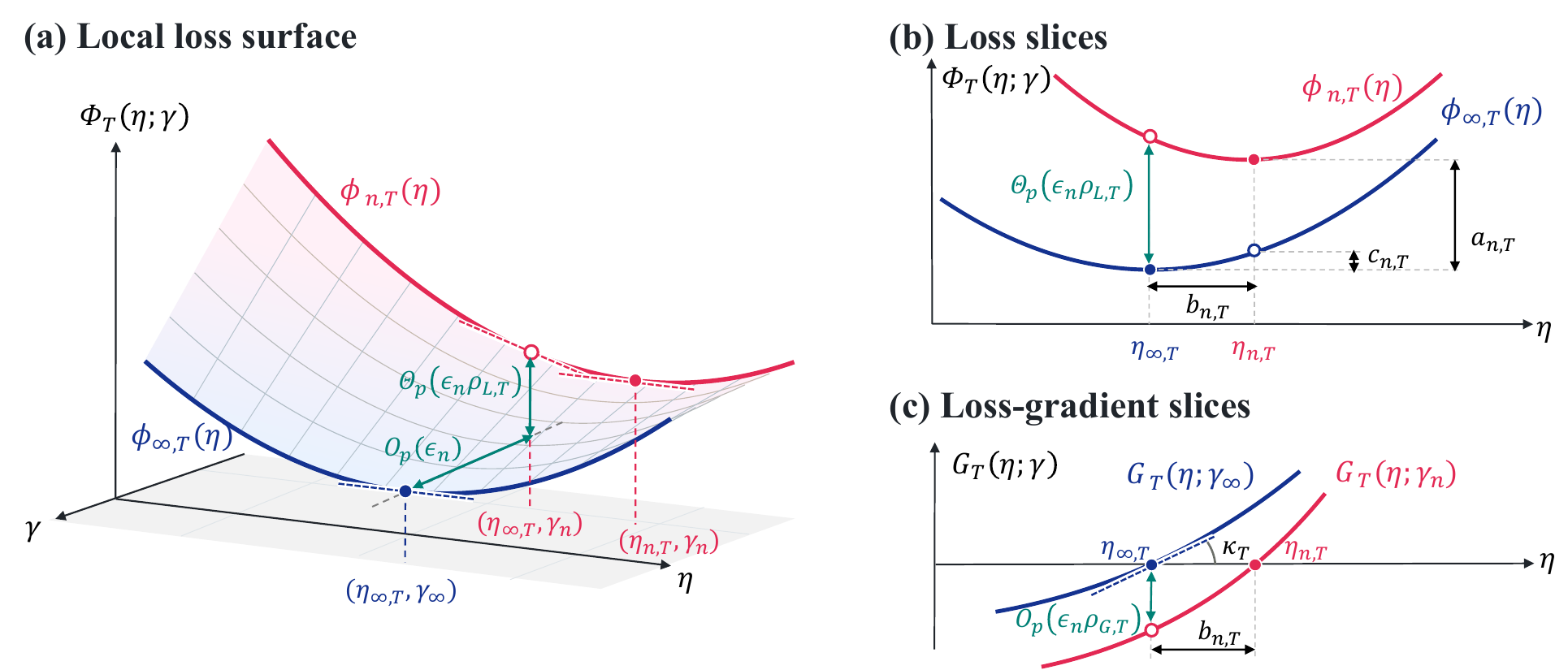}
    \caption{Illustration of the growing-horizon rate decomposition. \textbf{(a)} The local loss surface $\Phi_{T}(\eta;\gamma)$. A perturbation along the $\gamma$-axis from $\gamma_{\infty}$ to $\gamma_{n}$, of size $O_{p}(\epsilon_{n})$, affects both the loss value and its gradient with respect to $\eta$. \textbf{(b)} Loss slices at $\gamma_{\infty}$ and $\gamma_{n}$. At $\eta_{\infty,T}$, the induced loss perturbation has scale $\Theta_{p}(\epsilon_{n}\rho_{L,T})$, which determines the leading scale of $a_{n,T}$. \textbf{(c)} Loss-gradient slices. The finite-width perturbation changes the gradient at $\eta_{\infty,T}$ by $O_{p}(\epsilon_{n}\rho_{G,T})$. Local curvature of order $\kappa_{T}$ converts this gradient perturbation into a hyperparameter gap $b_{n,T}$.}
    \label{fig:HP-general-rates}
\end{figure}

In our model, the correspondence is $(\gamma_{n},\gamma_{\infty})=((\chi_{n}^{(0)},A_{n}),(-y,K))$, $G_{T}=-(T/m^{2})F_{T}$, $\epsilon_{n}=n^{-1/2}$, $\rho_{L,T}\asymp Tq_{*}^{2T-1}$, $\rho_{G,T}\asymp T^{2}q_{*}^{2T-2}$, $\kappa_{T}\asymp T^{2}q_{*}^{2T-2}$, and $r_{T}\asymp T^{-1}$. Consequently, we have $\epsilon_{n}\rho_{G,T}/(\kappa_{T}r_{T})\asymp T/\sqrt{n}$ and $\epsilon_{n}\rho_{G,T}^{2}/(\kappa_{T}\rho_{L,T})\asymp T/(q_{*}\sqrt{n})$. Since $q_{*}>0$ is fixed, $T=o(\sqrt{n})$ makes both ratios vanish. Thus the general formulation recovers the rate conclusions of Theorem~\ref{thm:a-b-c-summary} and the fast-transfer conclusion of Corollary~\ref{cor:a-b-c-summary}. The fixed-horizon proof of Theorem~\ref{thm:fixed-horizon-transfer} uses the same local structure, with horizon-dependent scales fixed as $n\to\infty$.
The model-specific estimates underlying these scales differ substantially from those in \citet{wen2026fast}. Here the perturbation is finite-dimensional and the relevant matrix powers admit uniform control near the optimum, whereas their source-capacity setting involves a spectrum accumulating at zero and an optimizer approaching the stability edge, requiring spectral-tail-dependent localization and fluctuation estimates.

\section{Implications for mechanisms of learning-rate transfer}\label{sec:transfer-mechanisms}

\paragraph{Relation to the top-$k$ loss decomposition.}
\citet{ghosh2026understanding} decompose the linearized loss change along an optimization trajectory into a dominant top-$k$ component and a residual component, and conjecture that fast transfer arises when the dominant component becomes rapidly width-stable while the width-sensitive residual has little effect on the hyperparameter optimum. In our model, the full-batch GD gradient with respect to $W_{1}$ has rank at most one. Since the update is proportional to this gradient, the associated gradient--update alignment matrix also has rank at most one, so its stepwise linearized loss change is captured by $k=1$. This exact top-$1$ representation, however, is only a structural statement: it holds both in the fast-transfer regime of Theorem~\ref{thm:a-b-c-summary} and in Proposition~\ref{prop:isotropic-no-fast-transfer}, where $c_{n,T}/a_{n,T}=\Theta_{p}(1)$ for every fixed $T\ge1$. Thus, low-rank concentration of the stepwise linearized loss change alone does not determine the transfer rate; the local behavior of the finite-width perturbation around the infinite-width optimum also matters.

Section~\ref{subsec:perturbation-formulation} characterizes this local behavior in terms of three scales: the loss-value response $\rho_{L,T}$, the hyperparameter-gradient response $\rho_{G,T}$, and the local curvature $\kappa_{T}$. Under its assumptions, the leading loss perturbation has scale $\Theta_{p}(\epsilon_{n}\rho_{L,T})$, whereas both the transfer penalty and the improvement from finite-width reoptimization are bounded by $O_{p}(\epsilon_{n}^{2}\rho_{G,T}^{2}/\kappa_{T})$. Hence fast transfer follows when $\epsilon_{n}\rho_{G,T}^{2}/(\kappa_{T}\rho_{L,T})\to0$. The $T=1$ case of Proposition~\ref{prop:isotropic-no-fast-transfer} illustrates why the loss-response term is essential: in this counterexample, the first-order loss response vanishes, while the loss-gradient response remains nondegenerate and $\kappa_{1}>0$, and $a_{n,1}$ and $c_{n,1}$ are both $\Theta_{p}(n^{-1})$. For fixed $T\ge2$, the first-order responses and the curvature all degenerate, placing the example in a higher-order regime outside the framework of Section~\ref{subsec:perturbation-formulation}.

Our growing-horizon limits also clarify the directional dependence of the loss-gradient and loss-value responses in Section~\ref{subsec:perturbation-formulation}. In Theorem~\ref{thm:a-b-c-summary}, the normalized optimizer displacement and optimized-loss fluctuation converge in distribution to nonzero deterministic multiples of $\ell_{\eta}(\Xi)$ and $\ell_{L}(\Xi)$, respectively, where $\ell_{\eta}(E):=u_{1}^{\top}Eu_{1}+u_{r}^{\top}Eu_{r}$ and $\ell_{L}(E):=c_{1}a_{1}^{*}u_{1}^{\top}Eu_{1}-c_{r}a_{r}^{*}u_{r}^{\top}Eu_{r}$. As shown in Section~\ref{subsec:proof-structure}, these limits arise from $-(T/m^{2})\langle Q_{F,T},A_{n}-K\rangle_{\mathrm{F}}$ and $\langle Q_{L,T},A_{n}-K\rangle_{\mathrm{F}}$, which are the leading terms of the respective first-order responses at $\eta_{\infty,T}$, whose scales are $\epsilon_{n}\rho_{G,T}$ and $\epsilon_{n}\rho_{L,T}$. Although $\ell_{\eta}$ and $\ell_{L}$ depend on the same two endpoint coordinates, they are not proportional. Hence there are perturbation directions with a nonzero leading loss response but a vanishing leading loss-gradient response, and therefore no leading-order optimizer displacement along those directions. This is qualitatively consistent with the idea that width-dependent loss changes can have little effect on hyperparameter selection, as in the mechanism proposed by \citet{ghosh2026understanding}; our argument concerns local responses to $A_{n}-K$, rather than a decomposition of the optimization trajectory.

\paragraph{Relation to spectral explanations of transfer.}
Our results show that, even in this solvable setting, the width consistency of a single leading eigenvalue does not determine the long-horizon optimal learning rate. \citet{noci2024super} study the width consistency of leading parameter-space Hessian eigenvalues, particularly sharpness. Our result highlights a distinction between the stability boundary and the optimal learning rate. Under Assumption~\ref{asmp:endpoint-support}, Theorem~\ref{thm:finite-minimax} implies $\eta_{\infty,T}\to2m/(\lambda_{1}+\lambda_{r})<2m/\lambda_{1}$ as $T\to\infty$. Thus the largest eigenvalue determines the stability boundary, whereas both spectral endpoints determine the long-horizon optimum.

\section{Conclusion}

We studied fast learning-rate transfer across width in a shallow linear network under $\mu$P at growing training horizons. Specifically, we characterized the growing-horizon infinite-width optimizer through the active spectral endpoints and showed that, under endpoint activity and $T=o(\sqrt{n})$, the three transfer quantities obey horizon-dependent rates that yield $c_{n,T}=o_{p}(a_{n,T})$. The local perturbation formulation organizes transfer in terms of the loss-value response, loss-gradient response, and local curvature, providing a sensitivity-based explanation for fast transfer and for the dependence of its rates on the spectral structure of the problem.

Our analysis is restricted to a solvable setting of linear networks with a single trainable linear layer. Extending the analysis to nonlinear or deep networks is therefore an important next step, as is studying transfer simultaneously across multiple scaling axes such as width and depth. 
Another limitation of the current framework is its reliance on nondegenerate first-order sensitivities: the fixed-horizon counterexample shows that these first-order responses can vanish, requiring a higher-order analysis. Developing such a theory would help provide a clearer picture of the mechanisms governing hyperparameter transfer.

\section*{Acknowledgments}

We thank Denny Wu for valuable discussions. Mana Sakai was supported by RIKEN Junior Research Associate Program and JST FOREST (Grant No. JPMJFR216I). Masaaki Imaizumi was supported by JSPS KAKENHI (Grant No. 24K02904), JST CREST (Grant No. JPMJCR21D2), JST FOREST (Grant No. JPMJFR216I), and JST BOOST (Grant No. JPMJBY24A9).

\clearpage

\appendix

\section{Additional related work}\label{app:additional-related-work}

\paragraph{Parameterization and the empirical scope of transfer.}
Successful transfer depends on more than the name of a parameterization. \citet{lingle2024empirical} examines $\mu$-transfer under practical training choices, while \citet{everett2024scaling} study alignment assumptions and optimizer-dependent scaling exponents. In particular, the latter show that alternative parameterizations can also support transfer when equipped with suitable layerwise learning rates. These findings concern the design and empirical effectiveness of scaling prescriptions. Our analysis instead fixes the initialization and update rule, and compares the resulting optimizer mismatch, optimized loss fluctuation, and transfer penalty.

\paragraph{Conditional separation-of-scales explanations.}
\citet{hong2025provable} study a separation between macroscopic and microscopic quantities under a width-dominance regime, with implications for early-stage hyperparameter selection. This provides a different perspective on the stability of tuning across scale. Our results directly analyze the exact terminal training loss in a specified finite-dimensional model and track the dependence of its local perturbation scales on a growing training horizon.

\paragraph{Transfer across depth, modules, and architectural changes.}
Depthwise transfer has been investigated through infinite-depth limits and residual-network dynamics \citep{yang2024tensor6,bodelon2024depthwise}, and through complete feature learning in deep Transformers \citep{dey2025dont}. \citet{mlodozeniec2026completed} further address transfer across modules, width, depth, batch size, and duration. Other work identifies architecture-specific scaling issues arising from large vocabularies or normalized Transformer parameterizations \citep{hayou2025optimal,shigida2026learning}. These studies address scaling axes and training mechanisms outside the scope of our model.

\paragraph{Training schedules and scaling laws.}
\citet{bordelon2026theory} analyze terminal-loss optimization over learning-rate schedules in a random feature model. More broadly, empirical and dynamical scaling-law studies relate performance to model size, data, and training time \citep{hoffmann2022empirical,bordelon2024dynamical}. Our question is narrower: we transfer a constant learning rate across width while holding the source and target horizons equal, and quantify the resulting relative penalty.

\paragraph{Spectral explanations and sharpness.}
\citet{noci2024super} connect transfer to the width consistency of leading Hessian eigenvalues, while \citet{lauditi2026spectral} study spectral dynamics and feature learning in deep networks. These perspectives are related to the edge-of-stability behavior documented by \citet{cohen2021gradient}. In our model, the largest eigenvalue determines the linear stability boundary, whereas the long-horizon optimal constant learning rate balances two active spectral endpoints.

\section{Numerical experiments}\label{app:numerical-experiments}

\paragraph{Setup.}
We use the model, initialization, and full-batch training rule of
Section~\ref{subsec:model}, with $m=d=4$ and fixed data
\[
    K
    =\operatorname{diag}(7,4,2,1)
    ,\qquad
    X
    =2K^{1/2}
    ,\qquad
    y
    =(1,1,1,1)^{\top}
    .
\]
The learning-rate interval is $I=[0.05,1.12]$. Thus Assumption~\ref{asmp:endpoint-support} holds with $\eta_{*}=1$ and $q_{*}=3/4$. At each width $n\in\{2^{8},2^{9},\ldots,2^{20}\}$, we draw 256 independent initializations per experiment. Figure~\ref{fig:numerical-transfer} compares the fixed-horizon regime $T=10$ with the growing-horizon regime $T_{n}=\lceil2n^{0.15}\rceil$, which ranges from 5 to 16.

\paragraph{Computation and estimation.}
For growing horizons, define
\[
    \tilde{a}_{n,T}=\frac{a_{n,T}}{Tq_{*}^{2T-1}},
    \qquad
    \tilde{c}_{n,T}=\frac{c_{n,T}}{T^{2}q_{*}^{2T-2}}.
\]
We estimate unweighted least-squares slopes by regressing the log empirical means on $\log n$ over $n=2^{13},\ldots,2^{20}$: the largest eight widths for fixed $T$, and the subset satisfying $T_{n}/\sqrt{n}\le0.1$ for the growing schedule.

\begin{figure}[t]
    \centering
    \includegraphics[width=\linewidth]{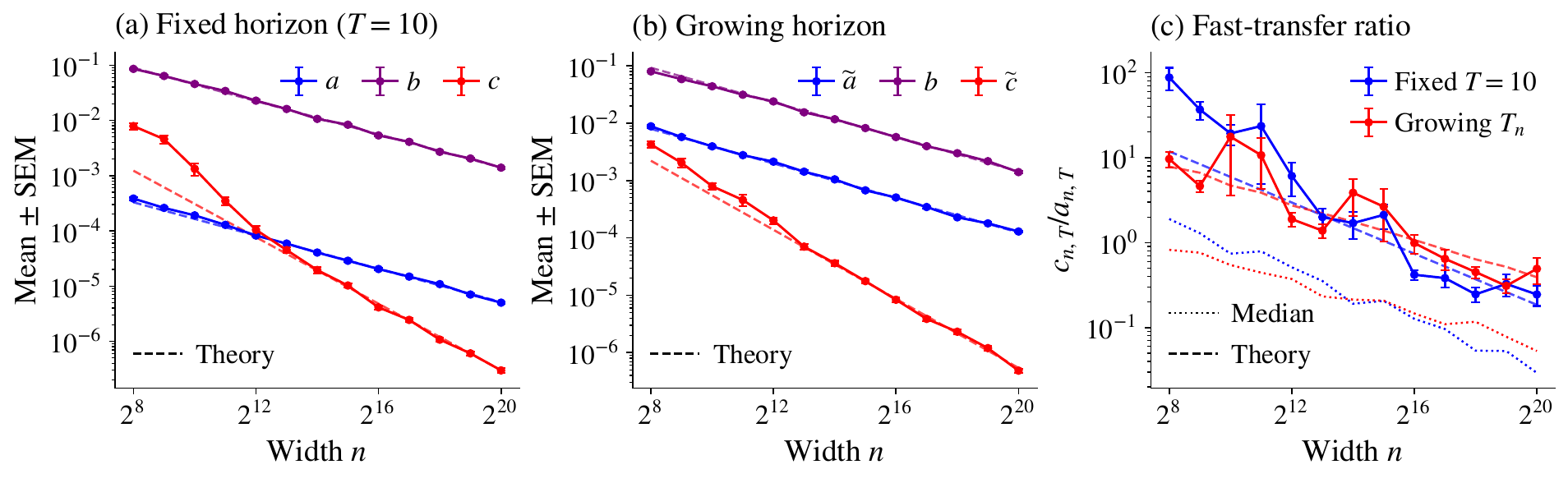}
    \caption{Numerical learning-rate transfer. \textbf{(a)} $a_{n,T},b_{n,T},c_{n,T}$ at fixed $T=10$. \textbf{(b)} $\tilde{a}_{n,T},b_{n,T},\tilde{c}_{n,T}$ at $T_{n}=\lceil2n^{0.15}\rceil$. \textbf{(c)} Per-initialization ratios $c_{n,T}/a_{n,T}$, summarized by means (solid) and medians (dotted); the reference rates are $n^{-1/2}$ for fixed $T$ and $T/(q_{*}\sqrt{n})$ for growing $T$.}
    \label{fig:numerical-transfer}
\end{figure}

\paragraph{Results and implications.}
In panels~(a) and~(b) of Figure~\ref{fig:numerical-transfer}, all six slopes are close to the reference exponents $(-1/2,-1/2,-1)$, consistent with Theorems~\ref{thm:fixed-horizon-transfer} and~\ref{thm:a-b-c-summary}. The overall decline of $c_{n,T}/a_{n,T}$ in panel~(c) supports fast transfer along the tested schedule and illustrates the quadratic transfer penalty discussed in Section~\ref{subsec:perturbation-formulation}. Mean ratios are noisy and nonmonotone because small $a_{n,T}$ values can dominate individual ratios; medians show a smoother decline. The theoretical rates are in probability and do not, by themselves, imply corresponding rates for expectations.

\section{Basic properties and finite-width fluctuations}

\subsection{Proof of Proposition~\ref{prop:dynamics}}

First, we compute
\[
    \nabla_{W_{1}}\mathcal{L}_{n}(W_{1}^{(t)})
    =\frac{1}{m}\sum_{i=1}^{m}\chi_{n,i}^{(t)}VH_{n,i\cdot}
    =\frac{1}{m}V(\chi_{n}^{(t)})^{\top}H_{n}
    .
\]
Thus we have $W_{1}^{(t+1)}=W_{1}^{(t)}-\frac{\eta}{m}V(\chi_{n}^{(t)})^{\top}H_{n}$, which implies
\[
    f_{n}^{(t+1)}(X)
    =H_{n}(W_{1}^{(t)})^{\top}V-\frac{\eta}{m}H_{n}H_{n}^{\top}\chi_{n}^{(t)}V^{\top}V
    =f_{n}^{(t)}(X)-\frac{\eta}{m}A_{n}\chi_{n}^{(t)}
    .
\]
Subtracting $y$ from both sides yields
\[
    \chi_{n}^{(t+1)}
    =B_{n}(\eta)\chi_{n}^{(t)}
    ,
\]
which implies
\[
    \chi_{n}^{(T)}
    =B_{n}(\eta)^{T}\chi_{n}^{(0)}
    .
\]
Therefore we have
\[
    \phi_{n,T}(\eta)
    =\mathcal{L}_{n}(W_{1}^{(T)})
    =\frac{1}{2m}\|\chi_{n}^{(T)}\|^{2}
    =\frac{1}{2m}\left\|B_{n}(\eta)^{T}\chi_{n}^{(0)}\right\|^{2}
    =\frac{1}{2m}(\chi_{n}^{(0)})^{\top}B_{n}(\eta)^{2T}\chi_{n}^{(0)}
    .
\]
\hfill\qedsymbol

\subsection{Proof of Lemma~\ref{lem:joint-clt}}

First, we show the weak convergence of the marginal distribution $\sqrt{n}(A_{n}-K)\stackrel{d}{\longrightarrow}\Xi$. Noting that $A_{n}=(n\|V\|^{2})X(n^{-1}W_{0}^{\top}W_{0})X^{\top}$ holds, we decompose
\begin{align*}
    A_{n}-K
    &=(n\|V\|^{2}-1)K+X(n^{-1}W_{0}^{\top}W_{0}-d^{-1}I_{d})X^{\top}\\
    &\quad+(n\|V\|^{2}-1)X(n^{-1}W_{0}^{\top}W_{0}-d^{-1}I_{d})X^{\top}
    .
\end{align*}
Define $\alpha_{n}=\sqrt{n}(n\|V\|^{2}-1)$ and $G_{n}=\sqrt{n}(n^{-1}W_{0}^{\top}W_{0}-d^{-1}I_{d})$.
Since the variables $nV_{j}$ are i.i.d. $\mathcal{N}(0,1)$, the central limit theorem gives
\[
    \alpha_{n}
    =\frac{1}{\sqrt{n}}\sum_{j=1}^{n}((nV_{j})^{2}-1)
    \stackrel{d}{\longrightarrow}\alpha
    ,\qquad
    \alpha\sim N(0,2)
    .
\]
Next, noting that $(W_{0,j\cdot})^{\top}\sim N(0,d^{-1}I_{d})$ independently, the multivariate central limit theorem gives
\[
    G_{n}
    =\frac{1}{\sqrt{n}}\sum_{j=1}^{n}((W_{0,j\cdot})^{\top}W_{0,j\cdot}-d^{-1}I_{d})
    \stackrel{d}{\longrightarrow}G
    ,
\]
where $G$ is the Gaussian matrix specified in the statement. Since $V$ and $W_{0}$ are independent, $\alpha_{n}$ and $G_{n}$ are independent for every $n$, and hence the joint convergence in distribution holds:
\[
    (\alpha_{n},G_{n})
    \stackrel{d}{\longrightarrow}
    (\alpha,G)
    .
\]
Combining the results, we have
\begin{equation}\label{eq:AToK}
    \sqrt{n}(A_{n}-K)
    =\alpha_{n}K+XG_{n}X^{\top}+O_{p}(n^{-1/2})
    \stackrel{d}{\longrightarrow}\alpha K+XGX^{\top}
    =\Xi
    .
\end{equation}

Fix $u\in\mathbb{R}^{m}$ and $T\in\mathbb{R}^{m\times m}$. Note that $\chi_{n}^{(0)}+y=f_{n}^{(0)}(X)=H_{n}(W_{1}^{(0)})^{\top}V$ holds. Since $W_{1}^{(0)}$ is independent of $(W_{0},V)$, we have
\begin{equation}\label{eq:f_n-conditional-dist}
    \sqrt{n}(\chi_{n}^{(0)}+y)\mid W_{0},V
    \sim N(0,A_{n})
    ,
\end{equation}
which implies
\[
    \mathbb{E}[\exp(iu^{\top}\sqrt{n}(\chi_{n}^{(0)}+y))\mid W_{0},V]
    =\exp\left(-\frac{1}{2}u^{\top}A_{n}u\right)
    .
\]
Thus we have
\begin{align*}
    &\mathbb{E}[\exp(iu^{\top}\sqrt{n}(\chi_{n}^{(0)}+y)+i\langle{T},{\sqrt{n}(A_{n}-K)}\rangle_{\mathrm{F}})]\\
    &=\mathbb{E}(\mathbb{E}[\exp(iu^{\top}\sqrt{n}(\chi_{n}^{(0)}+y)+i\langle{T},{\sqrt{n}(A_{n}-K)}\rangle_{\mathrm{F}})\mid W_{0},V])\\
    &=\mathbb{E}\left[\exp\left(-\frac{1}{2}u^{\top}A_{n}u\right)\exp(i\langle{T},{\sqrt{n}(A_{n}-K)}\rangle_{\mathrm{F}})\right]\\
    &=\mathbb{E}\left[\left(\exp\left(-\frac{1}{2}u^{\top}A_{n}u\right)-\exp\left(-\frac{1}{2}u^{\top}Ku\right)\right)\exp\left(i\langle T,\sqrt{n}(A_{n}-K)\rangle_{\mathrm{F}}\right)\right]\\
    &\quad+\exp\left(-\frac{1}{2}u^{\top}Ku\right)\mathbb{E}[\exp(i\langle{T},{\sqrt{n}(A_{n}-K)}\rangle_{\mathrm{F}})]
    .
\end{align*}
Note that \eqref{eq:AToK} implies $A_{n}\stackrel{p}{\longrightarrow}K$. Thus we compute the limit of the first term as
\begin{align*}
    &\left|\mathbb{E}\left[\left(\exp\left(-\frac{1}{2}u^{\top}A_{n}u\right)-\exp\left(-\frac{1}{2}u^{\top}Ku\right)\right)\exp\left(i\langle T,\sqrt{n}(A_{n}-K)\rangle_{\mathrm{F}}\right)\right]\right|\\
    &\le\mathbb{E}\left|\exp\left(-\frac{1}{2}u^{\top}A_{n}u\right)-\exp\left(-\frac{1}{2}u^{\top}Ku\right)\right|\\
    &\to0
    .
\end{align*}
For the second term, \eqref{eq:AToK} implies
\[
    \mathbb{E}[\exp(i\langle T,\sqrt{n}(A_{n}-K)\rangle_{\mathrm{F}})]
    \to\mathbb{E}[\exp(i\langle T,\Xi\rangle_{\mathrm{F}})]
    .
\]
Hence the joint characteristic function satisfies
\[
    \mathbb{E}[\exp(iu^{\top}\sqrt{n}(\chi_{n}^{(0)}+y)+i\langle{T},{\sqrt{n}(A_{n}-K)}\rangle_{\mathrm{F}})]
    \to\exp\left(-\frac{1}{2}u^{\top}Ku\right)\mathbb{E}[\exp(i\langle T,\Xi\rangle_{\mathrm{F}})]
    ,
\]
which is the characteristic function of $(Z,\Xi)$ with $Z\sim N(0,K)$ independent of $\Xi$.
\hfill\qedsymbol

\subsection{Infinite-width loss for fixed horizon}

\begin{corollary}[Infinite-width loss for fixed $T$]\label{cor:fixed-T-limit-loss}
    Fix $T\in\mathbb{N}$ and $\eta\in\mathbb{R}$.  Then
    \[
        \phi_{n,T}(\eta)
        \stackrel{p}{\longrightarrow}\phi_{\infty,T}(\eta)
        :=\frac{1}{2m}\left\|B_{\infty}(\eta)^{T}y\right\|^{2}
        =\frac{1}{2m}y^{\top}B_{\infty}(\eta)^{2T}y
        .
    \]
    The convergence is uniform on every deterministic compact interval of learning rates.
\end{corollary}
\begin{proof}
    By Lemma~\ref{lem:joint-clt}, $(\chi_{n}^{(0)},A_{n})\stackrel{p}{\longrightarrow}(-y,K)$.  For fixed $T$, the map
    \[
        (\chi,A,\eta)
        \longmapsto
        \frac{1}{2m}\chi^{\top}\left(I_{m}-\frac{\eta}{m}A\right)^{2T}\chi
    \]
    is a polynomial in the entries of $(\chi,A,\eta)$. Thus, the continuous mapping theorem proves $\phi_{n,T}(\eta)\stackrel{p}{\longrightarrow}\phi_{\infty,T}(\eta)$ for every fixed $\eta$. Furthermore, for fixed $T$, both $\phi_{n,T}$ and $\phi_{\infty,T}$ are polynomials in $\eta$ of degree at most $2T$. Their coefficients $c_{n,0},c_{n,1},\dots,c_{n,2T}$ are continuous functions of $(\chi_{n}^{(0)},A_{n})$ and converge in probability to the corresponding coefficients $c_{\infty,0},c_{\infty,1},\dots,c_{\infty,2T}$ determined by $(-y,K)$. Thus, for any compact interval $[\eta_{l},\eta_{h}]$ with $|\eta_{l}|,|\eta_{h}|\le C$, we have
    \begin{align*}
        &\sup_{\eta\in[\eta_{l},\eta_{h}]}|\phi_{n,T}(\eta)-\phi_{\infty,T}(\eta)|
        \le\sup_{\eta\in[\eta_{l},\eta_{h}]}\sum_{k=0}^{2T}|c_{n,k}-c_{\infty,k}|C^{k}\\
        &\le\max_{k\in\{0,1,\dots,2T\}}|c_{n,k}-c_{\infty,k}|\sum_{k=0}^{2T}C^{k}
        \stackrel{p}{\longrightarrow}0
        .
    \end{align*}
\end{proof}

\subsection{Range and null-space geometry}

\begin{lemma}\label{lem:kernel-H_n-A_n-K}
    When $n\ge d$, we have
    \[
        \ker(H_{n}^{\top})
        \stackrel{\mathrm{a.s.}}{=}\ker(A_{n})
        \stackrel{\mathrm{a.s.}}{=}\ker(K)
        .
    \]
\end{lemma}
\begin{proof}
    We first show that $\ker(H_{n}^{\top})\stackrel{\mathrm{a.s.}}{=}\ker(A_{n})$ holds. If $H_{n}^{\top}z=0$, then $A_{n}z=0$ follows immediately. Thus, we have $\ker(H_{n}^{\top})\subset\ker(A_{n})$. On the other hand, if $A_{n}z=0$, then $z^{\top}A_{n}z=\|V\|^{2}\|H_{n}^{\top}z\|^{2}=0$. Since $\|V\|^{2}>0$ with probability one, $H_{n}^{\top}z=0$ almost surely. Thus, with probability one, we have $\ker(A_{n})\subset\ker(H_{n}^{\top})$.

    When $n\ge d$, the Gaussian matrix $W_{0}$ has rank $d$ almost surely. Thus,
    \[
        \operatorname{range}(H_{n})
        =\operatorname{range}(XW_{0}^{\top})
        \stackrel{\mathrm{a.s.}}{=}\operatorname{range}(X)
        =\operatorname{range}(K)
        .
    \]
    This implies $\ker(H_{n}^{\top})=\operatorname{range}(H_{n})^{\perp}\stackrel{\mathrm{a.s.}}{=}\operatorname{range}(K)^{\perp}=\ker(K)$.
\end{proof}

\subsection{Optimization landscape}

\begin{proposition}[Derivatives of the finite-horizon objective]\label{prop:loss-derivatives}
    For every $T\ge1$, we have
    \begin{align}
        &\frac{\partial}{\partial\eta}\Phi_{T}(\eta;\chi,A)
        =-\frac{T}{m^{2}}F_{T}(\eta;\chi,A)
        ,\label{eq:first-derivative}\\
        &\frac{\partial^{2}}{\partial\eta^{2}}\Phi_{T}(\eta;\chi,A)
        =\frac{T(2T-1)}{m^{3}}\chi^{\top}A^{2}B_{A}(\eta)^{2T-2}\chi
        \ge0
        ,\label{eq:second-derivative}\\
        &\frac{\partial}{\partial\eta}F_{T}(\eta;\chi,A)
        =-\frac{2T-1}{m}\chi^{\top}A^{2}B_{A}(\eta)^{2T-2}\chi
        \le0
        .\label{eq:score-derivative}
    \end{align}
\end{proposition}
\begin{proof}
    Since the matrices $A$ and $B_{A}(\eta)$ commute, we have
    \[
        \frac{d}{d\eta}B_{A}(\eta)^{k}
        =-\frac{k}{m}AB_{A}(\eta)^{k-1}
        .
    \]
    Thus
    \begin{align*}
        &\frac{\partial}{\partial\eta}\Phi_{T}(\eta;\chi,A)
        =-\frac{1}{2m}\chi^{\top}\frac{2T}{m}AB_{A}(\eta)^{2T-1}\chi
        =-\frac{T}{m^{2}}F_{T}(\eta;\chi,A)
        ,\\
        &\frac{\partial}{\partial\eta}F_{T}(\eta;\chi,A)
        =-\chi^{\top}A\frac{2T-1}{m}AB_{A}(\eta)^{2T-2}\chi
        =-\frac{2T-1}{m}\chi^{\top}A^{2}B_{A}(\eta)^{2T-2}\chi
        .
    \end{align*}
\end{proof}

\begin{corollary}[Existence and uniqueness of the optimal learning rate]\label{cor:optimizer-characterization}
    Suppose $A\in\mathbb{R}_{\mathrm{sym}}^{m\times m}$, $\chi\in\mathbb{R}^{m}$, and $A\succeq0$. If $A\chi\neq0$, then $\eta\mapsto\Phi_{T}(\eta;\chi,A)$ is strictly convex and coercive on $\mathbb{R}$, that is, $\Phi_{T}(\eta;\chi,A)\to\infty$ as $|\eta|\to\infty$. Hence it has a unique global minimizer $\eta_{T}(\chi,A)>0$, characterized by $F_{T}(\eta_{T}(\chi,A);\chi,A)=0$. On a closed interval $[a,b]$, the optimizer is $\eta_{T}(\chi,A)$ if $\eta_{T}(\chi,A)\in[a,b]$, and is the appropriate boundary point otherwise. If $A\chi=0$, then $\Phi_{T}(\eta;\chi,A)$ is constant in $\eta$, and every $\eta\in\mathbb{R}$ is a global minimizer.
\end{corollary}
\begin{proof}
    If $A\chi=0$, then $B_{A}(\eta)\chi=\chi$ holds for every $\eta\in\mathbb{R}$. Hence $\Phi_{T}(\eta;\chi,A)$ is independent of $\eta$, which proves the degenerate case.

    In the remainder of the proof, suppose $A\chi\neq0$. Let $A=U\operatorname{diag}(\mu_{1},\ldots,\mu_{m})U^{\top}$ be the eigendecomposition of $A$, and let $q_{j}=(U_{\cdot j})^{\top}\chi$. Define the active index set
    \[
        \mathcal{J}
        =\{j\in[m]:\mu_{j}>0\ \text{and}\ q_{j}\neq0\}
        .
    \]
    Since $A\chi=\sum_{j=1}^{m}\mu_{j}q_{j}U_{\cdot j}$, the assumption $A\chi\neq0$ is equivalent to $\mathcal{J}\neq\emptyset$.

    Proposition~\ref{prop:loss-derivatives} gives
    \[
        \Phi_{T}(\eta;\chi,A)
        =\frac{1}{2m}\sum_{j=1}^{m}q_{j}^{2}\left(1-\frac{\eta\mu_{j}}{m}\right)^{2T}
        =\sum_{j=1}^{m}g_{j}(\eta)
        ,
    \]
    where we defined
    \[
        g_{j}(\eta)
        =\frac{q_{j}^{2}}{2m}\left(1-\frac{\eta\mu_{j}}{m}\right)^{2T}
        .
    \]
    Here, each $g_{j}$ corresponding to $j\notin\mathcal{J}$ is either identically zero or constant in $\eta$. For each $j\in\mathcal{J}$, we have
    \[
        g_{j}'(\eta)
        =-\frac{Tq_{j}^{2}\mu_{j}}{m^{2}}\left(1-\frac{\eta\mu_{j}}{m}\right)^{2T-1}
        .
    \]
    The map $\eta\mapsto g_{j}'(\eta)$ is strictly increasing; thus, $g_{j}$ is strictly convex on $\mathbb{R}$. Since $\mathcal{J}$ is nonempty, $\Phi_{T}(\eta;\chi,A)$ is strictly convex. Moreover, for any fixed $j\in\mathcal{J}$,
    \[
        g_{j}(\eta)
        \to\infty
        \qquad(|\eta|\to\infty)
        .
    \]
    Since $\Phi_{T}(\eta;\chi,A)\ge g_{j}(\eta)$, it follows that
    \[
        \Phi_{T}(\eta;\chi,A)
        \to\infty
        \qquad(|\eta|\to\infty)
        .
    \]
    Thus $\Phi_{T}(\cdot;\chi,A)$ is coercive. In particular, there exists $R>0$ such that
    \[
        \Phi_{T}(\eta;\chi,A)
        >\Phi_{T}(0;\chi,A)
        \qquad(|\eta|\ge R)
        .
    \]
    Since $\Phi_{T}(\cdot;\chi,A)$ is continuous, it attains its minimum on the compact interval $[-R,R]$. This minimum is also global, because every point outside $[-R,R]$ has objective value larger than $\Phi_{T}(0;\chi,A)$. Finally, strict convexity implies that the global minimizer is unique. We denote it by $\eta_{T}(\chi,A)$. The first-order optimality condition and \eqref{eq:first-derivative} give
    \[
        F_{T}(\eta_{T}(\chi,A);\chi,A)
        =0
        .
    \]
    We now verify the positivity of $\eta_{T}(\chi,A)$. At $\eta=0$, \eqref{eq:first-derivative} yields
    \[
        \Phi_{T}'(0;\chi,A)
        =-\frac{T}{m^{2}}\chi^{\top}A\chi
        =-\frac{T}{m^{2}}\|A^{1/2}\chi\|^{2}
        .
    \]
    Note that $A\chi\neq0$ implies $A^{1/2}\chi\neq0$. Consequently, we have
    \[
        \Phi_{T}'(0;\chi,A)
        <0
        =\Phi_{T}'(\eta_{T}(\chi,A);\chi,A)
        .
    \]
    Since $\Phi_{T}(\cdot;\chi,A)$ is strictly convex, its derivative is strictly increasing. Thus we have $\eta_{T}(\chi,A)>0$.
    Finally, consider minimization over a closed interval $[a,b]$ with $a<b$. By Proposition~\ref{prop:loss-derivatives}, we have
    \[
        F_{T}(\eta;\chi,A)
        =\sum_{j\in\mathcal{J}}q_{j}^{2}\mu_{j}
        \left(1-\frac{\eta\mu_{j}}{m}\right)^{2T-1}.
    \]
    For each $j\in\mathcal{J}$, the corresponding summand is strictly decreasing in $\eta$. Hence $F_{T}$ is strictly decreasing on $\mathbb{R}$. Since it is also continuous, the stated boundary characterization follows.
\end{proof}

\begin{proposition}[Almost-sure nondegeneracy of the finite-width optimization problem]\label{prop:loss-convex-as}
    Suppose $K\neq0$ and $n\ge d$. Then we have $A_{n}\chi_{n}^{(0)}\neq0$ almost surely. By Corollary~\ref{cor:optimizer-characterization}, for every $T\ge1$, $\phi_{n,T}$ is almost surely strictly convex and has a unique global minimizer.
\end{proposition}
\begin{proof}
    Let $P$ denote the orthogonal projector onto $\operatorname{range}(K)$. Define $\mathcal{E}_{n}=\{\ker(H_{n}^{\top})=\ker(A_{n})=\ker(K)\}$. Then by Lemma~\ref{lem:kernel-H_n-A_n-K}, the event $\mathcal{E}_{n}$ has probability one. On the event $\mathcal{E}_{n}$, we have $f_{n}^{(0)}(X)\in\operatorname{range}(K)$, and hence $Pf_{n}^{(0)}(X)=f_{n}^{(0)}(X)$. On the same event, $A_{n}\chi_{n}^{(0)}=0$ is equivalent to $\chi_{n}^{(0)}\in\ker(A_{n})=\operatorname{range}(K)^{\perp}$, and therefore to $P\chi_{n}^{(0)}=0$. Recalling the definition $\chi_{n}^{(0)}=f_{n}^{(0)}(X)-y$, this is further equivalent to $f_{n}^{(0)}(X)=Pf_{n}^{(0)}(X)=Py$. Noting that $\mathrm{P}(\mathcal{E}_{n})=1$, we have
    \begin{align*}
        \mathrm{P}(A_{n}\chi_{n}^{(0)}=0)
        =\mathrm{P}(\{A_{n}\chi_{n}^{(0)}=0\}\cap\mathcal{E}_{n})
        =\mathrm{P}(\{f_{n}^{(0)}(X)=Py\}\cap\mathcal{E}_{n})
        =\mathrm{P}(f_{n}^{(0)}(X)=Py)
        .
    \end{align*}
    It therefore suffices to show that $\mathrm{P}(f_{n}^{(0)}(X)=Py)=0$.

    By \eqref{eq:f_n-conditional-dist} in the proof of Lemma~\ref{lem:joint-clt}, we have
    \[
        f_{n}^{(0)}(X)\mid W_{0},V
        \sim\mathcal{N}\left(0,\frac{1}{n}A_{n}\right)
        .
    \]
    Since $K\neq0$, we can choose a fixed vector $u\in\operatorname{range}(K)\setminus\{0\}$. By Lemma~\ref{lem:kernel-H_n-A_n-K}, we have $u\notin\ker(A_{n})$ almost surely. Since $A_{n}\succeq0$, it follows that $u^{\top}A_{n}u>0$ almost surely. Hence,
    \[
        u^{\top}f_{n}^{(0)}(X)\mid W_{0},V
        \sim\mathcal{N}\left(0,\frac{1}{n}u^{\top}A_{n}u\right)
    \]
    is a nondegenerate Gaussian random variable almost surely. Therefore,
    \[
        \mathrm{P}(u^{\top}f_{n}^{(0)}(X)=u^{\top}Py\mid W_{0},V)
        \stackrel{\mathrm{a.s.}}{=}0
        .
    \]
    Since $\{f_{n}^{(0)}(X)=Py\}\subset\{u^{\top}f_{n}^{(0)}(X)=u^{\top}Py\}$, we obtain
    \[
        \mathrm{P}(f_{n}^{(0)}(X)=Py\mid W_{0},V)
        \stackrel{\mathrm{a.s.}}{=}0
        .
    \]
    Taking expectations yields $\mathrm{P}(f_{n}^{(0)}(X)=Py)=0$ as desired.
\end{proof}

\subsection{First-order sensitivities of the score and the loss}

Set $B_{T}=B_{\infty}(\eta_{\infty,T})$, and define
\begin{align}
    \label{eq:loss-derivative-sensitivity-chi}
    g_{F,T}
    &:=-2KB_{T}^{2T-1}y
    ,\\
    \label{eq:loss-derivative-sensitivity-A}
    Q_{F,T}
    &:=\frac{1}{2}\left(B_{T}^{2T-1}yy^{\top}+yy^{\top}B_{T}^{2T-1}\right)-\frac{\eta_{\infty,T}}{2m}\sum_{s=0}^{2T-2}(B_{T}^{2T-2-s}yy^{\top}KB_{T}^{s}+B_{T}^{s}Kyy^{\top}B_{T}^{2T-2-s})
    ,\\
    \label{eq:loss-sensitivity-chi}
    g_{L,T}
    &:=-\frac{1}{m}B_{T}^{2T}y
    ,\\
    \label{eq:loss-sensitivity-A}
    Q_{L,T}
    &:=-\frac{\eta_{\infty,T}}{2m^{2}}\sum_{s=0}^{2T-1}B_{T}^{2T-1-s}yy^{\top}B_{T}^{s}
    .
\end{align}

\begin{lemma}[First-order sensitivities at the infinite-width optimum]\label{lem:frechet}
    For fixed $\eta$, regard $F_{T}(\eta;\chi,A)$ and $\Phi_{T}(\eta;\chi,A)$ as real-valued maps on $\mathbb{R}^{m}\times\mathbb{R}_{\mathrm{sym}}^{m\times m}$, equipped with the product norm $\|(h,E)\|_{\times}=\left(\|h\|^{2}+\|E\|_{\mathrm{F}}^{2}\right)^{1/2}$. Then for $h\in\mathbb{R}^m$ and $E\in\mathbb{R}_{\mathrm{sym}}^{m\times m}$, we have
    \begin{align*}
        &D_{(\chi,A)}F_{T}(\eta_{\infty,T};-y,K)[h,E]
        =g_{F,T}^{\top}h+\langle Q_{F,T},E\rangle_{\mathrm{F}}
        ,\\
        &D_{(\chi,A)}\Phi_{T}(\eta_{\infty,T};-y,K)[h,E]
        =g_{L,T}^{\top}h+\langle Q_{L,T},E\rangle_{\mathrm{F}}
        .
    \end{align*}
\end{lemma}
\begin{proof}
    The derivative of $B_{A}(\eta)$ with respect to $A$ in the direction $E$ is
    \[
        D_{A}B_{A}(\eta)[E]
        =-\frac{\eta}{m}E
        .
    \]
    We also note that by Lemma~\ref{lem:matrix-power-derivative}, we have
    \begin{equation}\label{eq:power-A-derivative}
        D_{A}B_{A}(\eta)^{k}[E]
        =-\frac{\eta}{m}\sum_{s=0}^{k-1}B_{A}(\eta)^{s}EB_{A}(\eta)^{k-1-s}
        .
    \end{equation}
    We first consider the loss $\Phi_{T}(\eta;\chi,A)$. Since $B_{A}(\eta)$ is symmetric for symmetric $A$, differentiation with respect to $\chi$ gives
    \[
        D_{\chi}\Phi_{T}(\eta_{\infty,T};-y,K)[h]
        =-\frac{1}{m}y^{\top}B_{T}^{2T}h
        =g_{L,T}^{\top}h
        .
    \]
    Using \eqref{eq:power-A-derivative} with $k=2T$, we have
    \[
        D_{A}\Phi_{T}(\eta_{\infty,T};-y,K)[E]
        =-\frac{\eta_{\infty,T}}{2m^{2}}\sum_{s=0}^{2T-1}y^{\top}B_{T}^{s} E B_{T}^{2T-1-s}y
        =\langle Q_{L,T},E\rangle_{\mathrm{F}}
        .
    \]
    Combining the two partial derivatives gives the Fr\'echet derivative with respect to the pair $(\chi,A)$. \footnote{To see this explicitly, Lemma~\ref{lem:matrix-power-derivative} implies
    \[
        B_{K+E}(\eta_{\infty,T})^{2T}
        =B_{T}^{2T}+D_{A}B_{K}(\eta_{\infty,T})^{2T}[E]+O(\|E\|_{\mathrm{F}}^{2})
    \]
    for fixed $T$. Substituting this expansion into $\Phi_{T}(\eta_{\infty,T};-y+h,K+E)$ shows
    \begin{align*}
        &|\Phi_{T}(\eta_{\infty,T};-y+h,K+E)-\Phi_{T}(\eta_{\infty,T};-y,K)-D_{\chi}\Phi_{T}(\eta_{\infty,T};-y,K)[h]-D_{A}\Phi_{T}(\eta_{\infty,T};-y,K)[E]|\\
        &=O(\|h\|^{2}+\|h\|\|E\|_{\mathrm{F}}+\|E\|_{\mathrm{F}}^{2})
        =O(\|(h,E)\|_{\times}^{2})
        =o(\|(h,E)\|_{\times})
        .
    \end{align*}
    }
    Next, since $A$ commutes with $B_{A}(\eta)$, the matrix $AB_{A}(\eta)^{2T-1}$ is symmetric. Hence we have
    \[
        D_{\chi} F_{T}(\eta_{\infty,T};-y,K)[h]
        =-2y^{\top}KB_{T}^{2T-1}h
        =g_{F,T}^{\top}h
        .
    \]
    For the derivative with respect to $A$, \eqref{eq:power-A-derivative} gives
    \[
        D_{A}F_{T}(\eta_{\infty,T};-y,K)[E]
        =y^{\top}E B_{T}^{2T-1}y-\frac{\eta_{\infty,T}}{m}\sum_{s=0}^{2T-2}y^{\top}KB_{T}^{s} EB_{T}^{2T-2-s}y
        =\langle Q_{F,T},E\rangle_{\mathrm{F}}
        .
    \]
    Combining the results completes the proof.
\end{proof}

\section{Proof of the fixed-horizon transfer theorem}\label{app:proof-fixed-horizon-transfer}

Throughout this section, $T$ is fixed. Define
\begin{equation}\label{eq:r_j}
    r_{j,T}
    =1-\frac{\eta_{\infty,T}\lambda_{j}}{m}
    \qquad(j\in[r])
    .
\end{equation}

\begin{lemma}\label{lem:fixed-positive-curvature}
    Under the assumptions of Theorem~\ref{thm:fixed-horizon-transfer}, we have for every fixed $T$ that
    \[
        F_{\infty,T}'(\eta_{\infty,T})
        <0
        ,\qquad
        \phi_{\infty,T}''(\eta_{\infty,T})
        >0
        .
    \]
\end{lemma}
\begin{proof}
    Recall that $c_{j}=\|P_{j}y\|^{2}$. Proposition~\ref{prop:loss-derivatives} implies
    \[
        -F_{\infty,T}'(\eta_{\infty,T})
        =\frac{2T-1}{m}y^{\top}K^{2}B_{\infty}(\eta_{\infty,T})^{2T-2}y
        =\frac{2T-1}{m}\sum_{j=1}^{r}c_{j}\lambda_{j}^{2}r_{j,T}^{2T-2}
        .
    \]
    If $T=1$, since $Ky\neq0$, we have $y^{\top}K^{2}B_{\infty}(\eta_{\infty,T})^{2T-2}y>0$. If $T\ge2$, since $|\mathcal{J}_{y}|\ge2$, at least one of $r_{j,T}$ is nonzero. Thus the preceding sum is strictly positive. Hence in all cases, we have $F_{\infty,T}'(\eta_{\infty,T})<0$, and consequently, $\phi_{\infty,T}''(\eta_{\infty,T})=-\frac{T}{m^{2}}F_{\infty,T}'(\eta_{\infty,T})>0$.
\end{proof}

\begin{lemma}\label{lem:score-uniform-fixed-T}
    Since $\eta_{\infty,T}\in(\eta_{\ell},\eta_{u})$, there exists a deterministic $\delta>0$ such that
    \[
        [\eta_{\infty,T}-\delta,\eta_{\infty,T}+\delta]
        \subset(\eta_{\ell},\eta_{u})
        .
    \]
    For such $\delta$, we have
    \begin{align*}
        &\sup_{|\eta-\eta_{\infty,T}|\le\delta}|F_{n,T}(\eta)-F_{\infty,T}(\eta)|
        =O_{p}(n^{-1/2})
        ,\\
        &\sup_{|\eta-\eta_{\infty,T}|\le\delta}|F_{n,T}'(\eta)-F_{\infty,T}'(\eta)|
        =O_{p}(n^{-1/2})
        .
    \end{align*}
\end{lemma}
\begin{proof}
    For fixed $T$, both $F_{T}(\eta;\chi,A)$ and $\partial_{\eta}F_{T}(\eta;\chi,A)$ are polynomials in the entries of $(\eta,\chi,A)$. Hence, on the compact interval $[\eta_{\infty,T}-\delta,\eta_{\infty,T}+\delta]$, their derivatives with respect to $(\chi,A)$ are uniformly bounded on a fixed neighborhood of $(-y,K)$. Define the product norm $\|(\chi,A)\|_{\times}=\left(\|\chi\|^{2}+\|A\|_{\mathrm{F}}^{2}\right)^{1/2}$. Then Lemma~\ref{lem:joint-clt} gives $\|(f_{n}^{(0)}(X),E_{n})\|_{\times}=O_{p}(n^{-1/2})$, where we defined $E_{n}=A_{n}-K$. Thus, noting that $(\chi_{n}^{(0)},A_{n})\stackrel{p}{\longrightarrow}(-y,K)$ holds, we have
    \begin{align*}
        &\sup_{|\eta-\eta_{\infty,T}|\le\delta}|F_{n,T}(\eta)-F_{\infty,T}(\eta)|\\
        &=\sup_{|\eta-\eta_{\infty,T}|\le\delta}\left|\int_{0}^{1}D_{\chi,A}F_{T}(\eta;-y+sf_{n}^{(0)}(X),K+sE_{n})[f_{n}^{(0)}(X),E_{n}]ds\right|\\
        &\le\sup_{|\eta-\eta_{\infty,T}|\le\delta}\int_{0}^{1}\|D_{\chi,A}F_{T}(\eta;-y+sf_{n}^{(0)}(X),K+sE_{n})\|_{\mathrm{op}}\|(f_{n}^{(0)}(X),E_{n})\|_{\times}ds\\
        &=O_{p}(n^{-1/2})
        ,
    \end{align*}
    and, similarly,
    \[
        \sup_{|\eta-\eta_{\infty,T}|\le\delta}|F_{n,T}'(\eta)-F_{\infty,T}'(\eta)|
        =O_{p}(n^{-1/2})
        .
    \]
\end{proof}

\begin{lemma}\label{lem:fixed-optimizer-localization}
    Under the assumptions of Theorem~\ref{thm:fixed-horizon-transfer}, $\eta_{n,T}$ lies in $(\eta_{\ell},\eta_{u})$ with probability tending to one, and we have
    \[
        \eta_{n,T}-\eta_{\infty,T}
        =O_{p}(n^{-1/2})
        .
    \]
\end{lemma}
\begin{proof}
    By Lemmas~\ref{lem:fixed-positive-curvature},~\ref{lem:score-uniform-fixed-T}, and continuity of $F_{\infty,T}'$, there exists a deterministic $\gamma>0$ such that
    \[
        \sup_{|\eta-\eta_{\infty,T}|\le\delta}F_{\infty,T}'(\eta)
        \le-2\gamma
        .
    \]
    Hence we have
    \[
        \mathrm{P}\left(\sup_{|\eta-\eta_{\infty,T}|\le\delta}F_{n,T}'(\eta)
        \le-\gamma\right)
        \to1
        .
    \]

    Since Corollary~\ref{cor:optimizer-characterization} implies $F_{\infty,T}(\eta_{\infty,T})=0$, Lemma~\ref{lem:score-uniform-fixed-T} yields
    \[
        F_{n,T}(\eta_{\infty,T})
        =O_{p}(n^{-1/2})
        .
    \]
    Thus, for fixed $\epsilon>0$, there exists $M_{0}<\infty$ such that
    \[
        \liminf_{n\to\infty}
        \mathrm{P}\left(
        |F_{n,T}(\eta_{\infty,T})|
        \le\frac{M_{0}}{\sqrt{n}}
        \right)
        \ge1-\epsilon
        .
    \]
    Choose $M>M_{0}/\gamma$. For all sufficiently large $n$, we have $M/\sqrt{n}<\delta$. On the intersection of the two preceding high-probability events, the mean value theorem gives
    \begin{align*}
        &F_{n,T}\left(\eta_{\infty,T}-\frac{M}{\sqrt{n}}\right)
        \ge\frac{-M_{0}+\gamma M}{\sqrt{n}}
        >0
        ,\\
        &F_{n,T}\left(\eta_{\infty,T}+\frac{M}{\sqrt{n}}\right)
        \le\frac{M_{0}-\gamma M}{\sqrt{n}}
        <0
        .
    \end{align*}
    Proposition~\ref{prop:loss-convex-as} and Corollary~\ref{cor:optimizer-characterization} imply that, for all sufficiently large $n$, $\phi_{n,T}$ has almost surely a unique global minimizer characterized by the root of $F_{n,T}$. Therefore, on the same event, this root belongs to
    \[
        \left(\eta_{\infty,T}-\frac{M}{\sqrt{n}},\eta_{\infty,T}+\frac{M}{\sqrt{n}}\right)
        \subset(\eta_{\ell},\eta_{u})
        .
    \]
    Thus, with probability at least $1-\epsilon+o(1)$, we have
    \[
        |\eta_{n,T}-\eta_{\infty,T}|
        <\frac{M}{\sqrt{n}}
        .
    \]
    Since $\epsilon>0$ was arbitrary, we obtain
    \[
        \eta_{n,T}-\eta_{\infty,T}
        =O_{p}(n^{-1/2})
        ,
    \]
    and the finite-width optimizer is interior with probability tending to one.
\end{proof}

\begin{theorem}[Limit distribution of the optimal learning rate at fixed $T$]\label{thm:b-limit-dist-fixed-T}
    Under the assumptions of Theorem~\ref{thm:fixed-horizon-transfer}, we have
    \[
        \sqrt{n}(\eta_{n,T}-\eta_{\infty,T})
        \stackrel{d}{\longrightarrow}\Omega_{\eta,T}
        :=-\Omega_{F,T}/F_{\infty,T}'(\eta_{\infty,T})
        ,
    \]
    where $\Omega_{F,T}=g_{F,T}^{\top}Z+\langle Q_{F,T},\Xi\rangle_{F}$.
\end{theorem}
\begin{proof}
    By Lemma~\ref{lem:frechet}, applied at $(\eta_{\infty,T},-y,K)$, we have
    \[
        F_{n,T}(\eta_{\infty,T})
        =g_{F,T}^{\top}f_{n}^{(0)}(X)+\langle Q_{F,T},A_{n}-K\rangle_{\mathrm{F}}+o_{p}(n^{-1/2})
        .
    \]
    Here we used $F_{\infty,T}(\eta_{\infty,T})=0$ by Corollary~\ref{cor:optimizer-characterization}. Combining this with Lemma~\ref{lem:joint-clt} gives
    \begin{equation}\label{eq:fixed-score-limit}
        \sqrt{n}F_{n,T}(\eta_{\infty,T})
        \stackrel{d}{\longrightarrow}\Omega_{F,T}
        .
    \end{equation}
    By the mean value theorem, there exists a random $\tilde{\eta}_{n,T}$ between $\eta_{n,T}$ and $\eta_{\infty,T}$ such that
    \[
        F_{n,T}(\eta_{n,T})-F_{n,T}(\eta_{\infty,T})
        =F_{n,T}'(\tilde{\eta}_{n,T})(\eta_{n,T}-\eta_{\infty,T})
        .
    \]
    On the event $\{\eta_{n,T}\in(\eta_{\ell},\eta_{u})\}$, by Corollary~\ref{cor:optimizer-characterization}, we have $F_{n,T}(\eta_{n,T})=0$. Moreover, on the same event, Lemma~\ref{lem:score-uniform-fixed-T} implies $ F_{n,T}'(\tilde{\eta}_{n,T})=F_{\infty,T}'(\eta_{\infty,T})+O_{p}(n^{-1/2})$. Thus by Lemma~\ref{lem:fixed-optimizer-localization}, we have
    \[
        F_{n,T}'(\tilde{\eta}_{n,T})
        =F_{\infty,T}'(\eta_{\infty,T})+o_{p}(1)
    \]
    unconditionally. By Lemma~\ref{lem:fixed-positive-curvature}, this further implies that $F_{n,T}'(\tilde{\eta}_{n,T})<0$ holds with probability tending to one. Thus, on the intersection of this event and $\{\eta_{n,T}\in(\eta_{\ell},\eta_{u})\}$, both of which have probability tending to one, we have
    \[
        \sqrt{n}(\eta_{n,T}-\eta_{\infty,T})
        =-\frac{\sqrt{n}F_{n,T}(\eta_{\infty,T})}{F_{n,T}'(\tilde{\eta}_{n,T})}
        .
    \]
    Finally, by \eqref{eq:fixed-score-limit}, we obtain
    \[
        \sqrt{n}(\eta_{n,T}-\eta_{\infty,T})
        \stackrel{d}{\longrightarrow}
        -\frac{\Omega_{F,T}}{F_{\infty,T}'(\eta_{\infty,T})}
        =\Omega_{\eta,T}
        .
    \]
\end{proof}

\begin{theorem}[Limit distribution of the optimal loss at fixed $T$]\label{thm:a-limit-dist-fixed-T}
    Under the assumptions of Theorem~\ref{thm:fixed-horizon-transfer}, we have
    \[
        \sqrt{n}\left(\phi_{n,T}(\eta_{n,T})-\phi_{\infty,T}(\eta_{\infty,T})\right)
        \stackrel{d}{\longrightarrow}\Omega_{L,T}
        :=g_{L,T}^{\top}Z+\langle Q_{L,T},\Xi\rangle_{\mathrm{F}}
        .
    \]
\end{theorem}
\begin{proof}
    By Lemma~\ref{lem:frechet}, we have
    \[
        \phi_{n,T}(\eta_{\infty,T})-\phi_{\infty,T}(\eta_{\infty,T})
        =g_{L,T}^{\top}f_{n}^{(0)}(X)+\langle Q_{L,T},A_{n}-K\rangle_{\mathrm{F}}+o_{p}(n^{-1/2})
        .
    \]
    Hence Lemma~\ref{lem:joint-clt} yields
    \begin{equation}\label{eq:fixed-loss-at-limit-eta}
        \sqrt{n}(\phi_{n,T}(\eta_{\infty,T})-\phi_{\infty,T}(\eta_{\infty,T}))
        \stackrel{d}{\longrightarrow}\Omega_{L,T}
        .
    \end{equation}
    It remains to replace $\phi_{n,T}(\eta_{\infty,T})$ by $\phi_{n,T}(\eta_{n,T})$. By Taylor's theorem, there exists a random $\bar{\eta}_{n,T}$ between $\eta_{n,T}$ and $\eta_{\infty,T}$ such that
    \[
        \phi_{n,T}(\eta_{\infty,T})-\phi_{n,T}(\eta_{n,T})
        =\phi_{n,T}'(\eta_{n,T})(\eta_{\infty,T}-\eta_{n,T})+\frac{1}{2}\phi_{n,T}''(\bar{\eta}_{n,T})(\eta_{\infty,T}-\eta_{n,T})^{2}
    \]
    On the event $\{\eta_{n,T}\in(\eta_{\ell},\eta_{u})\}$, Proposition~\ref{prop:loss-derivatives} gives $\phi_{n,T}'(\eta_{n,T})=0$. Thus, on the same event, we have
    \[
        \phi_{n,T}(\eta_{\infty,T})-\phi_{n,T}(\eta_{n,T})
        =\frac{1}{2}\phi_{n,T}''(\bar{\eta}_{n,T})(\eta_{\infty,T}-\eta_{n,T})^{2}
        .
    \]
    Since $T$ is fixed, \eqref{eq:second-derivative}, Lemma~\ref{lem:joint-clt}, and Lemma~\ref{lem:fixed-optimizer-localization} imply $\phi_{n,T}''(\bar{\eta}_{n,T})=O_{p}(1)$. Therefore, combining this with Lemma~\ref{lem:fixed-optimizer-localization}, we have 
    \[
        \phi_{n,T}(\eta_{\infty,T})-\phi_{n,T}(\eta_{n,T})
        =O_{p}(n^{-1})
        =o_{p}(n^{-1/2})
    \]
    with probability tending to one. Together with \eqref{eq:fixed-loss-at-limit-eta}, this yields
    \[
        \sqrt{n}(\phi_{n,T}(\eta_{n,T})-\phi_{\infty,T}(\eta_{\infty,T}))
        \stackrel{d}{\longrightarrow}\Omega_{L,T}
        .
    \]
\end{proof}

\begin{lemma}\label{lem:transfer-quadratic-fixed-T}
    Under the assumptions of Theorem~\ref{thm:fixed-horizon-transfer}, we have
    \[
        c_{n,T}
        =\frac{\phi_{\infty,T}''(\eta_{\infty,T})}{2}(\eta_{n,T}-\eta_{\infty,T})^{2}(1+o_{p}(1))
        .
    \]
\end{lemma}
\begin{proof}
    We next expand the transfer gap. Since $\eta_{\infty,T}$ is the unique global minimizer, Proposition~\ref{prop:loss-derivatives} gives $\phi_{\infty,T}'(\eta_{\infty,T})=0$. By Taylor's theorem, there exists a random $\eta_{n,T}^{\dagger}$ between $\eta_{n,T}$ and $\eta_{\infty,T}$ such that
    \[
        c_{n,T}
        =\frac{\phi_{\infty,T}''(\eta_{\infty,T})}{2}
        (\eta_{n,T}-\eta_{\infty,T})^{2}+
        \frac{\phi_{\infty,T}'''(\eta_{n,T}^{\dagger})}{6}
        (\eta_{n,T}-\eta_{\infty,T})^{3}
        .
    \]
    For fixed $T$, by Lemma~\ref{lem:fixed-optimizer-localization}, with probability tending to one, we have
    \[
        |\phi_{\infty,T}'''(\eta_{n,T}^{\dagger})|
        \le\sup_{\eta_{\ell}\le\eta\le\eta_{u}}|\phi_{\infty,T}'''(\eta)|
        .
    \]
    Therefore, by Lemmas~\ref{lem:fixed-positive-curvature} and~\ref{lem:fixed-optimizer-localization}, we have
    \[
        c_{n,T}
        =\frac{\phi_{\infty,T}''(\eta_{\infty,T})}{2}(\eta_{n,T}-\eta_{\infty,T})^{2}(1+o_{p}(1))
        .
    \]
\end{proof}

\begin{proof}[Proof of Theorem~\ref{thm:fixed-horizon-transfer}]
    We verify that the limiting Gaussian variables $\Omega_{L,T}$ and $\Omega_{\eta,T}$ of $a_{n,T}$ and $b_{n,T}$ are nondegenerate. We first consider $\Omega_{\eta,T}$. For $s>0$, direct substitution into the definition of $F_{T}$ gives
    \[
        F_{T}(\eta_{\infty,T};-y,sK)
        =sF_{\infty,T}(s\eta_{\infty,T})
        .
    \]
    Differentiating both sides with respect to $s$ at $s=1$, applying Lemma~\ref{lem:frechet}, and using $F_{\infty,T}(\eta_{\infty,T})=0$, we have
    \[
        \langle Q_{F,T},K\rangle_{\mathrm{F}}
        =F_{\infty,T}(\eta_{\infty,T})+\eta_{\infty,T}F_{\infty,T}'(\eta_{\infty,T})
        =\eta_{\infty,T}F_{\infty,T}'(\eta_{\infty,T})
        .
    \]
    Recall from Lemma~\ref{lem:joint-clt} that $\Xi=\alpha K+XGX^{\top}$ with $\alpha\sim\mathcal{N}(0,2)$, where $\alpha$ is independent of $G$ and $Z$. Noting that Corollary~\ref{cor:optimizer-characterization} gives $\eta_{\infty,T}>0$, the contribution of the $\alpha K$ component of $\Xi$ to $\Omega_{\eta,T}$ is
    \[
        -\frac{\alpha\langle Q_{F,T},K\rangle_{\mathrm{F}}}{F_{\infty,T}'(\eta_{\infty,T})}
        =-\eta_{\infty,T}\alpha
        \sim\mathcal{N}(0,2\eta_{\infty,T}^{2})
        .
    \]
    Since this component is independent of the remaining Gaussian terms, $\Omega_{\eta,T}$ is nondegenerate. By Theorem~\ref{thm:b-limit-dist-fixed-T} and the continuous mapping theorem, this proves
    \[
        b_{n,T}
        =\Theta_{p}(n^{-1/2})
        .
    \]

    Next, we consider $\Omega_{L,T}$. By the definition of $g_{L,T}$, we have
    \[
        g_{L,T}^{\top}Kg_{L,T}
        =\frac{1}{m^{2}}
        \sum_{j=1}^{r}c_{j}\lambda_{j}r_{j,T}^{4T}
        ,
    \]
    where $r_{j,T}$ are defined in \eqref{eq:r_j}. Under $|\mathcal{J}_{y}|\ge2$, at least one of the $r_{j,T}$ is nonzero. Hence we have $g_{L,T}^{\top}Kg_{L,T}>0$. Since $Z\sim\mathcal{N}(0,K)$ is independent of $\Xi$ by Lemma~\ref{lem:joint-clt}, we have
    \[
        \operatorname{Var}(\Omega_{L,T})
        =\operatorname{Var}(g_{L,T}^{\top}Z)+\operatorname{Var}(\langle Q_{L,T},\Xi\rangle_{\mathrm{F}})
        =g_{L,T}^{\top}Kg_{L,T}+\operatorname{Var}(\langle Q_{L,T},\Xi\rangle_{\mathrm{F}})
        >0
        .
    \]
    Therefore, $\Omega_{L,T}$ is also nondegenerate. By Theorem~\ref{thm:a-limit-dist-fixed-T} and the continuous mapping theorem, this proves
    \[
        a_{n,T}
        =\Theta_{p}(n^{-1/2})
        .
    \]
    Finally, by Lemmas~\ref{lem:transfer-quadratic-fixed-T} and~\ref{lem:fixed-positive-curvature}, we obtain
    \[
        c_{n,T}
        =\Theta_{p}(b_{n,T}^{2})
        =\Theta_{p}(n^{-1})
        .
    \]
\end{proof}

\begin{proof}[Proof of Corollary~\ref{cor:fixed-horizon-fast-transfer}]
    By Theorem~\ref{thm:a-limit-dist-fixed-T} and the continuous mapping theorem, we have
    \[
        \sqrt{n}a_{n,T}
        \stackrel{d}{\longrightarrow}|\Omega_{L,T}|
        ,
    \]
    where $\Omega_{L,T}$ is a nondegenerate Gaussian random variable. Since $\mathrm{P}(|\Omega_{L,T}|=0)=0$, another application of the continuous mapping theorem gives
    \[
        \frac{1}{\sqrt{n}a_{n,T}}
        \stackrel{d}{\longrightarrow}\frac{1}{|\Omega_{L,T}|}
        .
    \]
    Therefore, we have
    \[
        \frac{1}{a_{n,T}}
        =O_{p}(\sqrt{n})
        .
    \]
    Combining this with $c_{n,T}=\Theta_{p}(n^{-1})$ in Theorem~\ref{thm:fixed-horizon-transfer} yields the claim.
\end{proof}

\section{Infinite-width large-horizon analysis}

\subsection{Proof of Theorem~\ref{thm:finite-minimax}}

Note that we have
\begin{align*}
    \phi_{\infty,T}(\eta)
    &=\frac{1}{2m}\left\|\left(I_{m}-\frac{\eta}{m}K\right)^{T}y\right\|^{2}
    =\frac{1}{2m}\left\|\left[P_{0}+\sum_{j=1}^{r}\left(1-\frac{\eta\lambda_{j}}{m}\right)P_{j}\right]^{T}y\right\|^{2}\\
    &=\frac{1}{2m}\left\|P_{0}y+\sum_{j=1}^{r}\left(1-\frac{\eta\lambda_{j}}{m}\right)^{T}P_{j}y\right\|^{2}
    =\frac{c_{0}}{2m}+\frac{1}{2m}\sum_{j=1}^{r}c_{j}\left(1-\frac{\eta\lambda_{j}}{m}\right)^{2T}
    .
\end{align*}
Here, $c_{0}/(2m)$ is an irreducible training loss. No choice of learning rate or number of steps changes it. Define
\[
    R_{T}(\eta)
    :=\left(2m\phi_{\infty,T}(\eta)-c_{0}\right)^{1/(2T)}
    =\left(\sum_{j=1}^{r}c_{j}\left|1-\frac{\eta\lambda_{j}}{m}\right|^{2T}\right)^{1/(2T)}
    .
\]
Then we have
\[
    \arg\min_{\eta\ge0}\phi_{\infty,T}(\eta)
    =\arg\min_{\eta\ge0}R_{T}(\eta)
    .
\]
Define
\[
    R_{\infty}(\eta)
    :=\max_{j\in\mathcal{J}_{y}}\left|1-\frac{\eta\lambda_{j}}{m}\right|
    =\left|1-\frac{\eta\lambda_{-}^{y}}{m}\right|\vee\left|1-\frac{\eta\lambda_{+}^{y}}{m}\right|
    .
\]
Choose an index $j_{*}\in\mathcal{J}_{y}$ such that $\left|1-\eta\lambda_{j_{*}}/m\right|=R_{\infty}(\eta)$. Then we have
\[
    c_{j_{*}}R_{\infty}(\eta)^{2T}
    \le
    \sum_{j\in\mathcal{J}_{y}}
    c_{j}\left|1-\frac{\eta\lambda_{j}}{m}\right|^{2T}
    \le
    \left(\sum_{j=1}^{r}c_{j}\right)
    R_{\infty}(\eta)^{2T}
    ,
\]
which implies
\[
    \min_{j\in\mathcal{J}_{y}}c_{j}^{1/(2T)}R_{\infty}(\eta)
    \le
    R_{T}(\eta)
    \le
    \left(\sum_{j=1}^{r}c_{j}\right)^{1/(2T)}
    R_{\infty}(\eta).
\]
Since $\mathcal{J}_{y}$ is finite and all active coefficients are strictly positive, we have
\[
    \min_{j\in\mathcal{J}_{y}}c_{j}^{1/(2T)}
    \to1
    ,\qquad
    \left(\sum_{j=1}^{r}c_{j}\right)^{1/(2T)}
    \to1
    .
\]
Thus, on every compact interval $\mathcal{I}$, we have
\[
    \sup_{\eta\in\mathcal{I}}|R_{T}(\eta)-R_{\infty}(\eta)|
    \le\left(\left|\min_{j\in\mathcal{J}_{y}}c_{j}^{1/(2T)}-1\right|\vee\left|\left(\sum_{j=1}^{r}c_{j}\right)^{1/(2T)}-1\right|\right)\sup_{\eta\in\mathcal{I}}|R_{\infty}(\eta)|
    \to0
    .
\]

It remains to verify that the global minimizers of $R_{T}$ are contained in a fixed compact interval. Fix any $j_{0}\in\mathcal{J}_{y}$. Since $\eta_{\infty,T}$ is a global minimizer over $[0,\infty)$, we have
\[
    R_{T}(\eta_{\infty,T})
    \le R_{T}(0)
    =\left(\sum_{j\in\mathcal{J}_{y}}c_{j}\right)^{1/(2T)}
    .
\]
On the other hand, for any $\eta\ge0$, we have
\[
    R_{T}(\eta)
    \ge c_{j_{0}}^{1/(2T)}
    \left|1-\frac{\eta\lambda_{j_{0}}}{m}\right|
    .
\]
Therefore we have
\[
    \left|1-\frac{\eta_{\infty,T}\lambda_{j_{0}}}{m}\right|
    \le c_{j_{0}}^{-1/(2T)}R_{T}(\eta_{\infty,T})
    \le\left(\frac{\sum_{j\in\mathcal{J}_{y}}c_{j}}{c_{j_{0}}}\right)^{1/(2T)}
    \le\left(\frac{\sum_{j\in\mathcal{J}_{y}}c_{j}}{c_{j_{0}}}\right)^{1/2}
    .
\]
Hence there exists a deterministic constant $M<\infty$, independent of $T$, such that $\eta_{\infty,T}\in[0,M]$ for every $T\ge1$. Enlarging $M$ if necessary, we may also assume $\eta_{*}^{y}\in[0,M]$. Since $\sup_{\eta\in[0,M]}|R_{T}(\eta)-R_{\infty}(\eta)|\to0$ holds and $R_{\infty}$ has the unique minimizer $\eta_{*}^{y}$, we have
\[
    \eta_{\infty,T}
    \to\eta_{*}^{y}
    .
\]

If $\lambda_{+}^{y}=\lambda_{-}^{y}=\lambda$, the loss is a positive multiple of $(1-\eta\lambda/m)^{2T}$ plus a constant, whose unique minimizer is $m/\lambda$.
\hfill\qedsymbol

\subsection{Expansion of the infinite-width optimizer and asymptotic loss}

\begin{proposition}[First-order large-$T$ expansion of the optimal learning rate]\label{prop:finite-T-correction}
    Suppose Assumption~\ref{asmp:endpoint-support} holds and define
    \[
        \bar{\vartheta}
        :=\frac{1}{q_{*}}\max_{2\le j\le r-1}\left|1-\frac{\eta_{*}\lambda_{j}}{m}\right|
        ,
    \]
    where $\max\emptyset:=0$. Then $0\le\bar{\vartheta}<1$. Fix any constant $\vartheta$ satisfying $\bar{\vartheta}<\vartheta<1$. Then, as $T\to\infty$, we have
    \[
        \eta_{\infty,T}
        =\eta_{*}-\frac{mq_{*}}{(\lambda_{1}+\lambda_{r})(2T-1)}L_{*}+O(T^{-2})+O\left(\frac{\vartheta^{2T}}{T}\right)
        .
    \]
\end{proposition}
\begin{proof}
    Near $\eta_{*}$, define the positive endpoint magnitudes
    \[
        r_{1}(\eta)
        =\frac{\eta\lambda_{1}}{m}-1
        ,\qquad
        r_{r}(\eta)
        =1-\frac{\eta\lambda_{r}}{m}
        .
    \]
    Then we have $r_{1}(\eta_{*})=r_{r}(\eta_{*})=q_{*}>0$. By the definition of $\bar{\vartheta}$, we have
    \[
        \max_{2\le j\le r-1}\frac{|1-\eta_{*}\lambda_{j}/m|}{r_{1}(\eta_{*})}
        =\bar{\vartheta}
        <\vartheta
        .
    \]
    Since Theorem~\ref{thm:finite-minimax} gives $\eta_{\infty,T}\to\eta_{*}$, for all sufficiently large $T$, we have
    \begin{equation}\label{eq:active-mode-contribution}
        \max_{2\le j\le r-1}\frac{|1-\eta_{\infty,T}\lambda_{j}/m|}{r_{1}(\eta_{\infty,T})}
        \le\vartheta
        .
    \end{equation}
    By Corollary~\ref{cor:optimizer-characterization}, the score equation is given by
    \begin{align*}
        0
        &=F_{\infty,T}(\eta_{\infty,T})
        =y^{\top}K B_{\infty}(\eta_{\infty,T})^{2T-1}y
        =y^{\top}K\left(I_{m}-\frac{\eta_{\infty,T}}{m}K\right)^{2T-1}y\\
        &=\sum_{j=1}^{r}c_{j}\lambda_{j}\left(1-\frac{\lambda_{j}\eta_{\infty,T}}{m}\right)^{2T-1}\\
        &=-c_{1}\lambda_{1}(r_{1}(\eta_{\infty,T}))^{2T-1}+c_{r}\lambda_{r}r_{r}(\eta_{\infty,T})^{2T-1}+\sum_{j=2}^{r-1}c_{j}\lambda_{j}\left(1-\frac{\lambda_{j}\eta_{\infty,T}}{m}\right)^{2T-1}
        .
    \end{align*}
    Hence we have
    \begin{align*}
        &|c_{r}\lambda_{r}r_{r}(\eta_{\infty,T})^{2T-1}-c_{1}\lambda_{1}r_{1}(\eta_{\infty,T})^{2T-1}|
        \le\sum_{j=2}^{r-1}c_{j}\lambda_{j}\left|1-\frac{\lambda_{j}\eta_{\infty,T}}{m}\right|^{2T-1}
        .
    \end{align*}
    Using \eqref{eq:active-mode-contribution}, we obtain
    \[
        \left|c_{r}\lambda_{r}\left(\frac{r_{r}(\eta_{\infty,T})}{r_{1}(\eta_{\infty,T})}\right)^{2T-1}-c_{1}\lambda_{1}\right|
        \le\left(\sum_{j=2}^{r-1}c_{j}\lambda_{j}\right)\vartheta^{2T-1}
        =O(\vartheta^{2T-1}).
    \]
    Therefore, we have
    \[
        \left(\frac{r_{r}(\eta_{\infty,T})}{r_{1}(\eta_{\infty,T})}\right)^{2T-1}
        =\frac{c_{1}\lambda_{1}}{c_{r}\lambda_{r}}\left(1+O(\vartheta^{2T-1})\right)
        .
    \]    
    Taking logarithms gives
    \[
        (2T-1)\left(\log r_{r}(\eta_{\infty,T})-\log r_{1}(\eta_{\infty,T})\right)
        =L_{*}+O(\vartheta^{2T})
        .
    \]
    Writing $\eta_{\infty,T}=\eta_{*}+\delta$, we have
    \begin{equation}\label{eq:r_eta_infty}
        r_{1}(\eta_{\infty,T})
        =q_{*}+\frac{\lambda_{1}\delta}{m}
        ,\qquad
        r_{r}(\eta_{\infty,T})
        =q_{*}-\frac{\lambda_{r}\delta}{m}
        .
    \end{equation}
    Thus, using $\log(1+x)=x+O(x^{2})$, we obtain
    \[
        \log r_{r}(\eta_{\infty,T})-\log r_{1}(\eta_{\infty,T})
        =\log\left(\frac{1-\frac{\lambda_{r}\delta}{mq_{*}}}{1+\frac{\lambda_{1}\delta}{mq_{*}}}\right)
        =-\frac{\lambda_{r}\delta}{mq_{*}}-\frac{\lambda_{1}\delta}{mq_{*}}+O(\delta^{2})
        =-\frac{\lambda_{1}+\lambda_{r}}{mq_{*}}\delta+O(\delta^{2})
        .
    \]
    Combining the results, we have
    \[
        L_{*}+O(\vartheta^{2T})
        =(2T-1)\left(-\frac{\lambda_{1}+\lambda_{r}}{mq_{*}}\delta+O(\delta^{2})\right)
        .
    \]
    Since the left-hand side is $O(1)$, we have $\delta=O(T^{-1})$. Thus
    \[
        L_{*}+O(\vartheta^{2T})
        =-(2T-1)\frac{\lambda_{1}+\lambda_{r}}{mq_{*}}\delta+O(T^{-1})
        .
    \]
    Solving for $\delta$, we have
    \[
        \eta_{\infty,T}-\eta_{*}
        =\delta
        =-\frac{L_{*}mq_{*}}{(\lambda_{1}+\lambda_{r})(2T-1)}+O(T^{-2})+O\left(\frac{\vartheta^{2T}}{T}\right)
        .
    \]
\end{proof}

\begin{corollary}[Loss and curvature at the large-horizon optimum]\label{cor:T-loss-rate}
    Under the assumptions of Proposition~\ref{prop:finite-T-correction}, we have
    \begin{align}
        \notag
        &\phi_{\infty,T}(\eta_{\infty,T})-\frac{c_{0}}{2m}
        =\Theta\left(q_{*}^{2T}\right)
        ,\\
        \label{eq:finite-score-curvature}
        &\left|F_{\infty,T}'(\eta_{\infty,T})\right|
        =\Theta\left(Tq_{*}^{2T-2}\right)
        ,\\
        \label{eq:finite-loss-curvature}
        &\phi_{\infty,T}''(\eta_{\infty,T})
        =\Theta\left(T^{2}q_{*}^{2T-2}\right)
        .
    \end{align}
\end{corollary}
\begin{proof}
    Define $r_{1}(\eta)$ and $r_{r}(\eta)$ as in the proof of Proposition~\ref{prop:finite-T-correction}. By \eqref{eq:r_eta_infty} and Proposition~\ref{prop:finite-T-correction}, we have
    \begin{align*}
        r_{r}(\eta_{\infty,T})
        &=q_{*}\left[1+\frac{\lambda_{r}}{\lambda_{1}+\lambda_{r}}\frac{L_{*}}{2T-1}+O(T^{-2})+O\left(\frac{\vartheta^{2T}}{T}\right)\right]
        ,\\
        r_{1}(\eta_{\infty,T})
        &=q_{*}\left[1-\frac{\lambda_{1}}{\lambda_{1}+\lambda_{r}}\frac{L_{*}}{2T-1}+O(T^{-2})+O\left(\frac{\vartheta^{2T}}{T}\right)\right]
        .
    \end{align*}
    In particular, both endpoint magnitudes are $q_{*}(1+O(T^{-1}))$.  More precisely, for every fixed integer $\ell$, the expansion
    \[
        \left(1+\frac{a}{2T-1}+O(T^{-2})\right)^{2T-\ell}
        =\exp(a)\left(1+O(T^{-1})\right)
    \]
    gives
    \begin{align}
        r_{r}(\eta_{\infty,T})^{2T-\ell}
        &=q_{*}^{2T-\ell}a_{r}^{*}\left(1+O(T^{-1})+O(\vartheta^{2T})\right)
        ,\label{eq:r-minus-power-asymptotic}\\
        r_{1}(\eta_{\infty,T})^{2T-\ell}
        &=q_{*}^{2T-\ell}a_{1}^{*}\left(1+O(T^{-1})+O(\vartheta^{2T})\right)
        .\label{eq:r-plus-power-asymptotic}
    \end{align}
    We now control the remaining active modes. Since $\eta_{\infty,T}\to\eta_{*}$ by Theorem~\ref{thm:finite-minimax}, we have
    \[
        \frac{1}{q_{*}}\max_{2\le j\le r-1}\left|1-\frac{\eta_{\infty,T}\lambda_{j}}{m}\right|
        \to\bar{\vartheta}
        <1
        ,
    \]
    which implies, for sufficiently large $T$,
    \[
        \max_{2\le j\le r-1}\left|1-\frac{\eta_{\infty,T}\lambda_{j}}{m}\right|
        \le\vartheta q_{*}
        .
    \]
    Hence we have
    \begin{equation}\label{eq:r-middle-power-asymptotic}
        \sum_{j=2}^{r-1}c_{j}\lambda_{j}^{2}\left|1-\frac{\eta_{\infty,T}\lambda_{j}}{m}\right|^{2T-2}
        =O\left((\vartheta q_{*})^{2T-2}\right)
        =o\left(q_{*}^{2T-2}\right)
        .
    \end{equation}
    Using the spectral representation of the loss, we therefore have
    \[
        \phi_{\infty,T}(\eta_{\infty,T})-\frac{c_{0}}{2m}
        =\frac{1}{2m}\sum_{j=1}^{r}c_{j}\left(1-\frac{\eta_{\infty,T}\lambda_{j}}{m}\right)^{2T}
        =\frac{q_{*}^{2T}}{2m}\left[c_{r}a_{r}^{*}+c_{1}a_{1}^{*}+o(1)\right]
        =\Theta(q_{*}^{2T})
        ,
    \]
    where the last equality follows since $c_{r}a_{r}^{*}+c_{1}a_{1}^{*}+o(1)$ converges to a strictly positive constant.

    Next, recall that by Proposition~\ref{prop:loss-derivatives},
    \[
        F_{\infty,T}'(\eta)
        =-\frac{2T-1}{m}\sum_{j=1}^{r}c_{j}\lambda_{j}^{2}\left(1-\frac{\eta\lambda_{j}}{m}\right)^{2T-2}
        .
    \]
    Note that all summands in the preceding display are nonnegative before applying the leading minus sign. Applying \eqref{eq:r-minus-power-asymptotic}, \eqref{eq:r-plus-power-asymptotic}, and \eqref{eq:r-middle-power-asymptotic} gives
    \[
        F_{\infty,T}'(\eta_{\infty,T})
        =-\frac{2T-1}{m}q_{*}^{2T-2}\left[c_{r}\lambda_{r}^{2}a_{r}^{*}+c_{1}\lambda_{1}^{2}a_{1}^{*}+o(1)\right]
        =-\Theta\left(Tq_{*}^{2T-2}\right)
        .
    \]
    The last equality follows since the limiting coefficient in square brackets is again strictly positive.
    
    Finally, noting that
    \[
        \phi_{\infty,T}'(\eta)
        =-\frac{T}{m^{2}}F_{\infty,T}(\eta)
    \]
    holds by Proposition~\ref{prop:loss-derivatives}, the curvature at the optimizer satisfies
    \[
        \phi_{\infty,T}''(\eta_{\infty,T})
        =-\frac{T}{m^{2}}F_{\infty,T}'(\eta_{\infty,T})
        .
    \]
    Substituting the previous result yields
    \[
        \phi_{\infty,T}''(\eta_{\infty,T})
        =\frac{T}{m^{2}}\Theta\left(Tq_{*}^{2T-2}\right)
        =\Theta(T^{2}q_{*}^{2T-2})
        .
    \]
\end{proof}

\subsection{Endpoint contractions and spectral separation}

\begin{lemma}[Endpoint contraction asymptotics]\label{lem:endpoint-contractions}
    For $r_{j,T},\ (j\in[r])$ defined in \eqref{eq:r_j}, we have
    \[
        \frac{r_{1,T}^{2T-1}}{q_{*}^{2T-1}}
        \to-a_{1}^{*}
        ,\qquad
        \frac{r_{r,T}^{2T-1}}{q_{*}^{2T-1}}
        \to a_{r}^{*}
        ,\qquad
        \frac{r_{1,T}^{2T-2}}{q_{*}^{2T-2}}
        \to a_{1}^{*}
        ,\qquad
        \frac{r_{r,T}^{2T-2}}{q_{*}^{2T-2}}
        \to a_{r}^{*}
        .
    \]
    Moreover, if $r>2$, there exists a constant $\vartheta_{0}\in(0,1)$ such that
    \[
        |r_{j,T}|
        \le\vartheta_{0}q_{*}
        \qquad
        (2\le j\le r-1)
    \]
    for all sufficiently large $T$.
\end{lemma}
\begin{proof}
    Proposition~\ref{prop:finite-T-correction} gives
    \[
        \eta_{\infty,T}-\eta_{*}
        =-\frac{mq_{*}L_{*}}{(\lambda_{1}+\lambda_{r})(2T-1)}+o(T^{-1})
        .
    \]
    Since $1-\eta_{*}\lambda_{1}/m=-q_{*}$ and $1-\eta_{*}\lambda_{r}/m=q_{*}$, we have
    \begin{align*}
        &r_{1,T}
        =-q_{*}\left(1-\frac{\lambda_{1}L_{*}}{(\lambda_{1}+\lambda_{r})(2T-1)}+o((2T-1)^{-1})\right)
        ,\\
        &r_{r,T}
        =q_{*}\left(1+\frac{\lambda_{r}L_{*}}{(\lambda_{1}+\lambda_{r})(2T-1)}+o((2T-1)^{-1})\right)
        .
    \end{align*}
    Thus, noting that $2T-1$ is odd, we obtain
    \[
        \frac{r_{1,T}^{2T-1}}{q_{*}^{2T-1}}
        \to-a_{1}^{*}
        ,\qquad
        \frac{r_{r,T}^{2T-1}}{q_{*}^{2T-1}}
        \to a_{r}^{*}
        ,\qquad
        \frac{r_{1,T}^{2T-2}}{q_{*}^{2T-2}}
        \to a_{1}^{*}
        ,\qquad
        \frac{r_{r,T}^{2T-2}}{q_{*}^{2T-2}}
        \to a_{r}^{*}
        .
    \]
    For an interior positive eigenvalue $\lambda_{j}\in(\lambda_{r},\lambda_{1})$, the definition of $\eta_{*}$ gives
    \[
        -q_{*}
        <1-\frac{\eta_{*}\lambda_{j}}{m}
        <q_{*}
        .
    \]
    By Proposition~\ref{prop:finite-T-correction}, we have
    \[
        \max_{2\le j\le r-1}\frac{|1-\eta_{*}\lambda_{j}/m|}{q_{*}}
        =\bar{\vartheta}
        <1
        .
    \]
    Since Theorem~\ref{thm:finite-minimax} gives $\eta_{\infty,T}\to\eta_{*}$, there exists $\vartheta_{0}\in(\bar{\vartheta},1)$ such that, for all sufficiently large $T$,
    \[
        \max_{2\le j\le r-1}\frac{|r_{j,T}|}{q_{*}}
        =\max_{2\le j\le r-1}\frac{|1-\eta_{\infty,T}\lambda_{j}/m|}{q_{*}}
        \le\vartheta_{0}
        .
    \]
\end{proof}

\section{Matrix-power perturbation bounds}

\subsection{Fr\'echet expansion of matrix powers}

\begin{lemma}[First-order expansion of matrix powers]\label{lem:matrix-power-derivative}
    Let $k\ge1$. Equip $\mathbb{R}^{p\times p}$ with any submultiplicative matrix norm $\|\cdot\|$, and define $\Psi_{k}(B)=B^{k}$ as a map from this normed space into itself. Then $\Psi_{k}$ is Fr\'echet differentiable at every $B$, and for every $C\in\mathbb{R}^{p\times p}$, we have
    \[
        D\Psi_{k}(B)[C]
        =\sum_{s=0}^{k-1}B^{s}CB^{k-1-s}
        .
    \]
    Consequently, for every fixed $C$, we have
    \[
        (B+\epsilon C)^{k}
        =B^{k}+\epsilon\sum_{s=0}^{k-1}B^{s}CB^{k-1-s}+O(\epsilon^{2})
    \]
    as $\epsilon\to0$. No commutativity between $B$ and $C$ is required.
\end{lemma}
\begin{proof}
    We first recall that for arbitrary square matrices $X$ and $Y$ of the same size and every $k\ge1$, we have
    \[
        X^{k}-Y^{k}
        =\sum_{s=0}^{k-1}X^{s}(X-Y)Y^{k-1-s}
        .
    \]
    Applying this identity with $X=B+E$ and $Y=B$, we obtain
    \[
        \Psi_{k}(B+E)-\Psi_{k}(B)
        =(B+E)^{k}-B^{k}
        =\sum_{s=0}^{k-1}(B+E)^{s}EB^{k-1-s}
        .
    \]
    Thus if we define
    \[
        R_{k}(B,E)
        :=\sum_{s=0}^{k-1}((B+E)^{s}-B^{s})EB^{k-1-s}
        ,
    \]
    we can write
    \[
        \Psi_{k}(B+E)-\Psi_{k}(B)
        =\sum_{s=0}^{k-1}B^{s}EB^{k-1-s}+R_{k}(B,E)
        .
    \]
    In $R_{k}(B,E)$, the summand corresponding to $s=0$ vanishes. For $s\ge1$, applying the same identity once more gives
    \[
        (B+E)^{s}-B^{s}
        =\sum_{j=0}^{s-1}(B+E)^{j}EB^{s-1-j}
        .
    \]
    Therefore, we have
    \[
        R_{k}(B,E)
        =\sum_{s=1}^{k-1}\sum_{j=0}^{s-1}(B+E)^{j}EB^{s-1-j}EB^{k-1-s}
        .
    \]
    Let $\|\cdot\|$ be any submultiplicative matrix norm. Then for each $s\in\{1,\ldots,k-1\}$ and $j\in\{0,\ldots,s-1\}$, we have
    \begin{align*}
        \|(B+E)^{j}EB^{s-1-j}EB^{k-1-s}\|
        &\le\|B+E\|^{j}\|E\|\|B\|^{s-1-j}\|E\|\|B\|^{k-1-s}\\
        &\le(\|B\|+\|E\|)^{k-2}\|E\|^{2}
        .
    \end{align*}
    Since there are $k(k-1)/2$ terms in the double sum, we have
    \[
        \|R_{k}(B,E)\|
        \le\frac{k(k-1)}{2}(\|B\|+\|E\|)^{k-2}\|E\|^{2}
        .
    \]
    Thus, we have
    \begin{align*}
        &\frac{\|\Psi_{k}(B+E)-\Psi_{k}(B)-\sum_{s=0}^{k-1}B^{s}EB^{k-1-s}\|}{\|E\|}
        =\frac{\|R_{k}(B,E)\|}{\|E\|}\\
        &\le\frac{k(k-1)}{2}(\|B\|+\|E\|)^{k-2}\|E\|
        \to0
    \end{align*}
    as $\|E\|\to0$. Since the map $E\mapsto\sum_{s=0}^{k-1}B^{s}EB^{k-1-s}$ is linear and bounded, it is the Fr\'echet derivative of $\Psi_{k}$ at $B$. Therefore, we have
    \[
        D\Psi_{k}(B)[C]
        =\sum_{s=0}^{k-1}B^{s}CB^{k-1-s}
        .
    \]
    Finally, taking $E=\epsilon C$ for a fixed matrix $C$ gives
    \begin{align*}
        (B+\epsilon C)^{k}
        &=\Psi_{k}(B+\epsilon C)
        =\Psi_{k}(B)+D\Psi_{k}(B)[\epsilon C]+R_{k}(B,\epsilon C)\\
        &=B^{k}+\epsilon\sum_{s=0}^{k-1}B^{s}CB^{k-1-s}+O(\epsilon^{2})
        ,
    \end{align*}
    which proves the final assertion.
\end{proof}

\subsection{Uniform noncommutative power bounds}

\begin{lemma}[Uniform local power perturbation bounds]\label{lem:local-word-bound}
    Let $K\succeq0$ be a finite-dimensional symmetric positive semidefinite matrix and let $P$ project onto $\operatorname{range}(K)$. Suppose there exists a deterministic sequence $q_{T}>0$, bounded away from zero, such that
    \[
        \sup_{\eta\in \mathcal{J}_{T}}\|PB_{K}(\eta)P\|_{\mathrm{op}}
        \le q_{T}
        ,\qquad
        q_{0}
        :=\inf_{T}q_{T}
        >0
        .
    \]
    Let $A=K+E\succeq0$ be symmetric and have the same kernel as $K$. Then we have
    \[
        PA=AP
        ,\qquad
        PK=KP
        ,\qquad
        PE=EP
        .
    \]
    Suppose $T\|E\|_{\mathrm{op}}\le c$ for some fixed $c$. Suppose $\sup_{T}\sup_{\eta\in\mathcal{J}_{T}}|\eta|<\infty$. Then, uniformly over $\eta\in\mathcal{J}_{T}$ and $k\le2T$, we have
    \begin{align}
        \label{eq:power}
        &\|PB_{A}(\eta)^{k}P\|_{\mathrm{op}}
        \lesssim q_{T}^{k}
        ,\\
        \label{eq:local-first-power}
        &\|P(B_{A}(\eta)^{k}-B_{K}(\eta)^{k})P\|_{\mathrm{op}}
        \lesssim kq_{T}^{k-1}\|E\|_{\mathrm{op}}
        .
    \end{align}
\end{lemma}
\begin{proof}
    Define $\mathcal{H}=\operatorname{range}(K)=\ker(K)^{\perp}$. Since $A$ and $K$ are symmetric and have the same kernel, we have
    \[
        \operatorname{range}(A)
        =\ker(A)^{\perp}
        =\ker(K)^{\perp}
        =\operatorname{range}(K)
        =\mathcal{H}
        .
    \]
    Consequently, $\mathcal{H}$ and $\mathcal{H}^{\perp}$ are invariant under $A$, $K$, and $E=A-K$. Equivalently, we have\footnote{This equivalence can be proved as follows. For any $x\in\mathbb{R}^{m}$, decompose $x=Px+(I_{m}-P)x$. Here, $Px\in\mathcal{H}$ and $(I_{m}-P)x\in\mathcal{H}^{\perp}$. If $\mathcal{H}$ and $\mathcal{H}^{\perp}$ are invariant under $A$, we have $APx\in\mathcal{H}$ and $A(I_{m}-P)x\in\mathcal{H}^{\perp}$. Thus,
    \[
        PAx
        =P[APx+A(I_{m}-P)x]
        =P(APx)+P[A(I_{m}-P)x]
        =APx+0
        .
    \]
    Since $x$ is arbitrary, we conclude $PA=AP$. On the other hand, if $PA=AP$, then $P(Ax)=A(Px)=Ax$ for $x\in\mathcal{H}$ and $P(Ax)=A(Px)=0$ for $x\in\mathcal{H}^{\perp}$; thus, $A\mathcal{H}\subset\mathcal{H}$ and $A\mathcal{H}^{\perp}\subset\mathcal{H}^{\perp}$ hold. Similar arguments apply to $K$ and $E$.}
    \[
        PA=AP
        ,\qquad
        PK=KP
        ,\qquad
        PE=EP
        ,
    \]
    and hence $P$ also commutes with $B_{A}(\eta)$ and $B_{K}(\eta)$. In particular, we have
    \[
        PB_{A}(\eta)^{k}P
        =\left(PB_{A}(\eta)P\right)^{k}
        ,\qquad
        PB_{K}(\eta)^{k}P
        =\left(PB_{K}(\eta)P\right)^{k}
        .
    \]
    These identities hold as operators on $\mathcal{H}$.
    
    In the remainder of the proof, we regard these matrices as linear operators from $\mathcal{H}$ to itself.  Fix $\eta\in\mathcal{J}_{T}$ and define
    \[
        B_{0}
        =PB_{K}(\eta)P\big|_{\mathcal{H}}
        =B_{K}(\eta)\big|_{\mathcal{H}}
        ,\qquad
        \Delta
        =-\frac{\eta}{m}PEP\big|_{\mathcal{H}}
        =-\frac{\eta}{m}E\big|_{\mathcal{H}}
        .
    \]
    Then
    \[
        PB_{A}(\eta)P\big|_{\mathcal{H}}
        =B_{0}+\Delta
        .
    \]
    With respect to the orthogonal decomposition $\mathbb{R}^{m}=\mathcal{H}\oplus\mathcal{H}^{\perp}$, the operators $A$ and $K$ have the block representations
    \begin{equation}\label{eq:AKDecomp}
        A
        =
        \begin{bmatrix}
            A|_{\mathcal{H}} & 0\\
            0 & 0
        \end{bmatrix}
        ,\qquad
        K
        =
        \begin{bmatrix}
            K|_{\mathcal{H}} & 0\\
            0 & 0
        \end{bmatrix}
        .
    \end{equation}
    Consequently, we have
    \begin{equation}\label{eq:BDecomp}
        B_{A}(\eta)
        =
        \begin{bmatrix}
            B_{0}+\Delta & 0\\
            0 & I_{\mathcal{H}^{\perp}}
        \end{bmatrix}
        ,\qquad
        B_{K}(\eta)
        =
        \begin{bmatrix}
            B_{0} & 0\\
            0 & I_{\mathcal{H}^{\perp}}
        \end{bmatrix}
        .
    \end{equation}

    By assumption, we have $\|B_{0}\|_{\mathrm{op}}\le q_{T}$. Set $M_{\eta}=\sup_{T\ge1}\sup_{\eta\in\mathcal{J}_{T}}|\eta|/m<\infty$. Then it follows that $\|\Delta\|_{\mathrm{op}}\le M_{\eta}\|E\|_{\mathrm{op}}$. Since $T\|E\|_{\mathrm{op}}\le c$, for $k\le2T$, we have
    \[
        \frac{k\|\Delta\|_{\mathrm{op}}}{q_{T}}
        \le\frac{2M_{\eta}T\|E\|_{\mathrm{op}}}{q_{T}}
        \le\frac{2M_{\eta}c}{q_{0}}
        .
    \]
    Thus $k\|\Delta\|_{\mathrm{op}}/q_{T}$ is bounded uniformly in $\eta$, $k$, and $T$.

    We first prove \eqref{eq:power}. For $s\le k\le 2T$, we have
    \begin{align*}
        &\|(B_{0}+\Delta)^{s}\|_{\mathrm{op}}
        \le(\|B_{0}\|_{\mathrm{op}}+\|\Delta\|_{\mathrm{op}})^{s}
        \le(q_{T}+\|\Delta\|_{\mathrm{op}})^{s}
        =q_{T}^{s}\left(1+\frac{\|\Delta\|_{\mathrm{op}}}{q_{T}}\right)^{s}\\
        &\le q_{T}^{s}\exp\left(\frac{s\|\Delta\|_{\mathrm{op}}}{q_{T}}\right)
        \le q_{T}^{s}\exp\left(\frac{2M_{\eta}c}{q_{0}}\right)
        \lesssim q_{T}^{s}
        ,
    \end{align*}
    which proves \eqref{eq:power}. By \eqref{eq:BDecomp}, for every $k\ge1$,
    \[
        B_{A}(\eta)^{k}-B_{K}(\eta)^{k}
        =
        \begin{bmatrix}
            (B_{0}+\Delta)^{k}-B_{0}^{k} & 0\\
            0 & 0
        \end{bmatrix}
        .
    \]
    Therefore,
    \[
        \left\|B_{A}(\eta)^{k}-B_{K}(\eta)^{k}\right\|_{\mathrm{op},\mathbb{R}^{m}}
        =\left\|(B_{0}+\Delta)^{k}-B_{0}^{k}\right\|_{\mathrm{op},\mathcal{H}}
        .
    \]
    Thus it is sufficient to establish the desired power bounds on $\mathcal{H}$. The telescoping identity $\Gamma^{k}-\Lambda^{k}=\sum_{s=0}^{k-1}\Gamma^{s}(\Gamma-\Lambda)\Lambda^{k-1-s}$ implies
    \[
        (B_{0}+\Delta)^{k}-B_{0}^{k}
        =\sum_{s=0}^{k-1}(B_{0}+\Delta)^{s}\Delta B_{0}^{k-1-s}
        .
    \]
    Consequently, we have
    \begin{align*}
        \|(B_{0}+\Delta)^{k}-B_{0}^{k}\|_{\mathrm{op}}
        &\le\sum_{s=0}^{k-1}\|(B_{0}+\Delta)^{s}\|_{\mathrm{op}}\|\Delta\|_{\mathrm{op}}\|B_{0}^{k-1-s}\|_{\mathrm{op}}\\
        &\lesssim \|\Delta\|_{\mathrm{op}}\sum_{s=0}^{k-1}q_{T}^{s}q_{T}^{k-1-s}
        =kq_{T}^{k-1}\|\Delta\|_{\mathrm{op}}
        \lesssim kq_{T}^{k-1}\|E\|_{\mathrm{op}}
        ,
    \end{align*}
    which implies \eqref{eq:local-first-power}.
\end{proof}

\subsection{Application to the finite-width kernel}

\begin{lemma}[Uniform finite-width contraction near the optimal learning rate]\label{lem:local-random-power-bound}
    Suppose $m$ and $d$ are fixed and $T=T_{n}$ satisfies $T\to\infty$ and $T/\sqrt{n}\to0$. Suppose Assumption~\ref{asmp:endpoint-support} holds. Let $P=\sum_{j=1}^{r}P_{j}$ denote the orthogonal projector onto $\operatorname{range}(K)$. Then uniformly over $|\eta-\eta_{\infty,T}|\le C/T$ and $0\le k\le2T$, we have    
    \[
        \|PB_{\infty}(\eta)^{k}P\|_{\mathrm{op}}
        \lesssim q_{*}^{k}
        .
    \]
    For any $c>0$, define the event
    \[
        \mathcal{E}_{n}
        =\{\ker(A_{n})=\ker(K),\ T\|A_{n}-K\|_{\mathrm{op}}\le c\}
        .
    \]
    Then we have $\mathrm{P}(\mathcal{E}_{n})\to1$, and uniformly over $|\eta-\eta_{\infty,T}|\le C/T$ and $0\le k\le2T$, we have
    \[
        \|PB_{n}(\eta)^{k}P\|_{\mathrm{op}}
        \lesssim q_{*}^{k}
        ,\qquad
        \|P\left(B_{n}(\eta)^{k}-B_{\infty}(\eta)^{k}\right)P\|_{\mathrm{op}}
        \lesssim kq_{*}^{k-1}\|A_{n}-K\|_{\mathrm{op}}
    \]
    on $\mathcal{E}_{n}$. On this event, we also have
    \[
        PA_{n}=A_{n}P
        ,\qquad
        PK=KP
        ,\qquad
        P(A_{n}-K)=(A_{n}-K)P
        .
    \]
\end{lemma}
\begin{proof}
    We confirm that the assumptions of Lemma~\ref{lem:local-word-bound} are satisfied with probability tending to one. First, we have
    \[
        \|PB_{\infty}(\eta)P\|_{\mathrm{op}}
        =\left\|\sum_{j=1}^{r}\left(1-\frac{\eta\lambda_{j}}{m}\right)P_{j}\right\|_{\mathrm{op}}
        =\max_{j\in[r]}\left|1-\frac{\eta\lambda_{j}}{m}\right|
        =\left|1-\frac{\eta\lambda_{1}}{m}\right|\vee\left|1-\frac{\eta\lambda_{r}}{m}\right|
        .
    \]
    Since $\|PB_{\infty}(\eta_{*})P\|_{\mathrm{op}}=\left|1-\frac{\eta_{*}\lambda_{1}}{m}\right|=\left|1-\frac{\eta_{*}\lambda_{r}}{m}\right|=q_{*}$ holds, we have
    \[
        \left|\|PB_{\infty}(\eta)P\|_{\mathrm{op}}-q_{*}\right|
        \le\left|\left|1-\frac{\eta\lambda_{1}}{m}\right|-q_{*}\right|\vee\left|\left|1-\frac{\eta\lambda_{r}}{m}\right|-q_{*}\right|
        \lesssim|\eta-\eta_{*}|
        .
    \]
    Thus, there exists a constant $C_{1}$ such that, uniformly over $|\eta-\eta_{\infty,T}|\le C/T$,
    \begin{equation}\label{eq:B-inf-range-K-norm}
        \|PB_{\infty}(\eta)P\|_{\mathrm{op}}
        \le q_{*}+\frac{C_{1}}{T}
        .
    \end{equation}
    This implies
    \[
        \|PB_{\infty}(\eta)^{k}P\|_{\mathrm{op}}
        \le\|PB_{\infty}(\eta)P\|_{\mathrm{op}}^{k}
        \le\left(q_{*}+\frac{C_{1}}{T}\right)^{k}
        \le q_{*}^{k}\exp\left(\frac{kC_{1}}{q_{*}T}\right)
        \lesssim q_{*}^{k}
        \qquad
        (0\le k\le2T)
        .
    \]

    By Lemma~\ref{lem:joint-clt}, we have
    \[
        \|A_{n}-K\|_{\mathrm{op}}
        =O_{p}(n^{-1/2})
        ,
        \qquad
        \|\chi_{n}^{(0)}+y\|
        =O_{p}(n^{-1/2})
        .
    \]
    The assumption $T/\sqrt{n}\to0$ therefore implies $T\|A_{n}-K\|_{\mathrm{op}}=o_{p}(1)$. Consequently, we have $\mathrm{P}(\mathcal{E}_{n})\to1$.
    
    Note that $\{q_{*}+\frac{C_{1}}{T}\}_{T}$ is bounded away from zero. Thus, by Lemma~\ref{lem:local-word-bound}, on the event $\mathcal{E}_{n}$, uniformly over $|\eta-\eta_{\infty,T}|\le C/T$ and $0\le k\le2T$, we have
    \[
        \|PB_{n}(\eta)^{k}P\|_{\mathrm{op}}
        \lesssim\left(q_{*}+\frac{C_{1}}{T}\right)^{k}
        =q_{*}^{k}\left(1+\frac{C_{1}}{q_{*}T}\right)^{k}
        \le q_{*}^{k}\exp\left(\frac{C_{1}k}{q_{*}T}\right)
        \le q_{*}^{k}\exp\left(\frac{2C_{1}}{q_{*}}\right)
        \lesssim q_{*}^{k}
        ,
    \]
    and similarly,
    \[
        \|P\left(B_{n}(\eta)^{k}-B_{\infty}(\eta)^{k}\right)P\|_{\mathrm{op}}
        \lesssim k\left(q_{*}+\frac{C_{1}}{T}\right)^{k-1}\|A_{n}-K\|_{\mathrm{op}}
        \lesssim kq_{*}^{k-1}\|A_{n}-K\|_{\mathrm{op}}
        .
    \]
    On this event, we also have
    \[
        PA_{n}=A_{n}P
        ,\qquad
        PK=KP
        ,\qquad
        P(A_{n}-K)=(A_{n}-K)P
    \]
    by Lemma~\ref{lem:local-word-bound}.
\end{proof}

\section{Proof of the joint width-horizon transfer theorem}\label{app:proof-width-horizon-transfer}

Throughout this section, assume the conditions of Theorem~\ref{thm:a-b-c-summary}.

By assumption, we have $\eta_{\ell}<\eta_{*}<\eta_{u}<\frac{2m}{\lambda_{1}}$. Since Proposition~\ref{prop:finite-T-correction} implies $\eta_{\infty,T}=\eta_{*}+O(T^{-1})$, there exists a constant $C>0$ such that, for all sufficiently large $T$,
\[
    \{\eta:|\eta-\eta_{\infty,T}|\le C/T\}
    \subset(\eta_{\ell},\eta_{u})
    .
\]

Let $P=\sum_{j=1}^{r}P_{j}$ be the orthogonal projector onto $\operatorname{range}(K)$, and set $E_{n}=A_{n}-K$. By Lemma~\ref{lem:kernel-H_n-A_n-K}, for all sufficiently large $n$, we have $\ker(A_{n})\stackrel{\mathrm{a.s.}}{=}\ker(K)$ and $f_{n}^{(0)}(X)\in\operatorname{range}(H_{n})\stackrel{\mathrm{a.s.}}{=}\operatorname{range}(K)$. Consequently, we have
\begin{equation}\label{eq:range-K-prediction-error}
    Pf_{n}^{(0)}(X)
    \stackrel{\mathrm{a.s.}}{=}f_{n}^{(0)}(X)
    ,\qquad
    PE_{n}
    \stackrel{\mathrm{a.s.}}{=}E_{n}P
    \stackrel{\mathrm{a.s.}}{=}E_{n}
    .
\end{equation}

\subsection{Uniform score approximation}

\begin{lemma}[Uniform approximation of the score and its derivative]\label{lem:score-uniform-growing-T}
    Note that Proposition~\ref{prop:finite-T-correction} places $\eta_{\infty,T}$ in a $C/T$-neighborhood of $\eta_{*}$ for some $C$. For this neighborhood, we have
    \begin{align*}
        \sup_{|\eta-\eta_{\infty,T}|\le C/T}|F_{n,T}(\eta)-F_{\infty,T}(\eta)|
        &=O_{p}\left(\frac{Tq_{*}^{2T-2}}{\sqrt{n}}\right)
        ,\\
        \sup_{|\eta-\eta_{\infty,T}|\le C/T}|F_{n,T}'(\eta)-F_{\infty,T}'(\eta)|
        &=O_{p}\left(\frac{T^{2}q_{*}^{2T-3}}{\sqrt{n}}\right)
        .
    \end{align*}
\end{lemma}
\begin{proof}
    An $O_{p}(\cdot)$ bound established on an event whose probability tends to one also holds unconditionally. Thus, it suffices to show upper bounds for $\sup_{|\eta-\eta_{\infty,T}|\le C/T}|F_{n,T}(\eta)-F_{\infty,T}(\eta)|$ and $\sup_{|\eta-\eta_{\infty,T}|\le C/T}|F_{n,T}'(\eta)-F_{\infty,T}'(\eta)|$ on the event $\mathcal{E}_{n}$ defined in Lemma~\ref{lem:local-random-power-bound}. We first consider the score. Set
    \[
        M_{n,2T-1}(\eta)
        =A_{n}B_{n}(\eta)^{2T-1}
        ,
        \qquad
        M_{\infty,2T-1}(\eta)
        =KB_{\infty}(\eta)^{2T-1}
        .
    \]
    On the event $\mathcal{E}_{n}$, we have
    \begin{align*}
        &\sup_{|\eta-\eta_{\infty,T}|\le C/T}\left\|P\left(M_{n,2T-1}(\eta)-M_{\infty,2T-1}(\eta)\right)P\right\|_{\mathrm{op}}\\
        &=\sup_{|\eta-\eta_{\infty,T}|\le C/T}\|P(A_{n}-K)B_{n}(\eta)^{2T-1}P+PK(B_{n}(\eta)^{2T-1}-B_{\infty}(\eta)^{2T-1})P\|_{\mathrm{op}}\\
        &\le\|A_{n}-K\|_{\mathrm{op}}\sup_{|\eta-\eta_{\infty,T}|\le C/T}\|PB_{n}(\eta)^{2T-1}P\|_{\mathrm{op}}\\
        &\qquad+\|K\|_{\mathrm{op}}\sup_{|\eta-\eta_{\infty,T}|\le C/T}\left\|P\left(B_{n}(\eta)^{2T-1}-B_{\infty}(\eta)^{2T-1}\right)P\right\|_{\mathrm{op}}\\
        &=O_{p}\left(\frac{q_{*}^{2T-1}}{\sqrt{n}}+\frac{Tq_{*}^{2T-2}}{\sqrt{n}}\right)\\
        &=O_{p}\left(\frac{Tq_{*}^{2T-2}}{\sqrt{n}}\right)
        .
    \end{align*}
    By Lemma~\ref{lem:local-random-power-bound}, on the same event, we have
    \begin{equation}\label{eq:AnBnP}
        A_{n}
        =PA_{n}P
        ,\qquad
        B_{n}(\eta)P
        =PB_{n}(\eta)
        .
    \end{equation}
    Therefore, for every integer $k\ge0$, we have
    \[
        A_{n}B_{n}(\eta)^{k}
        =PA_{n}PB_{n}(\eta)^{k}
        =PA_{n}P B_{n}(\eta)^{k}P
        =A_{n}PB_{n}(\eta)^{k}P
        .
    \]
    Consequently, we have
    \begin{align*}
        \sup_{|\eta-\eta_{\infty,T}|\le C/T}\|M_{n,2T-1}(\eta)\|_{\mathrm{op}}
        &\le\sup_{|\eta-\eta_{\infty,T}|\le C/T}\|A_{n}\|_{\mathrm{op}}\|PB_{n}(\eta)^{2T-1}P\|_{\mathrm{op}}\\
        &\lesssim q_{*}^{2T-1}(\|K\|_{\mathrm{op}}+\|A_{n}-K\|_{\mathrm{op}})\\
        &=O_{p}(q_{*}^{2T-1})
        .
    \end{align*}
    Since $A_{n}$ commutes with $B_{n}(\eta)$, the matrix $M_{n,2T-1}(\eta)$ is symmetric. Therefore,
    \begin{align*}
        F_{n,T}(\eta)-F_{\infty,T}(\eta)
        &=y^{\top}(M_{n,2T-1}(\eta)-M_{\infty,2T-1}(\eta))y-2(\chi_{n}^{(0)}+y)^{\top}M_{n,2T-1}(\eta)y\\
        &\quad+(\chi_{n}^{(0)}+y)^{\top}M_{n,2T-1}(\eta)(\chi_{n}^{(0)}+y)
        .
    \end{align*}
    On the event $\mathcal{E}_{n}$, since both $M_{n,2T-1}(\eta)$ and $M_{\infty,2T-1}(\eta)$ vanish on $\ker(K)$, we may insert the projector $P$ in the first term:
    \[
        y^{\top}(
        M_{n,2T-1}(\eta)-M_{\infty,2T-1}(\eta))y
        =y^{\top}P(M_{n,2T-1}(\eta)-M_{\infty,2T-1}(\eta))Py
        .
    \]
    Thus, on the event $\mathcal{E}_{n}$, we obtain
    \begin{align*}
        \sup_{|\eta-\eta_{\infty,T}|\le C/T}|F_{n,T}(\eta)-F_{\infty,T}(\eta)|
        &\le\|y\|^{2}\sup_{|\eta-\eta_{\infty,T}|\le C/T}\|P(M_{n,2T-1}(\eta)-M_{\infty,2T-1}(\eta))P\|_{\mathrm{op}}\\
        &\quad+2\|\chi_{n}^{(0)}+y\|\|y\|\sup_{|\eta-\eta_{\infty,T}|\le C/T}\|M_{n,2T-1}(\eta)\|_{\mathrm{op}}\\
        &\quad+\|\chi_{n}^{(0)}+y\|^{2}\sup_{|\eta-\eta_{\infty,T}|\le C/T}\|M_{n,2T-1}(\eta)\|_{\mathrm{op}}\\
        &=O_{p}\left(\frac{Tq_{*}^{2T-2}}{\sqrt{n}}\right)+O_{p}\left(\frac{q_{*}^{2T-1}}{\sqrt{n}}\right)+O_{p}\left(\frac{q_{*}^{2T-1}}{n}\right)\\
        &=O_{p}\left(\frac{Tq_{*}^{2T-2}}{\sqrt{n}}\right)
        .
    \end{align*}

    We next consider the derivative of the score. Define
    \[
        R_{n,2T-2}(\eta)
        :=A_{n}^{2}B_{n}(\eta)^{2T-2}
        ,
        \qquad
        R_{\infty,2T-2}(\eta)
        :=K^{2}B_{\infty}(\eta)^{2T-2}
        .
    \]
    Then on the event $\mathcal{E}_{n}$, we have by Lemma~\ref{lem:local-random-power-bound} that
    \begin{align*}
        &\sup_{|\eta-\eta_{\infty,T}|\le C/T}\left\|P\left(R_{n,2T-2}(\eta)-R_{\infty,2T-2}(\eta)\right)P\right\|_{\mathrm{op}}\\
        &=\sup_{|\eta-\eta_{\infty,T}|\le C/T}\|P(A_{n}^{2}-K^{2})B_{n}(\eta)^{2T-2}P+PK^{2}(B_{n}(\eta)^{2T-2}-B_{\infty}(\eta)^{2T-2})P\|_{\mathrm{op}}\\
        &\le\|A_{n}^{2}-K^{2}\|_{\mathrm{op}}\sup_{|\eta-\eta_{\infty,T}|\le C/T}\|PB_{n}(\eta)^{2T-2}P\|_{\mathrm{op}}\\
        &\qquad+\|K\|_{\mathrm{op}}^{2}\sup_{|\eta-\eta_{\infty,T}|\le C/T}\left\|P\left(B_{n}(\eta)^{2T-2}-B_{\infty}(\eta)^{2T-2}\right)P\right\|_{\mathrm{op}}\\
        &\lesssim\|(A_{n}-K)A_{n}+K(A_{n}-K)\|_{\mathrm{op}}q_{*}^{2T-2}+Tq_{*}^{2T-3}\|A_{n}-K\|_{\mathrm{op}}\\
        &\le\|A_{n}-K\|_{\mathrm{op}}(\|A_{n}\|_{\mathrm{op}}+\|K\|_{\mathrm{op}})q_{*}^{2T-2}+Tq_{*}^{2T-3}\|A_{n}-K\|_{\mathrm{op}}\\
        &=O_{p}\left(\frac{q_{*}^{2T-2}}{\sqrt{n}}+\frac{Tq_{*}^{2T-3}}{\sqrt{n}}\right)\\
        &=O_{p}\left(\frac{Tq_{*}^{2T-3}}{\sqrt{n}}\right)
        .
    \end{align*}
    Moreover, \eqref{eq:AnBnP} implies that, for every integer $k\ge0$,
    \[
        A_{n}^{2}B_{n}(\eta)^{k}
        =PA_{n}PA_{n}PB_{n}(\eta)^{k}
        =PA_{n}^{2}PB_{n}(\eta)^{k}
        =PA_{n}^{2}PB_{n}(\eta)^{k}P
        =A_{n}^{2}PB_{n}(\eta)^{k}P
        .
    \]
    Thus, we have
    \[
        \sup_{|\eta-\eta_{\infty,T}|\le C/T}\|R_{n,2T-2}(\eta)\|_{\mathrm{op}}
        \le\sup_{|\eta-\eta_{\infty,T}|\le C/T}\|A_{n}^{2}\|_{\mathrm{op}}\|PB_{n}(\eta)^{2T-2}P\|_{\mathrm{op}}
        =O_{p}(q_{*}^{2T-2})
        .
    \]
    Note that by Proposition~\ref{prop:loss-derivatives}, we have
    \begin{align*}
        &F_{n,T}'(\eta)-F_{\infty,T}'(\eta)\\
        &=-\frac{2T-1}{m}[(\chi_{n}^{(0)})^{\top}R_{n,2T-2}(\eta)\chi_{n}^{(0)}-y^{\top}R_{\infty,2T-2}(\eta)y]\\
        &=-\frac{2T-1}{m}[y^{\top}(R_{n,2T-2}(\eta)-R_{\infty,2T-2}(\eta))y-2(\chi_{n}^{(0)}+y)^{\top}R_{n,2T-2}(\eta)y]\\
        &\quad-\frac{2T-1}{m}(\chi_{n}^{(0)}+y)^{\top}R_{n,2T-2}(\eta)(\chi_{n}^{(0)}+y)
        .
    \end{align*}
    On the event $\mathcal{E}_{n}$, since both $R_{n,2T-2}(\eta)$ and $R_{\infty,2T-2}(\eta)$ vanish on $\ker(K)$, we may insert the projector $P$ in the first term:
    \[
        y^{\top}(R_{n,2T-2}(\eta)-R_{\infty,2T-2}(\eta))y
        =y^{\top}P(R_{n,2T-2}(\eta)-R_{\infty,2T-2}(\eta))Py
        .
    \]
    Thus, on the event $\mathcal{E}_{n}$, we have
    \begin{align*}
        &\sup_{|\eta-\eta_{\infty,T}|\le C/T}|F_{n,T}'(\eta)-F_{\infty,T}'(\eta)|\\
        &\le\frac{2T-1}{m}\|y\|^{2}\sup_{|\eta-\eta_{\infty,T}|\le C/T}\|P(R_{n,2T-2}(\eta)-R_{\infty,2T-2}(\eta))P\|_{\mathrm{op}}\\
        &\quad+\frac{2T-1}{m}2\|\chi_{n}^{(0)}+y\|\|y\|\sup_{|\eta-\eta_{\infty,T}|\le C/T}\|R_{n,2T-2}(\eta)\|_{\mathrm{op}}\\
        &\quad+\frac{2T-1}{m}\|\chi_{n}^{(0)}+y\|^{2}\sup_{|\eta-\eta_{\infty,T}|\le C/T}\|R_{n,2T-2}(\eta)\|_{\mathrm{op}}\\
        &=O_{p}\left(\frac{T^{2}q_{*}^{2T-3}}{\sqrt{n}}\right)+O_{p}\left(\frac{Tq_{*}^{2T-2}}{\sqrt{n}}\right)+O_{p}\left(\frac{Tq_{*}^{2T-2}}{n}\right)\\
        &=O_{p}\left(\frac{T^{2}q_{*}^{2T-3}}{\sqrt{n}}\right)
        .
    \end{align*}
\end{proof}

\subsection{Localization of the finite-width optimizer}

\begin{lemma}[$n^{-1/2}$-consistency of the finite-width optimizer]\label{lem:preliminary-b-rate}
    Under the assumptions of Theorem~\ref{thm:a-b-c-summary}, with probability tending to one, we have $\eta_{n,T}\in(\eta_{\ell},\eta_{u})$ and $F_{n,T}(\eta_{n,T})=0$. Moreover, we have
    \[
        \eta_{n,T}-\eta_{\infty,T}
        =O_{p}(n^{-1/2})
        ,
    \]
    which implies
    \[
        |\eta_{n,T}-\eta_{\infty,T}|
        =o_{p}(T^{-1})
        .
    \]
\end{lemma}
\begin{proof}
    Define $s_{T}=Tq_{*}^{2T-2}$. By \eqref{eq:finite-score-curvature}, there exist constants $0<\underline{C}<\overline{C}<\infty$ such that
    \begin{equation}\label{eq:score-derivative-two-sided}
        \underline{C}s_{T}
        \le
        -F_{\infty,T}'(\eta_{\infty,T})
        \le
        \overline{C}s_{T}
    \end{equation}
    for all sufficiently large $T$. Differentiating \eqref{eq:score-derivative} once more gives
    \[
        F_{\infty,T}''(\eta)
        =\frac{(2T-1)(2T-2)}{m^{2}}y^{\top}K^{3}B_{\infty}(\eta)^{2T-3}y
        .
    \]
    By Lemma~\ref{lem:local-random-power-bound}, uniformly over $|\eta-\eta_{\infty,T}|\le C/T$, we have
    \[
        \|PB_{\infty}(\eta)^{k}P\|_{\mathrm{op}}
        \lesssim q_{*}^{k}
        \qquad(0\le k\le2T)
        ,
    \]
    which implies
    \[
        \sup_{|\eta-\eta_{\infty,T}|\le C/T}|F_{\infty,T}''(\eta)|
        \lesssim T^{2}\|Py\|^{2}\|K\|_{\mathrm{op}}^{3}\|PB_{\infty}(\eta)^{2T-3}P\|_{\mathrm{op}}
        \lesssim T^{2}q_{*}^{2T-3}
        .
    \]
    For every fixed $M>0$, the assumption $T/\sqrt{n}\to0$ implies $M/\sqrt{n}\le C/T$ for all sufficiently large $n$. Hence we have
    \begin{equation}\label{eq:deterministic-derivative-local}
        \sup_{|\eta-\eta_{\infty,T}|\le M/\sqrt{n}}|F_{\infty,T}'(\eta)-F_{\infty,T}'(\eta_{\infty,T})|
        =O\left(\frac{MT^{2}q_{*}^{2T-3}}{\sqrt{n}}\right)
        =o(s_{T})
        .
    \end{equation}
    Since Lemma~\ref{lem:score-uniform-growing-T} gives
    \begin{align}\label{eq:random-derivative-local}
        \sup_{|\eta-\eta_{\infty,T}|\le M/\sqrt{n}}|F_{n,T}'(\eta)-F_{\infty,T}'(\eta)|
        =O_{p}\left(\frac{T^{2}q_{*}^{2T-3}}{\sqrt{n}}\right)
        =o_{p}(s_{T})
        ,
    \end{align}
    we have $\sup_{|\eta-\eta_{\infty,T}|\le M/\sqrt{n}}F_{n,T}'(\eta)\le F_{\infty,T}'(\eta_{\infty,T})+o_{p}(s_{T})$. Combining this with \eqref{eq:score-derivative-two-sided}, we obtain
    \begin{equation}\label{eq:negative-local-derivative}
        \mathrm{P}\left(\sup_{|\eta-\eta_{\infty,T}|\le M/\sqrt{n}}F_{n,T}'(\eta)\le-\frac{\underline{C}}{2}s_{T}\right)
        \to1
    \end{equation}
    for every fixed $M>0$.

    By Corollary~\ref{cor:optimizer-characterization} and Theorem~\ref{thm:finite-minimax}, we have $F_{\infty,T}(\eta_{\infty,T})=0$. Thus the first bound in Lemma~\ref{lem:score-uniform-growing-T} gives
    \[
        F_{n,T}(\eta_{\infty,T})
        =O_{p}\left(\frac{s_{T}}{\sqrt{n}}\right)
        .
    \]
    Therefore, for any fixed $\epsilon>0$, there exists $M_{0}<\infty$ such that
    \begin{equation}\label{eq:F-asymp-tight}
        \liminf_{n\to\infty}\mathrm{P}\left(|F_{n,T}(\eta_{\infty,T})|\le\frac{M_{0}s_{T}}{\sqrt{n}}\right)
        \ge1-\epsilon
        .
    \end{equation}
    Choose $M>2M_{0}/\underline{C}$. On the intersection of this event and the event in \eqref{eq:negative-local-derivative}, the mean value theorem gives
    \begin{align*}
        F_{n,T}\left(\eta_{\infty,T}-\frac{M}{\sqrt{n}}\right)
        &\ge-\frac{M_{0}s_{T}}{\sqrt{n}}+\frac{\underline{C}Ms_{T}}{2\sqrt{n}}
        >0
        ,\\
        F_{n,T}\left(\eta_{\infty,T}+\frac{M}{\sqrt{n}}\right)
        &\le\frac{M_{0}s_{T}}{\sqrt{n}}-\frac{\underline{C}Ms_{T}}{2\sqrt{n}}
        <0
        .
    \end{align*}

    By Proposition~\ref{prop:loss-derivatives}, $F_{n,T}$ is continuous. On the event in \eqref{eq:negative-local-derivative}, it is strictly decreasing throughout the interval $|\eta-\eta_{\infty,T}|\le M/\sqrt{n}$. Hence the preceding sign inequalities and the intermediate value theorem imply that on the intersection of events in \eqref{eq:negative-local-derivative} and \eqref{eq:F-asymp-tight}, $F_{n,T}$ has a unique root in
    \[
        \left(\eta_{\infty,T}-\frac{M}{\sqrt{n}},\eta_{\infty,T}+\frac{M}{\sqrt{n}}\right)
        \subset(\eta_{\ell},\eta_{u})
        .
    \]
    For all sufficiently large $n$, Proposition~\ref{prop:loss-convex-as} and Corollary~\ref{cor:optimizer-characterization} imply that $\phi_{n,T}$ is almost surely strictly convex. Therefore this root is the unique minimizer of $\phi_{n,T}$ over $\mathcal{I}$, namely $\eta_{n,T}$. Consequently, with probability at least $1-\epsilon+o(1)$, we have
    \[
        |\eta_{n,T}-\eta_{\infty,T}|<\frac{M}{\sqrt{n}}
        .
    \]
    Since $\epsilon>0$ was arbitrary, this proves that $\eta_{n,T}\in(\eta_{\ell},\eta_{u})$ with probability tending to one and that $\eta_{n,T}-\eta_{\infty,T}=O_{p}(n^{-1/2})$.
\end{proof}

\subsection{First-order perturbation expansion}

We next derive the first-order triangular-array expansion of the score. Set
\[
    B_{T}
    =B_{\infty}(\eta_{\infty,T})
    ,\qquad
    \Delta_{n}
    =-\frac{\eta_{\infty,T}}{m}E_{n}
    .
\]
Then we have
\[
    B_{n}(\eta_{\infty,T})
    =B_{T}+\Delta_{n}
    .
\]
For $1\le k\le2T$, define
\[
    \mathcal{L}_{k,n}
    =\sum_{s=0}^{k-1}B_{T}^{s}\Delta_{n}B_{T}^{k-1-s}
    ,\qquad
    \mathcal{R}_{k,n}
    =(B_{T}+\Delta_{n})^{k}-B_{T}^{k}-\mathcal{L}_{k,n}
    .
\]
Then by Lemma~\ref{lem:matrix-power-derivative}, we have
\[
    \mathcal{L}_{k,n}
    =D\Psi_{k}(B_{T})[\Delta_{n}]
    ,\qquad
    \Psi_{k}(B)=B^{k}
    .
\]

\begin{lemma}[Linear and quadratic matrix-power perturbations]\label{lem:power-derivative}
    Let $P$ denote the orthogonal projector onto $\operatorname{range}(K)$. Then we have
    \[
        \|P\mathcal{L}_{k,n}P\|_{\mathrm{op}}
        =O_{p}(kq_{*}^{k-1}\|E_{n}\|_{\mathrm{op}})
        ,\qquad
        \|P\mathcal{R}_{k,n}P\|_{\mathrm{op}}
        =O_{p}(k^{2}q_{*}^{k-2}\|E_{n}\|_{\mathrm{op}}^{2})
        .
    \]
\end{lemma}
\begin{proof}
    Define the event $\mathcal{E}_{n}$ as in Lemma~\ref{lem:local-random-power-bound}. By Lemma~\ref{lem:local-random-power-bound}, we have $\mathrm{P}(\mathcal{E}_{n})\to1$. An $O_{p}(\cdot)$ bound established on an event whose probability tends to one also holds unconditionally. Thus, it suffices to show that
    \[
        \|P\mathcal{L}_{k,n}P\|_{\mathrm{op}}
        \lesssim kq_{*}^{k-1}\|E_{n}\|_{\mathrm{op}}
        ,\qquad
        \|P\mathcal{R}_{k,n}P\|_{\mathrm{op}}
        \lesssim k^{2}q_{*}^{k-2}\|E_{n}\|_{\mathrm{op}}^{2}
    \]
    holds on $\mathcal{E}_{n}$.

    First, note that Lemma~\ref{lem:local-random-power-bound} gives
    \[
        \|PB_{T}^{j}P\|_{\mathrm{op}}
        \lesssim q_{*}^{j}
        \qquad(0\le j\le2T)
        .
    \]
    Since Proposition~\ref{prop:finite-T-correction} implies $\eta_{\infty,T}=\eta_{*}+O(T^{-1})$, we have $\sup_{T}\eta_{\infty,T}<\infty$. Hence, by the definition of $\Delta_{n}$, we have
    \begin{equation}\label{eq:B-difference-norm}
        \|\Delta_{n}\|_{\mathrm{op}}
        \lesssim\|E_{n}\|_{\mathrm{op}}
        .
    \end{equation}
    Therefore, we obtain
    \begin{align*}
        \|P\mathcal{L}_{k,n}P\|_{\mathrm{op}}
        &\le\sum_{s=0}^{k-1}\|P(B_{T}^{s}\Delta_{n}B_{T}^{k-1-s})P\|_{\mathrm{op}}
        \le\sum_{s=0}^{k-1}\|PB_{T}^{s}P\|_{\mathrm{op}}\|\Delta_{n}\|_{\mathrm{op}}\|PB_{T}^{k-1-s}P\|_{\mathrm{op}}\\
        &\lesssim\sum_{s=0}^{k-1}q_{*}^{s}\|\Delta_{n}\|_{\mathrm{op}}q_{*}^{k-1-s}
        =kq_{*}^{k-1}\|\Delta_{n}\|_{\mathrm{op}}
        \lesssim kq_{*}^{k-1}\|E_{n}\|_{\mathrm{op}}
        .
    \end{align*}
    Moreover, since $B_{T}+\Delta_{n}=B_{n}(\eta_{\infty,T})$, Lemma~\ref{lem:local-random-power-bound} gives
    \[
        \|P(B_{T}+\Delta_{n})^{j}P\|_{\mathrm{op}}
        \lesssim q_{*}^{j}
    \]
    uniformly over $0\le j\le2T$ on $\mathcal{E}_{n}$.

    For the remainder, the proof of Lemma~\ref{lem:matrix-power-derivative} gives the identity
    \[
        \mathcal{R}_{k,n}
        =\sum_{s=1}^{k-1}\sum_{j=0}^{s-1}(B_{T}+\Delta_{n})^{j}\Delta_{n}B_{T}^{s-1-j}\Delta_{n}B_{T}^{k-1-s}
    \]
    for $k\ge2$, while $\mathcal{R}_{1,n}=0$. Hence, using the preceding projected power bounds and \eqref{eq:B-difference-norm}, we obtain
    \begin{align*}
        \|P\mathcal{R}_{k,n}P\|_{\mathrm{op}}
        &\le\sum_{s=1}^{k-1}\sum_{j=0}^{s-1}\|P(B_{T}+\Delta_{n})^{j}P\|_{\mathrm{op}}\|\Delta_{n}\|_{\mathrm{op}}\|PB_{T}^{s-1-j}P\|_{\mathrm{op}}\|\Delta_{n}\|_{\mathrm{op}}\|PB_{T}^{k-1-s}P\|_{\mathrm{op}}\\
        &\lesssim\sum_{s=1}^{k-1}\sum_{j=0}^{s-1}q_{*}^{j}\|\Delta_{n}\|_{\mathrm{op}}^{2}q_{*}^{s-1-j}q_{*}^{k-1-s}
        =\sum_{s=1}^{k-1}sq_{*}^{k-2}\|\Delta_{n}\|_{\mathrm{op}}^{2}
        =\frac{k(k-1)}{2}q_{*}^{k-2}\|\Delta_{n}\|_{\mathrm{op}}^{2}\\
        &\lesssim k^{2}q_{*}^{k-2}\|E_{n}\|_{\mathrm{op}}^{2}
        .
    \end{align*}
\end{proof}

\subsection{Limits of the first-order sensitivities of the loss and score}

\begin{lemma}[Asymptotics of the local Fr\'echet coefficients]\label{lem:frechet-coefficient-limits}
    We have
    \[
        \frac{PQ_{F,T}P}{Tq_{*}^{2T-2}}
        \to Q_{F,*}
        :=-\frac{2\eta_{*}h_{*}}{m}(u_{1}u_{1}^{\top}+u_{r}u_{r}^{\top})
        ,\qquad
        \frac{PQ_{L,T}P}{Tq_{*}^{2T-1}}
        \to Q_{L,*}
    \]
    in Frobenius norm, where $Q_{L,*}$ is defined by $Q_{L,*}=m^{-2}\eta_{*}(c_{1}a_{1}^{*}u_{1}u_{1}^{\top}-c_{r}a_{r}^{*}u_{r}u_{r}^{\top})$. In addition, we have
    \[
        \frac{\|Pg_{F,T}\|}{Tq_{*}^{2T-2}}
        \to0
        ,\qquad
        \frac{\|Pg_{L,T}\|}{Tq_{*}^{2T-1}}
        \to0
        .
    \]
\end{lemma}
\begin{proof}
    Set
    \[
        \Gamma_{ij,T}^{(k)}
        :=\sum_{s=0}^{k-1}r_{i,T}^{k-1-s}r_{j,T}^{s}
        =\begin{cases}
            kr_{i,T}^{k-1}&(i=j)
            ,\\
            (r_{i,T}^{k}-r_{j,T}^{k})/(r_{i,T}-r_{j,T})&(i\neq j)
            .
        \end{cases}
        .
    \]
    Since $\lambda_{i}\neq\lambda_{j}$ for $i\neq j$ and $\eta_{\infty,T}\to\eta_{*}>0$, the denominator $|r_{i,T}-r_{j,T}|$ is bounded away from zero for all sufficiently large $T$.
    Directly projecting $Q_{F,T}$ onto the two eigenspaces gives
    \[
        P_{i}Q_{F,T}P_{j}
        =\left(\frac{r_{i,T}^{2T-1}+r_{j,T}^{2T-1}}{2}-\frac{\eta_{\infty,T}(\lambda_{i}+\lambda_{j})}{2m}\Gamma_{ij,T}^{(2T-1)}
        \right)(P_{i}y)(P_{j}y)^{\top}
        .
    \]
    In particular, we have
    \[
        P_{j}Q_{F,T}P_{j}
        =\left(r_{j,T}^{2T-1}-\frac{\eta_{\infty,T}(2T-1)\lambda_{j}}{m}r_{j,T}^{2T-2}\right)(P_{j}y)(P_{j}y)^{\top}
        .
    \]
    By Lemma~\ref{lem:endpoint-contractions}, for every $j=2,\ldots,r-1$, we have $P_{j}Q_{F,T}P_{j}=o(Tq_{*}^{2T-2})$. For $i\neq j$, noting that $|\Gamma_{ij,T}^{(2T-1)}|\lesssim |r_{i,T}|^{2T-1}+|r_{j,T}|^{2T-1}$, we have $P_{i}Q_{F,T}P_{j}=O(q_{*}^{2T-1})=o(Tq_{*}^{2T-2})$ by Lemma~\ref{lem:endpoint-contractions}. By Assumption~\ref{asmp:endpoint-support}, we can write $(P_{1}y)(P_{1}y)^{\top}=c_{1}u_{1}u_{1}^{\top}$ and $(P_{r}y)(P_{r}y)^{\top}=c_{r}u_{r}u_{r}^{\top}$. Using Lemma~\ref{lem:endpoint-contractions} and $\eta_{\infty,T}\to\eta_{*}$, we have
    \begin{align*}
        \frac{P_{1}Q_{F,T}P_{1}}{Tq_{*}^{2T-2}}
        &=\left(r_{1,T}^{2T-1}-\frac{\eta_{\infty,T}(2T-1)\lambda_{1}}{m}r_{1,T}^{2T-2}\right)\frac{(P_{1}y)(P_{1}y)^{\top}}{Tq_{*}^{2T-2}}\\
        &\to-\frac{\eta_{*}2\lambda_{1}}{m}a_{1}^{*}(P_{1}y)(P_{1}y)^{\top}
        =-\frac{2\eta_{*}h_{*}}{m}u_{1}u_{1}^{\top}
        ,\\
        \frac{P_{r}Q_{F,T}P_{r}}{Tq_{*}^{2T-2}}
        &=\left(r_{r,T}^{2T-1}-\frac{\eta_{\infty,T}(2T-1)\lambda_{r}}{m}r_{r,T}^{2T-2}\right)\frac{(P_{r}y)(P_{r}y)^{\top}}{Tq_{*}^{2T-2}}\\
        &\to-\frac{\eta_{*}2\lambda_{r}}{m}a_{r}^{*}(P_{r}y)(P_{r}y)^{\top}
        =-\frac{2\eta_{*}h_{*}}{m}u_{r}u_{r}^{\top}
        .
    \end{align*}
    Thus we obtain
    \[
        \frac{PQ_{F,T}P}{Tq_{*}^{2T-2}}
        =\sum_{i=1}^{r}\sum_{j=1}^{r}\frac{P_{i}Q_{F,T}P_{j}}{Tq_{*}^{2T-2}}
        =\frac{P_{1}Q_{F,T}P_{1}}{Tq_{*}^{2T-2}}+\frac{P_{r}Q_{F,T}P_{r}}{Tq_{*}^{2T-2}}+o(1)
        \to Q_{F,*}
        .
    \]    
    The loss coefficient is handled in the same way. Projecting $Q_{L,T}$ gives
    \[
        P_{i}Q_{L,T}P_{j}
        =-\frac{\eta_{\infty,T}}{2m^{2}}\Gamma_{ij,T}^{(2T)}(P_{i}y)(P_{j}y)^{\top}
        .
    \]
    Thus all off-diagonal blocks and all interior diagonal blocks are $o(Tq_{*}^{2T-1})$. At the two endpoints, we have by Lemma~\ref{lem:endpoint-contractions} and $\eta_{\infty,T}\to\eta_{*}$ that
    \begin{align*}
        &\frac{P_{1}Q_{L,T}P_{1}}{Tq_{*}^{2T-1}}
        =-\frac{\eta_{\infty,T}2Tr_{1,T}^{2T-1}(P_{1}y)(P_{1}y)^{\top}}{2m^{2}Tq_{*}^{2T-1}}
        \to-\frac{\eta_{*}}{m^{2}}(-a_{1}^{*})(P_{1}y)(P_{1}y)^{\top}
        =\frac{\eta_{*}}{m^{2}}c_{1}a_{1}^{*}u_{1}u_{1}^{\top}
        ,\\
        &\frac{P_{r}Q_{L,T}P_{r}}{Tq_{*}^{2T-1}}
        =-\frac{\eta_{\infty,T}2Tr_{r,T}^{2T-1}(P_{r}y)(P_{r}y)^{\top}}{2m^{2}Tq_{*}^{2T-1}}
        \to-\frac{\eta_{*}}{m^{2}}a_{r}^{*}(P_{r}y)(P_{r}y)^{\top}
        =-\frac{\eta_{*}}{m^{2}}c_{r}a_{r}^{*}u_{r}u_{r}^{\top}
        .
    \end{align*}
    Thus we obtain
    \[
        \frac{PQ_{L,T}P}{Tq_{*}^{2T-1}}
        =\sum_{i=1}^{r}\sum_{j=1}^{r}\frac{P_{i}Q_{L,T}P_{j}}{Tq_{*}^{2T-1}}
        =\frac{P_{1}Q_{L,T}P_{1}}{Tq_{*}^{2T-1}}+\frac{P_{r}Q_{L,T}P_{r}}{Tq_{*}^{2T-1}}+o(1)
        \to Q_{L,*}
        .
    \]
    Finally, on $\operatorname{range}(K)$, Lemma~\ref{lem:endpoint-contractions} gives
    \[
        \|PB_{T}^{2T-1}P\|_{\mathrm{op}}
        =\left\|\sum_{i=1}^{r}r_{i,T}^{2T-1}P_{i}\right\|
        \le\sum_{i=1}^{r}|r_{i,T}|^{2T-1}
        \lesssim q_{*}^{2T-1}
        ,
    \]
    and similarly, $\|PB_{T}^{2T}P\|_{\mathrm{op}}\lesssim q_{*}^{2T}$. Thus we have
    \[
        \frac{\|Pg_{F,T}\|}{Tq_{*}^{2T-2}}
        \to0
        ,\qquad
        \frac{\|Pg_{L,T}\|}{Tq_{*}^{2T-1}}
        \to0
        .
    \]
\end{proof}

\subsection{Score linearization and weak convergence of the optimal learning rate}

\begin{lemma}[Score linearization with an explicit remainder]\label{lem:score-linearization}
    Define
    \begin{align*}
        &S_{0,n}
        :=KB_{T}^{2T-1}
        ,\qquad
        S_{1,n}
        :=E_{n}B_{T}^{2T-1}+K\mathcal{L}_{2T-1,n}
        ,\\
        &S_{2,n}
        :=K\mathcal{R}_{2T-1,n}+E_{n}\mathcal{L}_{2T-1,n}+E_{n}\mathcal{R}_{2T-1,n}
        ,
    \end{align*}
    and
    \[
        R_{F,n,T}
        :=-2f_{n}^{(0)}(X)^{\top}(S_{1,n}+S_{2,n})y+f_{n}^{(0)}(X)^{\top}(S_{0,n}+S_{1,n}+S_{2,n})f_{n}^{(0)}(X)+y^{\top}S_{2,n}y
        .
    \]
    Then we have
    \[
        F_{n,T}(\eta_{\infty,T})-F_{\infty,T}(\eta_{\infty,T})
        =(g_{F,T})^{\top}f_{n}^{(0)}(X)+\left\langle Q_{F,T},E_{n}\right\rangle_{\mathrm{F}}+R_{F,n,T}
    \]
    with
    \[
        |R_{F,n,T}|
        =O_{p}\left(\frac{T^{2}q_{*}^{2T-3}}{n}\right)
        ,
    \]
    which consequently implies
    \[
        \frac{\sqrt{n}R_{F,n,T}}{Tq_{*}^{2T-2}}
        =o_{p}(1)
        .
    \]
\end{lemma}
\begin{proof}
    By Lemma~\ref{lem:kernel-H_n-A_n-K}, we have $PE_{n}=E_{n}P=E_{n}$ almost surely for all sufficiently large $n$. Hence $P$ also commutes with $\Delta_{n}$, $B_{T}$, $\mathcal{L}_{k,n}$, and $\mathcal{R}_{k,n}$ almost surely, and all three matrices $S_{0,n},S_{1,n},S_{2,n}$ vanish on $\ker(K)$. In addition, $A_{n}B_{n}(\eta_{\infty,T})^{2T-1}$ and $S_{0,n}$ are symmetric. $S_{1,n}$ is the Fr\'echet derivative at $K$, in the symmetric direction $E_{n}$, of the map
    \[
        \mathcal{F}(A)
        :=A\left(I_{m}-\frac{\eta_{\infty,T}}{m}A\right)^{2T-1}.
    \]
    Since $\mathcal{F}(A)$ is symmetric whenever $A$ is symmetric, and $K+tE_{n}$ is symmetric for every $t\in\mathbb{R}$, the difference quotient $(\mathcal{F}(K+tE_{n})-\mathcal{F}(K))/t$ is symmetric. Passing to the limit $t\to0$ therefore shows that $S_{1,n}=D\mathcal{F}(K)[E_{n}]$ is symmetric. Consequently, $S_{2,n}$ is symmetric as well.

    Note that we have
    \[
        A_{n}B_{n}(\eta_{\infty,T})^{2T-1}
        =S_{0,n}+S_{1,n}+S_{2,n}
        .
    \]
    Thus, using $\chi_{n}^{(0)}=-y+f_{n}^{(0)}(X)$ and $F_{\infty,T}(\eta_{\infty,T})=y^{\top}S_{0,n}y$, we obtain
    \begin{align*}
        &F_{n,T}(\eta_{\infty,T})-F_{\infty,T}(\eta_{\infty,T})\\
        &=-2f_{n}^{(0)}(X)^{\top}S_{0,n}y+y^{\top}S_{1,n}y-2f_{n}^{(0)}(X)^{\top}(S_{1,n}+S_{2,n})y\\
        &\qquad+f_{n}^{(0)}(X)^{\top}(S_{0,n}+S_{1,n}+S_{2,n})f_{n}^{(0)}(X)+y^{\top}S_{2,n}y\\
        &=(g_{F,T})^{\top}f_{n}^{(0)}(X)+y^{\top}S_{1,n}y+R_{F,n,T}
        .
    \end{align*}
    Since the definitions of $\mathcal{L}_{2T-1,n}$ and $\Delta_{n}$ give
    \[
        y^{\top}S_{1,n}y
        =y^{\top}E_{n}B_{T}^{2T-1}y-\frac{\eta_{\infty,T}}{m}\sum_{s=0}^{2T-2}y^{\top}KB_{T}^{s}E_{n}B_{T}^{2T-2-s}y
        =\left\langle Q_{F,T},E_{n}\right\rangle_{\mathrm{F}}
        ,
    \]
    we have
    \[
        F_{n,T}(\eta_{\infty,T})-F_{\infty,T}(\eta_{\infty,T})
        =(g_{F,T})^{\top}f_{n}^{(0)}(X)+\langle Q_{F,T},E_{n}\rangle_{\mathrm{F}}+R_{F,n,T}
        .
    \]

    We next bound the remainder. Lemma~\ref{lem:local-random-power-bound} implies
    \[
        \|PS_{0,n}P\|_{\mathrm{op}}
        =O_{p}(q_{*}^{2T-1})
        .
    \]
    Lemmas~\ref{lem:local-random-power-bound} and~\ref{lem:power-derivative} imply
    \[
        \|PS_{1,n}P\|_{\mathrm{op}}
        =O_{p}(Tq_{*}^{2T-2}\|E_{n}\|_{\mathrm{op}})
        .
    \]
    Lemma~\ref{lem:power-derivative} and $T\|E_{n}\|_{\mathrm{op}}=o_{p}(1)$ imply
    \[
        \|PS_{2,n}P\|_{\mathrm{op}}
        =O_{p}(T^{2}q_{*}^{2T-3}\|E_{n}\|_{\mathrm{op}}^{2})
        .
    \]
    Thus we have
    \[
        \|PS_{1,n}P\|_{\mathrm{op}}+\|PS_{2,n}P\|_{\mathrm{op}}
        =O_{p}(Tq_{*}^{2T-2}\|E_{n}\|_{\mathrm{op}})
    \]
    and
    \[
        \|PS_{0,n}P\|_{\mathrm{op}}+\|PS_{1,n}P\|_{\mathrm{op}}+\|PS_{2,n}P\|_{\mathrm{op}}
        =O_{p}(q_{*}^{2T-1})
        .
    \]
    By \eqref{eq:range-K-prediction-error}, we have $f_{n}^{(0)}(X)\stackrel{\mathrm{a.s.}}{=}Pf_{n}^{(0)}(X)$ and $S_{j,n}\stackrel{\mathrm{a.s.}}{=}PS_{j,n}P$ for $j=0,1,2$. Hence, the definition of $R_{F,n,T}$ and the preceding bounds give
    \begin{align*}
        &|R_{F,n,T}|\\
        &\le|2f_{n}^{(0)}(X)^{\top}(S_{1,n}+S_{2,n})y|+|f_{n}^{(0)}(X)^{\top}(S_{0,n}+S_{1,n}+S_{2,n})f_{n}^{(0)}(X)|+|y^{\top}S_{2,n}y|\\
        &=O_{p}\left(\|f_{n}^{(0)}(X)\|(\|PS_{1,n}P\|_{\mathrm{op}}+\|PS_{2,n}P\|_{\mathrm{op}})+\|f_{n}^{(0)}(X)\|^{2}\sum_{j=0}^{2}\|PS_{j,n}P\|_{\mathrm{op}}+\|PS_{2,n}P\|_{\mathrm{op}}\right)\\
        &=O_{p}(\|f_{n}^{(0)}(X)\|Tq_{*}^{2T-2}\|E_{n}\|_{\mathrm{op}}+\|f_{n}^{(0)}(X)\|^{2}q_{*}^{2T-1}+T^{2}q_{*}^{2T-3}\|E_{n}\|_{\mathrm{op}}^{2})\\
        &=O_{p}\left(\frac{T^{2}q_{*}^{2T-3}}{n}\right)
        ,
    \end{align*}
    where the last equality follows from $\|f_{n}^{(0)}(X)\|=O_{p}(n^{-1/2})$ and $\|E_{n}\|_{\mathrm{op}}=O_{p}(n^{-1/2})$ by Lemma~\ref{lem:joint-clt}. Consequently, we have
    \[
        \frac{\sqrt{n}|R_{F,n,T}|}{Tq_{*}^{2T-2}}
        =O_{p}\left(\frac{T}{q_{*}\sqrt{n}}\right)
        =o_{p}(1)
        .
    \]
\end{proof}

\begin{theorem}[Limit distribution of the optimal learning rate at growing $T$]\label{thm:b-limit-dist}
    Under the assumptions of Theorem~\ref{thm:a-b-c-summary}, the finite-width optimizer is interior with probability tending to one and we have
    \[
        \sqrt{n}(\eta_{n,T}-\eta_{\infty,T})
        \stackrel{d}{\longrightarrow}\Omega_{\eta,*}
        :=-\frac{2m}{(\lambda_{1}+\lambda_{r})^{2}}(u_{1}^{\top}\Xi u_{1}+u_{r}^{\top}\Xi u_{r})
        ,
    \]
    where $\Xi$ is defined in Lemma~\ref{lem:joint-clt}.
\end{theorem}
\begin{proof}
    Note that by Corollary~\ref{cor:optimizer-characterization}, we have $F_{\infty,T}(\eta_{\infty,T})=0$. Thus, Lemma~\ref{lem:score-linearization} implies
    \[
        \frac{\sqrt{n}F_{n,T}(\eta_{\infty,T})}{Tq_{*}^{2T-2}}
        =\frac{\sqrt{n}[(g_{F,T})^{\top}f_{n}^{(0)}(X)+\langle Q_{F,T},E_{n}\rangle_{\mathrm{F}}]}{Tq_{*}^{2T-2}}+o_{p}(1)
        .
    \]
    By \eqref{eq:range-K-prediction-error} and Lemmas~\ref{lem:frechet-coefficient-limits} and~\ref{lem:joint-clt}, we have
    \[
        \left|\frac{\sqrt{n}(g_{F,T})^{\top}f_{n}^{(0)}(X)}{Tq_{*}^{2T-2}}\right|
        \stackrel{\mathrm{a.s.}}{=}\left|\frac{\sqrt{n}(Pg_{F,T})^{\top}f_{n}^{(0)}(X)}{Tq_{*}^{2T-2}}\right|
        \le\frac{\|Pg_{F,T}\|}{Tq_{*}^{2T-2}}\|\sqrt{n}f_{n}^{(0)}(X)\|
        =o_{p}(1)
        .
    \]
    Moreover, since $E_{n}=PE_{n}P$ by \eqref{eq:range-K-prediction-error}, we have by Lemmas~\ref{lem:frechet-coefficient-limits} and~\ref{lem:joint-clt} that
    \[
        \frac{\sqrt{n}\langle Q_{F,T},E_{n}\rangle_{\mathrm{F}}}{Tq_{*}^{2T-2}}
        =\left\langle\frac{PQ_{F,T}P}{Tq_{*}^{2T-2}},\sqrt{n}E_{n}\right\rangle_{\mathrm{F}}
        \stackrel{d}{\longrightarrow}\langle Q_{F,*},\Xi\rangle_{\mathrm{F}}
        =-\frac{2\eta_{*}h_{*}}{m}(u_{1}^{\top}\Xi u_{1}+u_{r}^{\top}\Xi u_{r})
        .
    \]
    Therefore we have
    \begin{equation}\label{eq:normalized-score-limit}
        \frac{\sqrt{n}F_{n,T}(\eta_{\infty,T})}{Tq_{*}^{2T-2}}
        =\frac{\sqrt{n}\langle Q_{F,T},E_{n}\rangle_{\mathrm{F}}}{Tq_{*}^{2T-2}}+o_{p}(1)
        \stackrel{d}{\longrightarrow}-\frac{2\eta_{*}h_{*}}{m}(u_{1}^{\top}\Xi u_{1}+u_{r}^{\top}\Xi u_{r})
        .
    \end{equation}

    By Proposition~\ref{prop:loss-derivatives}, we can write
    \[
        F_{\infty,T}'(\eta_{\infty,T})
        =-\frac{2T-1}{m}\sum_{j=1}^{r}\lambda_{j}^{2}\|P_{j}y\|^{2}r_{j,T}^{2T-2}
        .
    \]
    Thus, by Lemma~\ref{lem:endpoint-contractions}, we have
    \[
        \frac{F_{\infty,T}'(\eta_{\infty,T})}{Tq_{*}^{2T-2}}
        =-\frac{2T-1}{mT}\sum_{j=1}^{r}\lambda_{j}^{2}\|P_{j}y\|^{2}\frac{r_{j,T}^{2T-2}}{q_{*}^{2T-2}}
        \to-\frac{2}{m}(\lambda_{1}^{2}\|P_{1}y\|^{2}a_{1}^{*}+\lambda_{r}^{2}\|P_{r}y\|^{2}a_{r}^{*})
        .
    \]
    By Assumption~\ref{asmp:endpoint-support}, we have  $\|P_{1}y\|^{2}=c_{1}$ and $\|P_{r}y\|^{2}=c_{r}$. Therefore we have
    \begin{equation}\label{eq:normalized-denominator}
        \frac{F_{\infty,T}'(\eta_{\infty,T})}{Tq_{*}^{2T-2}}
        \to-\frac{2}{m}(c_{1}\lambda_{1}^{2}a_{1}^{*}+c_{r}\lambda_{r}^{2}a_{r}^{*})
        =-\frac{2h_{*}}{m}(\lambda_{1}+\lambda_{r})
        <0
        .
    \end{equation}

    By the mean value theorem, there exists a random $\tilde{\eta}_{n,T}$ between $\eta_{n,T}$ and $\eta_{\infty,T}$ such that
    \[
        F_{n,T}(\eta_{n,T})-F_{n,T}(\eta_{\infty,T})
        =F_{n,T}'(\tilde{\eta}_{n,T})(\eta_{n,T}-\eta_{\infty,T})
        .
    \]
    Since $F_{n,T}(\eta_{n,T})=0$ on an event with probability tending to one, we have on the same event that
    \begin{equation}\label{eq:optimizer-difference}
        -F_{n,T}(\eta_{\infty,T})
        =F_{n,T}'(\tilde{\eta}_{n,T})(\eta_{n,T}-\eta_{\infty,T})
        .
    \end{equation}
    By Lemma~\ref{lem:preliminary-b-rate}, we have $\tilde{\eta}_{n,T}-\eta_{\infty,T}=O_{p}(n^{-1/2})$. Hence \eqref{eq:deterministic-derivative-local} and \eqref{eq:random-derivative-local} in the proof of Lemma~\ref{lem:preliminary-b-rate} imply
    \begin{equation}\label{eq:denominator-difference}
        F_{n,T}'(\tilde{\eta}_{n,T})-F_{\infty,T}'(\eta_{\infty,T})
        =o_{p}(Tq_{*}^{2T-2})
        .
    \end{equation}
    By \eqref{eq:normalized-denominator}, this further yields
    \[
        \frac{F_{n,T}'(\tilde{\eta}_{n,T})}{Tq_{*}^{2T-2}}
        =-\frac{2h_{*}}{m}(\lambda_{1}+\lambda_{r})+o_{p}(1)
        ,
    \]
    thus $F_{n,T}'(\tilde{\eta}_{n,T})$ is negative with probability tending to one. Combining this with \eqref{eq:optimizer-difference}, we obtain, on an event with probability tending to one,
    \[
        \sqrt{n}(\eta_{n,T}-\eta_{\infty,T})
        =-\frac{\sqrt{n}F_{n,T}(\eta_{\infty,T})/Tq_{*}^{2T-2}}{F_{n,T}'(\tilde{\eta}_{n,T})/Tq_{*}^{2T-2}}
        .
    \]
    Therefore, \eqref{eq:normalized-score-limit}, \eqref{eq:normalized-denominator}, and \eqref{eq:denominator-difference} yield
    \[
        \sqrt{n}(\eta_{n,T}-\eta_{\infty,T})
        \stackrel{d}{\longrightarrow}-\frac{\eta_{*}}{\lambda_{1}+\lambda_{r}}(u_{1}^{\top}\Xi u_{1}+u_{r}^{\top}\Xi u_{r})
        =-\frac{2m}{(\lambda_{1}+\lambda_{r})^{2}}(u_{1}^{\top}\Xi u_{1}+u_{r}^{\top}\Xi u_{r})
        .
    \]
\end{proof}

\subsection{Fluctuations and weak convergence of the optimal loss}

\begin{lemma}[Loss linearization with a quadratic remainder]\label{lem:loss-linearization}
    Define
    \[
        R_{L,n,T}
        :=\frac{1}{2m}[y^{\top}\mathcal{R}_{2T,n}y-2f_{n}^{(0)}(X)^{\top}(\mathcal{L}_{2T,n}+\mathcal{R}_{2T,n})y+f_{n}^{(0)}(X)^{\top}B_{n}(\eta_{\infty,T})^{2T}f_{n}^{(0)}(X)]
        .
    \]
    Then we have
    \[
        \phi_{n,T}(\eta_{\infty,T})-\phi_{\infty,T}(\eta_{\infty,T})
        =(g_{L,T})^{\top}f_{n}^{(0)}(X)+\langle Q_{L,T},E_{n}\rangle_{\mathrm{F}}+R_{L,n,T}
        .
    \]
    Thus
    \[
        |R_{L,n,T}|
        =O_{p}\left(\frac{T^{2}q_{*}^{2T-2}}{n}\right)
        ,
    \]
    which implies
    \[
        \frac{\sqrt{n}R_{L,n,T}}{Tq_{*}^{2T-1}}
        =o_{p}(1)
        .
    \]
\end{lemma}
\begin{proof}
    Since $B_{n}(\eta_{\infty,T})^{2T}=B_{T}^{2T}+\mathcal{L}_{2T,n}+\mathcal{R}_{2T,n}$, expanding $\chi_{n}^{(0)}=-y+f_{n}^{(0)}(X)$ in the quadratic loss gives
    \begin{align*}
        \phi_{n,T}(\eta_{\infty,T})-\phi_{\infty,T}(\eta_{\infty,T})
        &=\frac{1}{2m}(\chi_{n}^{(0)})^{\top}B_{n}(\eta_{\infty,T})^{2T}\chi_{n}^{(0)}-\frac{1}{2m}y^{\top}B_{T}^{2T}y\\
        &=-\frac{1}{m}f_{n}^{(0)}(X)^{\top}B_{T}^{2T}y+\frac{1}{2m}y^{\top}\mathcal{L}_{2T,n}y+R_{L,n,T}\\
        &=(g_{L,T})^{\top}f_{n}^{(0)}(X)+\frac{1}{2m}y^{\top}\mathcal{L}_{2T,n}y+R_{L,n,T}
        .
    \end{align*}
    By the definition of $\mathcal{L}_{2T,n}$, we have
    \[
        \frac{1}{2m}y^{\top}\mathcal{L}_{2T,n}y
        =-\frac{\eta_{\infty,T}}{2m^{2}}\sum_{s=0}^{2T-1}y^{\top}B_{T}^{s}E_{n}B_{T}^{2T-1-s}y
        =\langle Q_{L,T},E_{n}\rangle_{\mathrm{F}}
        ,
    \]
    which proves
    \[
        \phi_{n,T}(\eta_{\infty,T})-\phi_{\infty,T}(\eta_{\infty,T})
        =(g_{L,T})^{\top}f_{n}^{(0)}(X)+\langle Q_{L,T},E_{n}\rangle_{\mathrm{F}}+R_{L,n,T}
        .
    \]
    By Lemma~\ref{lem:power-derivative}, we have 
    \[
        \|P\mathcal{L}_{2T,n}P\|_{\mathrm{op}}
        =O_{p}(Tq_{*}^{2T-1}\|E_{n}\|_{\mathrm{op}})
        ,\qquad
        \|P\mathcal{R}_{2T,n}P\|_{\mathrm{op}}
        =O_{p}(T^{2}q_{*}^{2T-2}\|E_{n}\|_{\mathrm{op}}^{2})
        .
    \]
    By Lemma~\ref{lem:local-random-power-bound}, we also have
    \[
        \|PB_{n}(\eta_{\infty,T})^{2T}P\|_{\mathrm{op}}
        =O_{p}(q_{*}^{2T})
        .
    \]
    By \eqref{eq:range-K-prediction-error}, we have 
    \[
        f_{n}^{(0)}(X)
        \stackrel{\mathrm{a.s.}}{=}Pf_{n}^{(0)}(X)
        ,\qquad
        \mathcal{L}_{2T,n}
        \stackrel{\mathrm{a.s.}}{=}P\mathcal{L}_{2T,n}P
        ,\qquad
        \mathcal{R}_{2T,n}
        \stackrel{\mathrm{a.s.}}{=}P\mathcal{R}_{2T,n}P
        .
    \]
    Lemma~\ref{lem:joint-clt} then gives
    \begin{align*}
        |R_{L,n,T}|
        &\lesssim|y^{\top}\mathcal{R}_{2T,n}y|+|2f_{n}^{(0)}(X)^{\top}(\mathcal{L}_{2T,n}+\mathcal{R}_{2T,n})y|+|f_{n}^{(0)}(X)^{\top}B_{n}(\eta_{\infty,T})^{2T}f_{n}^{(0)}(X)|\\
        &=O_{p}(T^{2}q_{*}^{2T-2}\|E_{n}\|_{\mathrm{op}}^{2}+Tq_{*}^{2T-1}\|E_{n}\|_{\mathrm{op}}\|f_{n}^{(0)}(X)\|+q_{*}^{2T}\|f_{n}^{(0)}(X)\|^{2})\\
        &=O_{p}\left(\frac{T^{2}q_{*}^{2T-2}}{n}\right)
        ,
    \end{align*}
    and
    \[
        \frac{\sqrt{n}|R_{L,n,T}|}{Tq_{*}^{2T-1}}
        =O_{p}\left(\frac{T}{q_{*}\sqrt{n}}\right)
        =o_{p}(1)
        .
    \]
\end{proof}

\begin{theorem}[Limit distribution of the optimal loss at growing $T$]\label{thm:a-limit-dist}
    Under the assumptions of Theorem~\ref{thm:a-b-c-summary}, we have
    \[
        \frac{\sqrt{n}}{Tq_{*}^{2T-1}}(\phi_{n,T}(\eta_{n,T})-\phi_{\infty,T}(\eta_{\infty,T}))
        \stackrel{d}{\longrightarrow}\Omega_{L,*}
        =\langle Q_{L,*},\Xi\rangle_{\mathrm{F}}
        .
    \]
\end{theorem}
\begin{proof}
    Lemma~\ref{lem:loss-linearization} implies
    \begin{align*}
        \frac{\sqrt{n}}{Tq_{*}^{2T-1}}(\phi_{n,T}(\eta_{\infty,T})-\phi_{\infty,T}(\eta_{\infty,T}))
        =\frac{\sqrt{n}}{Tq_{*}^{2T-1}}[(g_{L,T})^{\top}f_{n}^{(0)}(X)+\langle Q_{L,T},E_{n}\rangle_{\mathrm{F}}]+o_{p}(1)
        .
    \end{align*}
    By \eqref{eq:range-K-prediction-error} and Lemmas~\ref{lem:frechet-coefficient-limits} and~\ref{lem:joint-clt}, we have
    \[
        \left|\frac{\sqrt{n}(g_{L,T})^{\top}f_{n}^{(0)}(X)}{Tq_{*}^{2T-1}}\right|
        \stackrel{\mathrm{a.s.}}{=}\left|\frac{\sqrt{n}(Pg_{L,T})^{\top}f_{n}^{(0)}(X)}{Tq_{*}^{2T-1}}\right|
        \le\frac{\|Pg_{L,T}\|}{Tq_{*}^{2T-1}}\|\sqrt{n}f_{n}^{(0)}(X)\|
        =o_{p}(1)
        .
    \]
    Moreover, since $E_{n}=PE_{n}P$ by \eqref{eq:range-K-prediction-error}, we have by Lemma~\ref{lem:frechet-coefficient-limits} that
    \[
        \frac{\sqrt{n}\langle Q_{L,T},E_{n}\rangle_{\mathrm{F}}}{Tq_{*}^{2T-1}}
        =\left\langle\frac{PQ_{L,T}P}{Tq_{*}^{2T-1}},\sqrt{n}E_{n}\right\rangle_{\mathrm{F}}
        \stackrel{d}{\longrightarrow}\langle Q_{L,*},\Xi\rangle_{\mathrm{F}}
        .
    \]
    Thus we have
    \begin{equation}\label{eq:loss-at-limit-eta}
        \frac{\sqrt{n}}{Tq_{*}^{2T-1}}(\phi_{n,T}(\eta_{\infty,T})-\phi_{\infty,T}(\eta_{\infty,T}))
        \stackrel{d}{\longrightarrow}\langle Q_{L,*},\Xi\rangle_{\mathrm{F}}
        .
    \end{equation}

    It remains to replace $\eta_{\infty,T}$ by $\eta_{n,T}$ in the finite-width loss. By Taylor's theorem, there exists a random $\bar{\eta}_{n,T}$ between $\eta_{\infty,T}$ and $\eta_{n,T}$ such that
    \[
        \phi_{n,T}(\eta_{\infty,T})-\phi_{n,T}(\eta_{n,T})
        =\phi_{n,T}'(\eta_{n,T})(\eta_{\infty,T}-\eta_{n,T})+\frac{1}{2}\phi_{n,T}''(\bar{\eta}_{n,T})(\eta_{\infty,T}-\eta_{n,T})^{2}
        .
    \]
    By Lemma~\ref{lem:preliminary-b-rate}, with probability tending to one, we have $\eta_{n,T}\in(\eta_{\ell},\eta_{u})$ and $F_{n,T}(\eta_{n,T})=0$. Hence, on the same event, Proposition~\ref{prop:loss-derivatives} gives $\phi_{n,T}'(\eta_{n,T})=0$. Therefore, with probability tending to one, we have
    \begin{equation}\label{eq:loss-difference}
        \phi_{n,T}(\eta_{\infty,T})-\phi_{n,T}(\eta_{n,T})
        =\frac{1}{2}\phi_{n,T}''(\bar{\eta}_{n,T})(\eta_{\infty,T}-\eta_{n,T})^{2}
        .
    \end{equation}
    By Lemma~\ref{lem:preliminary-b-rate}, we have $|\bar{\eta}_{n,T}-\eta_{\infty,T}|=o_{p}(T^{-1})$. Moreover, \eqref{eq:second-derivative} and Lemma~\ref{lem:kernel-H_n-A_n-K} yield
    \begin{align*}
        &\sup_{|\eta-\eta_{\infty,T}|\le C/T}|\phi_{n,T}''(\eta)|
        =\frac{T(2T-1)}{m^{3}}\sup_{|\eta-\eta_{\infty,T}|\le C/T}|(\chi_{n}^{(0)})^{\top}A_{n}^{2}B_{n}(\eta)^{2T-2}\chi_{n}^{(0)}|\\
        &\stackrel{\mathrm{a.s.}}{=}\frac{T(2T-1)}{m^{3}}\sup_{|\eta-\eta_{\infty,T}|\le C/T}|(P\chi_{n}^{(0)})^{\top}(PA_{n}^{2}P)(PB_{n}(\eta)^{2T-2}P)(P\chi_{n}^{(0)})|\\
        &\le\frac{T(2T-1)}{m^{3}}\|P\chi_{n}^{(0)}\|^{2}\|A_{n}\|_{\mathrm{op}}^{2}\sup_{|\eta-\eta_{\infty,T}|\le C/T}\|PB_{n}(\eta)^{2T-2}P\|_{\mathrm{op}}
        .
    \end{align*}
    Here, $\|P\chi_{n}^{(0)}\|=O_{p}(1)$ and $\|A_{n}\|_{\mathrm{op}}=O_{p}(1)$ hold by Lemma~\ref{lem:joint-clt}, while Lemma~\ref{lem:local-random-power-bound} implies $\sup_{|\eta-\eta_{\infty,T}|\le C/T}\|PB_{n}(\eta)^{2T-2}P\|_{\mathrm{op}}=O_{p}(q_{*}^{2T-2})$. Thus we have
    \[
        \sup_{|\eta-\eta_{\infty,T}|\le C/T}|\phi_{n,T}''(\eta)|
        =O_{p}(T^{2}q_{*}^{2T-2})
        ,
    \]
    which implies $\phi_{n,T}''(\bar{\eta}_{n,T})=O_{p}(T^{2}q_{*}^{2T-2})$. Equation~\eqref{eq:loss-difference} and Lemma~\ref{lem:preliminary-b-rate} imply that
    \[
        \frac{\sqrt{n}}{Tq_{*}^{2T-1}}(\phi_{n,T}(\eta_{n,T})-\phi_{n,T}(\eta_{\infty,T}))
        =O_{p}\left(\frac{T}{q_{*}\sqrt{n}}\right)
        =o_{p}(1)
        .
    \]
    Therefore, by \eqref{eq:loss-at-limit-eta} and Slutsky's theorem, we obtain
    \[
        \frac{\sqrt{n}}{Tq_{*}^{2T-1}}(\phi_{n,T}(\eta_{n,T})-\phi_{\infty,T}(\eta_{\infty,T}))
        \stackrel{d}{\longrightarrow}\langle Q_{L,*},\Xi\rangle_{\mathrm{F}}
        .
    \]
\end{proof}

\subsection{Convergence rates of the optimal loss and the optimal learning rate}

\begin{lemma}[Quadratic expansion of the transfer gap]\label{lem:transfer-quadratic}
    Under the assumptions of Theorem~\ref{thm:a-b-c-summary}, we have
    \[
        c_{n,T}
        =\frac{\phi_{\infty,T}''(\eta_{\infty,T})}{2}(\eta_{n,T}-\eta_{\infty,T})^{2}(1+o_{p}(1))
        .
    \]
\end{lemma}
\begin{proof}
    Recall that $\eta_{\infty,T}\in(\eta_{\ell},\eta_{u})$ for sufficiently large $T$. Thus Corollary~\ref{cor:optimizer-characterization} implies  $\phi_{\infty,T}'(\eta_{\infty,T})=0$. By Taylor's theorem, there exists a random $\eta_{n,T}^{\dagger}$ between $\eta_{n,T}$ and $\eta_{\infty,T}$ such that
    \[
        c_{n,T}
         =\phi_{\infty,T}(\eta_{n,T})-\phi_{\infty,T}(\eta_{\infty,T})
        =\frac{\phi_{\infty,T}''(\eta_{\infty,T})}{2}(\eta_{n,T}-\eta_{\infty,T})^{2}+\frac{1}{6}\phi_{\infty,T}'''(\eta_{n,T}^{\dagger})(\eta_{n,T}-\eta_{\infty,T})^{3}
        .
    \]
    Differentiating \eqref{eq:second-derivative} gives
    \[
        \phi_{\infty,T}'''(\eta)
        =-\frac{T(2T-1)(2T-2)}{m^{4}}
        y^{\top}K^{3}B_{\infty}(\eta)^{2T-3}y
        .
    \]
    Lemma~\ref{lem:preliminary-b-rate} gives $|\eta_{n,T}^{\dagger}-\eta_{\infty,T}|=O_{p}(n^{-1/2})=o_{p}(T^{-1})$. Therefore, with probability tending to one, $\eta_{n,T}^{\dagger}$ belongs to the $C/T$-neighborhood of $\eta_{\infty,T}$ in Lemma~\ref{lem:local-random-power-bound}. Thus, Lemma~\ref{lem:local-random-power-bound} yields
    \[
        |\phi_{\infty,T}'''(\eta_{n,T}^{\dagger})|
        =O_{p}(T^{3}q_{*}^{2T-3})
        .
    \]
    Together with \eqref{eq:finite-loss-curvature} and Lemma~\ref{lem:preliminary-b-rate}, we have
    \[
        \left|\frac{\phi_{\infty,T}'''(\eta_{n,T}^{\dagger})(\eta_{n,T}-\eta_{\infty,T})/6}{\phi_{\infty,T}''(\eta_{\infty,T})/2}\right|
        =O_{p}\left(\frac{T^{3}q_{*}^{2T-3}}{T^{2}q_{*}^{2T-2}\sqrt{n}}\right)
        =O_{p}\left(\frac{T}{q_{*}\sqrt{n}}\right)
        =o_{p}(1)
        ,
    \]
    which implies
    \[
        c_{n,T}
        =\frac{\phi_{\infty,T}''(\eta_{\infty,T})}{2}(\eta_{n,T}-\eta_{\infty,T})^{2}(1+o_{p}(1))
        .
    \]
\end{proof}

\begin{theorem}[Rates of $a_{n,T},b_{n,T}$, and $c_{n,T}$]\label{thm:a-b-c-rate}
    Under the assumptions of Theorem~\ref{thm:a-b-c-summary}, we have 
    \begin{align*}
        a_{n,T}
        =\Theta_{p}\left(\frac{Tq_{*}^{2T-1}}{\sqrt{n}}\right)
        ,\qquad
        b_{n,T}
        =\Theta_{p}(n^{-1/2})
        ,\qquad
        c_{n,T}
        =\Theta_{p}\left(\frac{T^{2}q_{*}^{2T-2}}{n}\right)
        .
    \end{align*}
\end{theorem}
\begin{proof}
    We first show that the limit distribution of Theorem~\ref{thm:b-limit-dist} is nondegenerate. Recall from Lemma~\ref{lem:joint-clt} that $\Xi=\alpha K+XGX^{\top}$, where $\alpha\sim\mathcal{N}(0,2)$ is independent of $G$. The contribution of $\alpha K$ to $\Omega_{\eta,*}$ is
    \[
        -\frac{2m}{(\lambda_{1}+\lambda_{r})^{2}}\alpha(\lambda_{1}+\lambda_{r})
        =-\frac{2m}{\lambda_{1}+\lambda_{r}}\alpha
        =-\eta_{*}\alpha
        \sim\mathcal{N}(0,2\eta_{*}^{2})
        ,
    \]
    which has strictly positive variance and is independent of the remaining Gaussian term. Hence $\Omega_{\eta,*}$ is nondegenerate, and we have
    \begin{equation}\label{eq:b-rate}
        b_{n,T}
        =\Theta_{p}(n^{-1/2})
        .
    \end{equation}

    We next show the nondegeneracy of the limit distribution of Theorem~\ref{thm:a-limit-dist}. Since $c_{r}\lambda_{r}a_{r}^{*}/(c_{1}\lambda_{1}a_{1}^{*})=c_{r}\lambda_{r}e^{L_{*}}/(c_{1}\lambda_{1})=1$, we have
    \begin{equation}\label{eq:endpoint-balance-constant}
        h_{*}
        :=c_{1}\lambda_{1}a_{1}^{*}
        =c_{r}\lambda_{r}a_{r}^{*}
        >0
        .
    \end{equation}
    This implies
    \[
        \left\langle Q_{L,*},K\right\rangle_{\mathrm{F}}
        =\frac{\eta_{*}}{m^{2}}(c_{1}a_{1}^{*}\lambda_{1}-c_{r}a_{r}^{*}\lambda_{r})
        =0
        .
    \]
    Thus we can write
    \[
        \Omega_{L,*}
        =\langle Q_{L,*},XGX^{\top}\rangle_{\mathrm{F}}
        =\left\langle X^{\top}Q_{L,*}X,G\right\rangle_{\mathrm{F}}
        .
    \]
    Recall from Lemma~\ref{lem:joint-clt} that $\mathbb{E}(G_{rs}G_{tu})=d^{-2}(\delta_{rt}\delta_{su}+\delta_{ru}\delta_{st})$. Thus for a deterministic symmetric matrix $S$, we have
    \[
        \operatorname{Var}(\langle S,G\rangle_{\mathrm{F}})
        =\sum_{r,s,t,u}S_{rs}S_{tu}\mathbb{E}(G_{rs}G_{tu})
        =2d^{-2}\|S\|_{\mathrm{F}}^{2}
        .
    \]
    Therefore we have
    \begin{align*}
        \operatorname{Var}(\Omega_{L,*})
        &=2d^{-2}\|X^{\top}Q_{L,*}X\|_{\mathrm{F}}^{2}
        =2\operatorname{tr}(Q_{L,*}KQ_{L,*}K)\\
        &=2\frac{\eta_{*}}{m^{2}}(c_{1}a_{1}^{*}u_{1}^{\top}KQ_{L,*}Ku_{1}-c_{r}a_{r}^{*}u_{r}^{\top}KQ_{L,*}Ku_{r})\\
        &=2\frac{\eta_{*}}{m^{2}}(c_{1}a_{1}^{*}\lambda_{1}^{2}u_{1}^{\top}Q_{L,*}u_{1}-c_{r}a_{r}^{*}\lambda_{r}^{2}u_{r}^{\top}Q_{L,*}u_{r})\\
        &=2\frac{\eta_{*}^{2}}{m^{4}}[(c_{1}a_{1}^{*}\lambda_{1})^{2}+(c_{r}a_{r}^{*}\lambda_{r})^{2}]\\
        &>0
        .
    \end{align*}
    Hence $\Omega_{L,*}$ is a nondegenerate centered Gaussian random variable. Theorem~\ref{thm:a-limit-dist} therefore yields
    \[
        a_{n,T}
        =\Theta_{p}\left(\frac{Tq_{*}^{2T-1}}{\sqrt{n}}\right)
        .
    \]

    Finally, Lemma~\ref{lem:transfer-quadratic}, \eqref{eq:finite-loss-curvature} and \eqref{eq:b-rate} imply
    \[
        c_{n,T}
        =\Theta_{p}\left(\frac{T^{2}q_{*}^{2T-2}}{n}\right)
        .
    \]
\end{proof}

\begin{proof}[Proof of Corollary~\ref{cor:a-b-c-summary}]
    By Theorems~\ref{thm:a-limit-dist},~\ref{thm:a-b-c-rate}, and the continuous mapping theorem, we have
    \[
        \frac{\sqrt{n}}{Tq_{*}^{2T-1}}a_{n,T}
        \stackrel{d}{\longrightarrow}|\Omega_{L,*}|
        ,
    \]
    where $\Omega_{L,*}$ is a nondegenerate Gaussian random variable. Since $\mathrm{P}(|\Omega_{L,*}|=0)=0$, applying the continuous mapping theorem once more gives
    \[
        \frac{Tq_{*}^{2T-1}}{\sqrt{n}a_{n,T}}
        \stackrel{d}{\longrightarrow}\frac{1}{|\Omega_{L,*}|}
        ,
    \]
    and consequently,
    \[
        \frac{1}{a_{n,T}}
        =O_{p}\left(\frac{\sqrt{n}}{Tq_{*}^{2T-1}}\right)
        .
    \]
    Together with $c_{n,T}=\Theta_{p}\left(T^{2}q_{*}^{2T-2}/n\right)$, this yields
    \[
        \frac{c_{n,T}}{a_{n,T}}
        =\Theta_{p}\left(\frac{T}{q_{*}\sqrt{n}}\right)
        \stackrel{p}{\longrightarrow}0
        ,
    \]
    where the last convergence follows from $T/\sqrt{n}\to0$ and the fact that $q_{*}>0$ is fixed.
\end{proof}

\section{Counterexample to fast transfer}

\begin{proof}[Proof of Proposition~\ref{prop:isotropic-no-fast-transfer}]
    Since $K=d^{-1}I_{m}$, Corollary~\ref{cor:fixed-T-limit-loss} gives
    \begin{equation}\label{eq:isotropic-infinite-loss}
        \phi_{\infty,T}(\eta)
        =\frac{\|y\|^{2}}{2m}\left(1-\frac{\eta}{\eta_{0}}\right)^{2T}
        .
    \end{equation}
    Then we have
    \[
        \eta_{\infty,T}
        =\eta_{0}
        ,\qquad
        \phi_{\infty,T}(\eta_{\infty,T})
        =0
        .
    \]
    Define $\epsilon_{n}:=n^{-1/2}$ and $M_{n}:=\sqrt{n}(dA_{n}-I_{m})$. Recall the Gaussian random variable $\alpha$ and the Gaussian matrix $G$ from Lemma~\ref{lem:joint-clt}, and define $M=\alpha I_{m}+dXGX^{\top}$. Here,  $\mathcal{G}:=dXGX^{\top}$ is a symmetric Gaussian matrix whose upper-triangular entries are independent and satisfy
    \[
        \mathcal{G}_{ii}\sim\mathcal{N}(0,2)
        ,\qquad
        \mathcal{G}_{ij}\sim\mathcal{N}(0,1)
        \quad(i<j)
        .
    \]
    Then Lemma~\ref{lem:joint-clt} and $XX^{\top}=I_{m}$ imply
    \begin{equation}\label{eq:isotropic-matrix-clt}
        (\chi_{n}^{(0)},M_{n})
        \stackrel{d}{\longrightarrow}(-y,M)
        .
    \end{equation}

    We first localize the finite-width optimizer without using positive curvature at the infinite-width optimum. By Lemma~\ref{lem:kernel-H_n-A_n-K}, we have $\ker(A_{n})\stackrel{\mathrm{a.s.}}{=}\ker(K)=\{0\}$. Since $A_{n}\succeq0$, this implies $A_{n}\succ0$ almost surely. By Proposition~\ref{prop:loss-convex-as}, we have $A_{n}\chi_{n}^{(0)}\ne0$ almost surely, and consequently, $\chi_{n}^{(0)}\ne0$ almost surely. Suppose $A_{n}=\sum_{j}\lambda_{j}(A_{n})v_{n,j}v_{n,j}^{\top}$ is the orthonormal eigendecomposition of $A_{n}$. Then the score is expressed as
    \[
        F_{n}(\eta)
        =\sum_{j}\lambda_{j}(A_{n})\left(1-\frac{\eta}{m}\lambda_{j}(A_{n})\right)^{2T-1}(v_{n,j}^{\top}\chi_{n}^{(0)})^{2}
        .
    \]
    Substituting $\eta=m/\lambda_{\max}(A_{n})$ and $\eta=m/\lambda_{\min}(A_{n})$, we have
    \[
        F_{n}\left(\frac{m}{\lambda_{\max}(A_{n})}\right)
        \ge0
        ,\qquad
        F_{n}\left(\frac{m}{\lambda_{\min}(A_{n})}\right)
        \le0
        .
    \]
    Thus the score characterization in Corollary~\ref{cor:optimizer-characterization} implies that the unconstrained optimizer $\hat{\eta}_{n,T}$ satisfies
    \[
        \frac{m}{\lambda_{\max}(A_{n})}
        \le\hat{\eta}_{n,T}
        \le\frac{m}{\lambda_{\min}(A_{n})}
        .
    \]
    Since $A_{n}=d^{-1}(I_{m}+\epsilon_{n} M_{n})$, we have $\lambda_{\max}(A_{n})=d^{-1}[1+\epsilon_{n}\lambda_{\max}(M_{n})]$ and $\lambda_{\min}(A_{n})=d^{-1}[1+\epsilon_{n}\lambda_{\min}(M_{n})]$. Noting that $M_{n}=O_{p}(1)$, we have $m/\lambda_{\max}(A_{n})=\eta_{0}+O_{p}(\epsilon_{n})$ and $m/\lambda_{\min}(A_{n})=\eta_{0}+O_{p}(\epsilon_{n})$. Therefore $\hat{\eta}_{n,T}$ lies in $(\eta_{\ell},\eta_{u})$ with probability tending to one, and coincides with $\eta_{n,T}$. In particular, we have
    \begin{equation}\label{eq:isotropic-optimizer-tightness}
        s_{n,T}:=\epsilon_{n}^{-1}\left(\frac{\eta_{n,T}}{\eta_{0}}-1\right)
        =O_{p}(1)
        .
    \end{equation}

    For $s\in\mathbb{R}$, the rescaled iteration matrix satisfies the exact identity
    \[
        B_{n}(\eta_{0}(1+\epsilon_{n}s))
        =-\epsilon_{n}(M_{n}+sI_{m}+\epsilon_{n}sM_{n})
        .
    \]
    Consequently, the rescaled finite-width objective is
    \begin{equation}\label{eq:isotropic-local-objective}
        n^{T}\phi_{n,T}(\eta_{0}(1+\epsilon_{n}s))
        =\frac{1}{2m}(\chi_{n}^{(0)})^{\top}(M_{n}+sI_{m}+\epsilon_{n} sM_{n})^{2T}\chi_{n}^{(0)}
        .
    \end{equation}
    Define $\mathcal{H}_{T}(s;D):=(2m)^{-1}y^{\top}(D+sI_{m})^{2T}y$ for symmetric matrix $D$. Since $T$ is fixed, the right-hand side of \eqref{eq:isotropic-local-objective} is a polynomial in $s$ whose coefficients converge jointly in distribution to those of $\mathcal{H}_{T}(s;M)$. Combining this with \eqref{eq:isotropic-matrix-clt}, for every $R<\infty$, we have
    \begin{equation}\label{eq:isotropic-joint-weak-conv}
        \left(M_{n},\chi_{n}^{(0)},n^{T}\phi_{n,T}(\eta_{0}(1+\epsilon_{n}s))\big|_{[-R,R]}\right)
        \stackrel{d}{\longrightarrow}\left(M,-y,\mathcal{H}_{T}(\cdot;M)\big|_{[-R,R]}\right)
    \end{equation}
    in $\mathbb{R}_{\mathrm{sym}}^{m\times m}\times\mathbb{R}^{m}\times C([-R,R])$, where $C([-R,R])$ is equipped with the uniform norm.

    For any symmetric $D$, let $\mu_{1}(D),\ldots,\mu_{m}(D)$ denote the eigenvalues of $D$, including zero and negative ones. Then a spectral decomposition expresses $\mathcal{H}_{T}(s;D)$ as a nonnegative weighted sum of $(s+\mu_{j}(D))^{2T}$, with at least one strictly positive weight because $y\neq0$. Moreover, for every $T\ge1$, the map $s\mapsto(s+\lambda)^{2T}$ is strictly convex and coercive. Hence $\mathcal{H}_{T}(\cdot;D)$ is strictly convex and coercive, and therefore admits a unique minimizer $s_{T}(D)$. Since $\hat{\eta}_{n,T}$ coincides with $\eta_{n,T}$ with probability tending to one, the global minimizer of $s\mapsto n^{T}\phi_{n,T}(\eta_{0}(1+\epsilon_{n}s))$ coincides with $s_{n,T}$ on the same event. Thus the continuous mapping theorem and \eqref{eq:isotropic-joint-weak-conv} yield
    \[
        (M_{n},\chi_{n}^{(0)},s_{n,T})
        \stackrel{d}{\longrightarrow}(M,-y,s_{T}(M))
        .
    \]

    Since $\phi_{\infty,T}(\eta_{0})=0$, we have $a_{n,T}=\phi_{n,T}(\eta_{n,T})$. Evaluating \eqref{eq:isotropic-local-objective} at $s=s_{n,T}$ therefore gives
    \[
        n^{T}a_{n,T}
        \stackrel{d}{\longrightarrow}L_{T}
        :=\mathcal{H}_{T}(s_{T}(M);M)
        .
    \]
    By \eqref{eq:isotropic-infinite-loss}, we further obtain
    \[
        n^{T}c_{n,T}
        =\frac{\|y\|^{2}}{2m}|s_{n,T}|^{2T}
        \stackrel{d}{\longrightarrow}C_{T}
        :=\frac{\|y\|^{2}}{2m}|s_{T}(M)|^{2T}
        .
    \]
    These convergences are joint; hence,
    \[
        \left(\sqrt{n}\left(\frac{\eta_{n,T}}{\eta_{0}}-1\right),n^{T}a_{n,T},n^{T}c_{n,T}\right)
        \stackrel{d}{\longrightarrow}(s_{T}(M),L_{T},C_{T})
        .
    \]
    Observe that the original random initial prediction is retained in \eqref{eq:isotropic-local-objective}. Since $\chi_{n}^{(0)}=-y+O_{p}(n^{-1/2})$, replacing $\chi_{n}^{(0)}$ by $-y$ changes the $n^{T}$-rescaled local objective by only $O_{p}(n^{-1/2})$, uniformly on compact sets of $s$. Hence the initialization fluctuation in $\chi_{n}^{(0)}$ does not contribute to the leading $O_{p}(1)$ local limit.

    It remains to establish nondegeneracy. If $L_{T}=0$, then $(M+s_{T}(M)I_{m})^{T}y=0$. By symmetry of $M$, this implies $(M+s_{T}(M)I_{m})y=0$, so $y$ must be an eigenvector of $M$, and hence of $\mathcal{G}$. For a fixed nonzero $y$ this event has probability zero when $m\ge2$: writing $e=y/\|y\|$, the vector $(I_{m}-ee^{\top})\mathcal{G}e=P_{e^{\perp}}\mathcal{G}e$ is a nondegenerate Gaussian vector on $e^{\perp}$. Thus $L_{T}>0$ almost surely.

    We next show that $C_{T}>0$ almost surely. Since $M=\alpha I_{m}+\mathcal{G}$, for every $s\in\mathbb{R}$ we have $\mathcal{H}_{T}(s;M)=\mathcal{H}_{T}(s+\alpha;\mathcal{G})$. Therefore, adding the scalar perturbation $\alpha I_{m}$ only translates the minimizing argument, and hence we have
    \[
        s_{T}(M)
        =s_{T}(\mathcal{G})-\alpha
        .
    \]
    Moreover, evaluating the objective at the minimizer gives
    \[
        L_{T}
        =\mathcal{H}_{T}(s_{T}(M);M)
        =\mathcal{H}_{T}(s_{T}(\mathcal{G});\mathcal{G})
        .
    \]
    Thus $L_{T}$ depends only on $\mathcal{G}$, whereas
    \[
        C_{T}
        =\frac{\|y\|^{2}}{2m}|s_{T}(\mathcal{G})-\alpha|^{2T}
        .
    \]
    Conditional on $\mathcal{G}$, the quantity $s_{T}(\mathcal{G})$ is fixed and $\alpha\sim\mathcal{N}(0,2)$ is independent of $\mathcal{G}$. Hence we have $\mathrm{P}(s_{T}(M)=0\mid\mathcal{G})=0$, and therefore $C_{T}>0$ almost surely. Together with the preceding argument showing $L_{T}>0$ almost surely, this implies
    \[
        \mathrm{P}\left(0<\frac{C_{T}}{L_{T}}<\infty\right)
        =1
        .
    \]
    Finally, the continuous mapping theorem proves
    \[
        \frac{c_{n,T}}{a_{n,T}}
        \stackrel{d}{\longrightarrow}\frac{C_{T}}{L_{T}}
        =\Theta_{p}(1)
        .
    \]
\end{proof}

\begin{remark}[Why the spectral condition matters]
    In this example, we have $B_{\infty}(\eta_{0})=0$. Thus by Lemma~\ref{lem:frechet}, the first-order loss sensitivity $D_{(\chi,A)}\Phi_{T}(\eta_{0};-y,K)$ vanishes, and the $n^{-1/2}$ optimal-loss fluctuation used in Corollary~\ref{cor:fixed-horizon-fast-transfer} is absent. Instead, both loss gaps are of order $n^{-T}$. This is not a failure of learning-rate consistency or of absolute transfer accuracy: $b_{n,T}\to0$ and $c_{n,T}\to0$ in probability still hold. What fails is the relative criterion $c_{n,T}=o_{p}(a_{n,T})$.
    We also note that $D_{(\chi,A)}F_{T}(\eta_{0};-y,K)$ also vanishes when $T\ge2$ while $D_{(\chi,A)}F_{1}(\eta_{0};-y,K)[h,E]=-y^{\top}Ey$.
\end{remark}

\bibliographystyle{chicago}
\bibliography{main}

\end{document}